\pdfoutput=1

\documentclass[11pt]{article} 
\usepackage[utf8]{inputenc} 

\usepackage{graphicx} 
\usepackage{amsmath, amsfonts, amssymb, color, amsxtra}
\usepackage{bm}
\usepackage{enumerate}
\usepackage{lineno}
\usepackage[utf8]{inputenc}
\usepackage{verbatim}
\usepackage{appendix}
\usepackage{amsmath}
\usepackage{amsfonts}
\usepackage{amssymb}
\usepackage{amsthm}
\usepackage{color}
\usepackage{caption}
\usepackage{subcaption}
\usepackage[affil-it]{authblk}
\usepackage{algpseudocode, algorithm, algorithmicx}

\usepackage[top=1in, bottom=1in, left=0.8in, right=0.8in]{geometry} 
\usepackage{booktabs} 
\usepackage{array} 
\usepackage{paralist} 
\usepackage{verbatim} 

\usepackage{fancyhdr} 
\usepackage[round]{natbib}
\RequirePackage[colorlinks,citecolor=blue,urlcolor=blue]{hyperref}

\newcommand{\diag}{\operatornamewithlimits{diag}}

\newtheorem{thm}{Theorem}
\newtheorem{lemma}{Lemma}
\newtheorem{col}{Corollary}
\newtheorem{prop}{Proposition}

\newtheorem{cond}{Condition}
\DeclareMathOperator*{\argmin}{arg\,min}

\DeclareMathOperator*{\E}{\mathbb{E}}

\date{} 

\author{Takumi Saegusa\thanks{tsaegusa@math.umd.edu.}} 
\author{Ali Shojaie\thanks{ashojaie@u.washington.edu.}}
\affil{University of Maryland and University of Washington}

\title{Joint Estimation of Precision Matrices in Heterogeneous Populations}

\begin{document}

\maketitle
\def\spacingset#1{\renewcommand{\baselinestretch}%
{#1}\normalsize} \spacingset{1}
\abstract{
We introduce a general framework for estimation of inverse covariance, or precision, matrices from heterogeneous populations.
The proposed framework uses a Laplacian shrinkage penalty to encourage similarity among estimates from disparate, but related, subpopulations, while allowing for differences among matrices. We propose an efficient alternating direction method of multipliers (ADMM) algorithm for parameter estimation, as well as its extension for faster computation in high dimensions by thresholding the empirical covariance matrix to identify the joint block diagonal structure in the estimated precision matrices.
We establish both variable selection and norm consistency of the proposed estimator for distributions with exponential or polynomial tails. Further, to extend the applicability of the method to the settings  with unknown populations structure, we propose a Laplacian penalty based on hierarchical clustering, and discuss conditions under which this data-driven choice results in consistent estimation of precision matrices in heterogenous populations. Extensive numerical studies and applications to gene expression data from subtypes of cancer with distinct clinical outcomes indicate the potential advantages of the proposed method over existing approaches.

}


\spacingset{1.5}

\section{Introduction}\label{sec:intro}
Estimation of large inverse covariance, or precision, matrices has received considerable attention in recent years. This interest is in part driven by the advent of high-dimensional data in many scientific areas, including high throughput \textit{omics} measurements, functional magnetic resonance images (fMRI), and applications in finance and industry. 
Applications of various statistical methods in such settings require an estimate of the (inverse) covariance matrix. Examples include  dimension reduction using principal component analysis (PCA), classification using linear or quadratic discriminant analysis (LDA/QDA), and discovering conditional independence relations in Gaussian graphical models (GGM).

In high-dimensional settings, where the data dimension $p$ is often comparable or larger than the sample size $n$, regularized estimation procedures often result in more reliable estimates. Of particular interest is the use of sparsity inducing penalties, specifically the $\ell_1$ or lasso penalty \citep{tibshirani1996}, which encourages sparsity in off-diagonal elements of the precision matrix \citep{MR2367824,MR2399568,FHTglasso,MR2719856}.
Theoretical properties of $\ell_1$-penalized precision matrix estimation have been studied under both multivariate normality, as well as some relaxations of this assumption \citep{MR2278363,MR2417391,MR2847973,MR2836766}.

Sparse estimation is particularly relevant in the setting of GGMs, where conditional independencies among variables correspond to zero off-diagonal elements of the precision matrix \citep{lauritzen1996}. 
The majority of existing approaches for estimation of high-dimensional precision matrices, including those cited in the previous paragraph, assume that the observations are identically distributed, and correspond to a single population.
However, data sets in many application areas include observations from several distinct subpopulations.
For instance, gene expression measurements are often collected for both healthy subjects, as well as patients diagnosed with different subtypes of cancer.
Despite increasing evidence for differences among genetic networks of cancer and healthy subjects \citep{ideker:krogan2012,sedaghat2014}, the networks are also expected to share many common edges. Separate estimation of graphical models for each of the subpopulations would ignore the common structure of the precision matrices, and may thus be inefficient; this inefficiency can be particularly significant in high-dimensional low sample settings, where $p \gg n$.

To address the need for estimation of graphical models in related subpopulations, few methods have been recently proposed for joint estimation of $K$ precision matrices $\Omega^{(k)}=(\omega^{(k)}_{ij})_{i,j=1}^p\in\mathbb{R}^{p\times p},k=1,\ldots,K$ \citep{MR2804206,Danaher}. 
These methods extend the penalized maximum likelihood approach by combining the Gaussian likelihoods for the $K$ subpopulations
\begin{equation}
\label{eqn:lik}
\ell_n(\Omega)
= \frac{1}{n}\sum_{k=1}^K  n_k\left(
\log \mbox{det}(\Omega^{(k)})-\mbox{tr}\left(\hat{\Sigma}_n^{(k)}\Omega^{(k)}\right)\right).
\end{equation}
Here, $n_k$ and $\hat{\Sigma}_n^{(k)}$ are the number of observations and the sample covariance matrix for the $k$th subpopulation, respectively, $n=\sum_{k=1}^K n_k$ is the total sample size and $\mbox{tr}(\cdot)$ and $\mbox{det}(\cdot)$ denote matrix trace and determinant.

To encourage similarity among estimated precision matrices, \citet{MR2804206} modeled the $(i,j)$-element of $\Omega^{(k)}$ as  product of a common factor $\theta_{ij}$ and group-specific parameters $\gamma^{(k)}_{ij}$, i.e. $\omega^{(k)}_{ij}=\delta_{ij}\gamma_{ij}^{(k)}$. 
Identifiability of the estimates is ensured by assuming $\delta_{ij} \ge 0$.
A zero common factor $\delta_{ij}=0$ induces sparsity across all subpopulations, whereas $\gamma_{ij}^{(k)}=0$ results in condition-specific sparsity for $\omega_{ij}^{(k)}$.
This reparametrization results in a non-convex optimization problem based on the Gaussian likelihood with $\ell_1$-penalties $\sum_{i\neq j}\delta_{ij}$ and $\sum_{i\neq j}\sum_{k=1}^K |\gamma_{ij}^{(k)}|$.
\citet{Danaher} proposed two alternative estimators by adding an additional convex penalty to the graphical lasso objective function: either a fused lasso penalty $\sum_{i\neq j}\sum_{k\neq k'}|\omega_{ij}^{(k)}-\omega_{ij}^{k'}|$ (FGL), or a group lasso penalty
$\sum_{i\neq j}\sqrt{\sum_{k=1}^K(\omega_{ij}^{(k)})^2}$
(GGL).
The fused lasso penalty has also been used by \citet{kolar2009}, for joint estimation of multiple graphical models in multiple time points. 
The fused lasso penalty strongly encourages the values of $\omega_{ij}^{(k)}$ to be similar across all subpopulations, both in values as well as sparsity patterns. On the other hand, the group lasso penalty results in similar estimates by shrinking all $\omega_{ij}^{(k)}$ across subpopulations to zero if $\sum_{k=1}^K (\omega_{ij}^{(k)})^2$ is small.

Despite their differences, methods of \citet{MR2804206} and \citet{Danaher} inherently assume that precision matrices in $K$ subpopulations are equally similar to each other, in that they encourage $\omega_{ij}^{(k)}$ and $\omega_{ij}^{(k')}$ and $\omega_{ij}^{(k)}$ and $\omega_{ij}^{(k'')}$ to be equally similar. 
However, when $K > 2$, some subpopulations are expected to be more similar to each other than others. For instance, it is expected that genetic networks of two subtypes of cancer be more similar to each other than to the network of normal cells.
Similarly, differences among genetic networks of various strains of a virus or bacterium are expected to correspond to the evolutionary lineages of their phylogenetic trees.
Unfortunately, existing methods for joint estimation of multiple graphical models ignore this heterogeneity in multiple subpopulations. 
Furthermore, existing methods assume subpopulation memberships are known, which limits their applicability in settings with complex but \emph{unknown} population structures; an important example is estimation of genetic networks of cancer cells with unknown subtypes.

In this paper, we propose a general framework for joint estimation of multiple precision matrices by capturing the heterogeneity among subpopulations.
In this framework, similarities among disparate subpopulations are presented using a \emph{subpopulation network} $G(V,E,W)$, a weighted graph whose node set $V$ is the set of subpopulations. The edges in $E$ and the weights $W_{kk'}$ for $(k,k')\in E$ represent the degree of similarity between any two subpopulations $k,k'$.
In the special case where $W_{kk'} = 1$ for all $k,k'$, the subpopulation similarities are only captured by the structure of the graph $G$. 
An example of such a subpopulation network is the line graph corresponding to observations over multiple time points, which is used in estimation of time-varying graphical models \citep{kolar2009}. 
As we will show in Section~\ref{sec:others}, other existing methods for joint estimation of multiple graphical models, e.g. proposals of \citet{Danaher}, can also be seen as special cases of this general framework. 

Our proposed estimator is the solution to a convex optimization problem based on the Gaussian likelihood with both $\ell_1$ and graph Laplacian \citep{MR2758338} penalties.
The graph Laplacian has been used in other applications for incorporating \emph{a priori} knowledge in classification \citep{RapaportZDBV07}, for principal component analysis on network data \citep{ShojaieM10}, and for penalized linear regression with correlated covariates \citep{MR2758338,MR2893860, WeinbergerEtal2006,LiuEtal2014,LiuEtal2011,ZhaoShojaie2015}.
The Laplacian penalty encourages similarity among estimated precision matrices according to the subpopulation network $G$. The $\ell_1$-penalty, on the other hand, encourages sparsity in the estimated precision matrices. Together, these two penalties capture both unique patterns specific to each subpopulation, as well as common patterns shared among different subpopulations.

We first discuss the setting where $G(V,E,W)$ is known from external information, e.g. known phylogenetic trees (Section~\ref{sec:method}), and later discuss the estimation of the subpopulation memberships and similarities using hierarchical clustering (Section~\ref{sec:HC}).
We propose an alternating methods of multipliers (ADMM) algorithm \citep{BoydPCPE11} for parameter estimation, as well as its extension for efficient computation in high dimensions by decomposing the  problem into block-diagonal matrices. 
Although we use the Gaussian likelihood, our theoretical results also hold for non-Gaussian distributions. 
We establish model selection and norm consistency of the proposed estimator under different model assumptions (Section~\ref{sec:theory}), with improved rates of convergence over existing methods based on penalized likelihood. 
We also establish the consistency of the proposed algorithm for the estimation of multiple precision matrices, in settings where the subpopulation network $G$ or subpopulation memberships are unknown. 
To achieve this, we establish the consistency of hierarchical clustering in high dimensions, by generalizing recent results of \citet{borysov2014} to the setting of arbitrary covariance matrices, which is of independent interest.

The rest of the paper is organized as follows.
In Section~\ref{sec:method} we describe the formal setup of the problem and present our estimator.
Theoretical properties of the proposed estimator are studied in Section~\ref{sec:theory}, and Section~\ref{sec:HC} discusses the extension of the method to the setting where the subpopulation network is unknown.
The ADMM algorithm for parameter estimation and its extension for efficient computation in high dimensions are presented in Section~\ref{sec:alg}.
Results of the numerical studies, using both simulated and real data examples, are presented in Section~\ref{sec:NumRes}.
Section~\ref{sec:conc} concludes the paper with a discussion. Technical proofs are collected in the Appendix.

\section{Model and Estimator}\label{sec:method}
\subsection{Problem Setup}\label{sec:setup}
Consider $K$ subpopulations with distributions $\mathcal{P}^{(k)}$, $k=1,\ldots,K$.
Let $X^{(k)} = (X^{(k),1},\ldots,X^{(k),p})^T\in \mathbb{R}^p$ be a random vector from the $k$th subpopulation with mean $\mu_k$ and the covariance matrix $\Sigma_0^{(k)}=(\sigma_{ij}^{(k)})_{i,j=1}^p$.
Suppose that an observation comes from the $k$th subpopulation with probability $\pi_k>0$.

Our goal is to estimate the precision matrices $\Omega_0^{(k)}\equiv (\Sigma_0^{(k)})^{-1}\equiv (\omega^{(k)}_{ij})_{i,j=1}^p$, $k=1,\ldots,K$. To this end, we use the Gaussian log-likelihood based on the \textit{correlation matrix} (see \citet{MR2417391}) as a working model for estimation of true $\Omega_0^{(k)}, k=1, \ldots, K$.
Let $X^{(k)}_i,i=1,\ldots,n_k,$ be independent and identically distributed (i.i.d.) copies from $\mathcal{P}^{(k)},k=1,\ldots,K$.
We denote the correlation matrices and their inverse by $\Theta^{(k)}=(\theta^{(k)}_{ij})_{i,j=1}^p$, and $\Psi^{(k)}=(\psi^{(k)}_{ij})_{i,j=1}^p, k = 1, \ldots, K$, respectively.
The Gaussian log-likelihood based on the correlation matrix can then be written as
\begin{equation}
\label{eqn:corrlik}
\tilde{\ell}_n(\Theta)
= \frac{1}{n}\sum_{k=1}^K  n_k\left(
\log \mbox{det}(\Theta^{(k)})-\mbox{tr}\left(\Psi_n^{(k)}\Theta^{(k)}\right)\right),
\end{equation}
where $\Psi_n^{(k)}, k=1, \ldots, K$ is the sample correlation matrix for subpopulation $k$.

Examining the derivative of \eqref{eqn:corrlik}, which consists of $\Psi_0^{(k)} - \Psi_n^{(k)}, k=1, \ldots, K$,
justifies its use as a working model for non-Gaussian data: the stationary points of \eqref{eqn:corrlik} is $\Psi_n^{(k)}$, which gives a consistent estimate of $\Psi^{(k)}_0$. Thus we do not, in general, need to assume multivariate normality. However, in certain applications, for instance LDA/QDA and GGM, the resulting estimate is useful only if the data follows a multivariate normal distribution.

\subsection{The Laplacian Shrinkage Estimator}
Let $\Theta = (\Theta^{(1)},\ldots,\Theta^{(K)})$ and write $\Theta_{ij} = (\theta_{ij}^{(1)},\ldots,\theta_{ij}^{(K)})^T \in \mathbb{R}^K, i,j=1,\ldots, p$ for a vector of $(i,j)$-elements across subpopulations.
Our proposed estimator, Laplacian Shrinkage for Inverse Covariance matrices from Heterogeneous populations (LASICH), first estimates the inverse of the correlation matrices for each of the $K$ subpopulations, and then transforms them into the estimator of inverse covariance matrices, as in \citet{MR2417391}.
In particular, we first obtain the estimate $\hat{\Theta}$ of the true inverse correlation matrix by solving the following optimization problem
\begin{eqnarray}
\hat{\Theta}_{\rho_n}
&&\equiv\argmin_{\Theta = \Theta^T,\Theta \succ 0}-\tilde{\ell}_n(\Theta) + \rho_n\lVert \Theta \rVert_1 + \rho_n\rho_2 \lVert \Theta\rVert_{L}\nonumber\\
&&\equiv\argmin_{\Theta = \Theta^T,\Theta \succ 0}-\tilde{\ell}_n(\Theta) + \rho_n \sum_{k=1}^K\sum_{i\neq j}\left| \Theta_{ij}^{(k)}\right|
+  \rho_{n}\rho_2\sum_{i\neq j} \lVert \Theta_{ij}\rVert_L,
\label{eqn:originallasso}
\end{eqnarray}
where $\Theta=\Theta^T$ enforces the symmetry of individual inverse correlation matrices, i.e. $\Theta^{(k)} = (\Theta^{(k)})^T$, and $\Theta \succ 0$ requires that $\Theta^{(k)}$ is positive definite for $k=1,\ldots,K$.
The $\ell_1$-penalty $\lVert \Theta\rVert_1=\sum_{k=1}^K\lVert \Theta^{(k)} \rVert_1$ in \eqref{eqn:originallasso} encourages sparsity in estimated inverse correlation matrices.
The graph Laplacian penalty, on the other hand, exploits the information in the subpopulation network $G$ to encourage similarity among values of $\theta_{ij}^{(k)}$ and $\theta_{ij}^{(k')}$. 
The tuning parameters $\rho_n$ and $\rho_n\rho_2$ control the size of each penalty term.
\begin{figure}[htbp]
  \begin{center}
    \includegraphics[height=0.175\textheight]{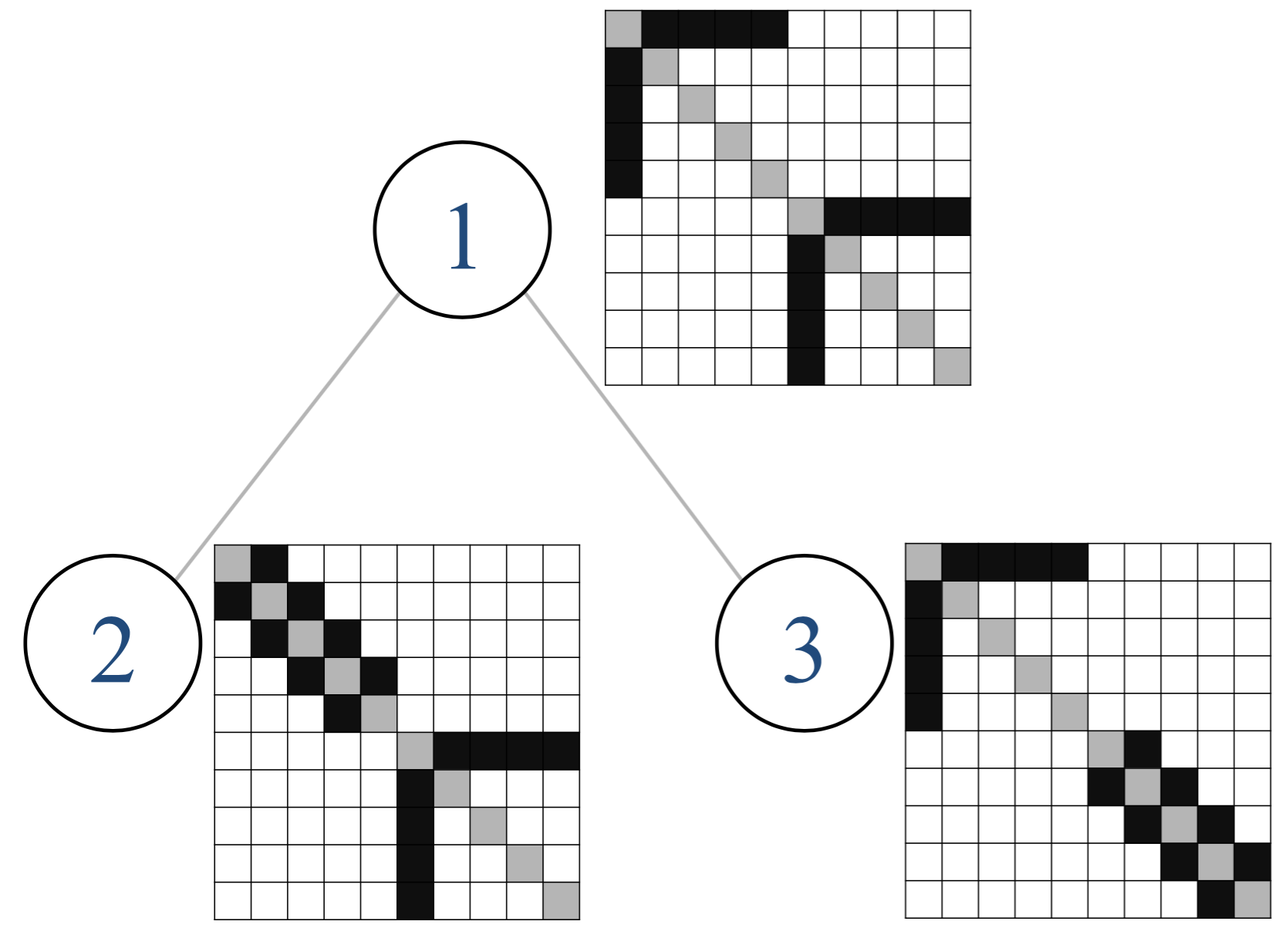}
  \end{center}
\caption{Illustration of similarities in the sparsity patterns of precision matrices $\Omega^{(1)}, \Omega^{(2)}$ and $\Omega^{(3)}$. Nonzero and zero off-diagonal entries are colored in black and white, respectively, while diagonal entires are colored in gray.  
The associated subpopulation network $G$ reflects the similarities between precision matrices of subpopulations 1 and 2 and 1 and 3. The simulation experiments in Section~\ref{sec:sims} use a similar subpopulation network in a high-dimensional setting.}
\label{fig:graphgroup}
\end{figure}

Figure~\ref{fig:graphgroup} illustrates the motivation for the graph Laplacian penalty $\lVert \Theta_{ij}\rVert_L$ in \eqref{eqn:originallasso}. The gray-scale images in the figure show the hypothetical sparsity patterns of precision matrices $\Theta^{(1)}, \Theta^{(2)}, \Theta^{(3)}$ for three related subpopulations. Here, $\Theta^{(1)}$ consists of two blocks with one ``hub'' node in each block; in $\Theta^{(2)}$ and $\Theta^{(3)}$ one of the blocks is changed into a ``banded'' structure. It can be seen that one of the two blocks in both $\Theta^{(2)}$ and $\Theta^{(3)}$ have a similar sparsity pattern as $\Theta^{(1)}$. However, $\Theta^{(2)}$ and $\Theta^{(3)}$ are not similar. The subpopulation network $G$ in this figure captures the relationship among precision matrices of the three subpopulations. Such complex relationships cannot be captured using the existing approaches, e.g. \citet{MR2804206,Danaher}, which encourage all precision matrices to be equally similar to each other. 
More generally, $G$ can be a weighted graph, $G(V,E,W)$, whose nodes represent the subpopulations $1, \ldots, K$. The edge weights $W:E\rightarrow \mathbb{R}_+$ represent the similarity among pairs of subpopulations, with larger values of $W_{kk'} \equiv W(k,k')>0$ corresponding to more similarity between precision matrices of subpopulations $k$ and $k'$.

In this section, we assume that the weighted graph $G$ is externally available, and defer the discussion of data-driven choices of $G$, based on hierarchical clustering, to Section~\ref{sec:HC}.
Given $G$, the (unnormalized) graph Laplacian penalty $\lVert \Theta_{ij}\rVert_L$ is defined as 
\begin{equation}\label{eq:lappen}
\lVert \Theta_{ij} \rVert_L = \left\{ \sum_{k,k'=1}^K W_{kk'}\left(\theta^{(k)}_{ij} - \theta^{(k')}_{ij} \right)^2 \right\}^{1/2}
\end{equation}
where $W_{kk'} = 0$ if $k$ and $k'$ are not connected.
The Laplacian shrinkage penalty can be alternatively written as $\lVert \Theta_{ij}\rVert_L = \Theta_{ij}^T L \Theta_{ij}$, where $L = (l_{kk'})_{k,k'=1}^K\in \mathbb{R}^{K\times K}$ is the Laplacian matrix \citep{chung1997} of the subpopulation network $G$ defined as
\begin{eqnarray*}
l_{kk'} =\left\{
\begin{array}{ll}
d_k -W_{kk} , & k=k',d_k\neq 0,\\
-W_{kk'},& k\neq k',\\
0, & \mbox{ otherwise,}
\end{array}
\right.
\end{eqnarray*}
where $d_k = \sum_{k' \ne k} W_{kk'}$ is the degree of node $k$ in $G$ with $W_{kk'} = 0$ if $k$ and $k'$ are not connected.
The Laplacian shrinkage penalty can also be defined in terms of the \emph{normalized} graph Laplacian, $I - D^{-1/2}WD^{-1/2}$, where $D = \diag(d_1, \ldots, d_K)$ is the diagonal degree matrix. The normalized Laplacian penalty,  
\begin{equation*}
\lVert \Theta_{ij} \rVert_L = \left\{ \sum_{k,k'=1}^K W_{kk'}\left(
\frac{\theta^{(k)}_{ij}}{\sqrt{d_k}} - \frac{\theta^{(k')}_{ij}}{\sqrt{d_{k'}}} \right)^2 \right\}^{1/2}, 
\end{equation*}
which we also denote as $\lVert \Theta_{ij}\rVert_L$, 
imposes smaller shrinkage on coefficients associated with highly connected subpopulations. We henceforth primarily focus on the normalized penalty. 

Given estimates of the inverse correlation matrices $\hat{\Theta}^{(1)},\ldots,\hat{\Theta}^{(K)}$ from \eqref{eqn:originallasso}, we obtain estimates of precision matrices $\Omega^{(k)}$ by noting that  $\Omega^{(k)} = \Xi^{(k)}\Theta^{(k)}\Xi^{(k)}$, where $\Xi^{(k)}$ is the diagonal matrix of reciprocals of the standard deviations $\Xi^{(k)}=\diag(\{\sigma_{11}^{(k)}\}^{-1/2},\ldots,\{\sigma_{pp}^{(k)}\}^{-1/2})$. Our estimator $\hat{\Omega}_{\rho_n} = (\hat{\Omega}_{\rho_n}^{(1)},\ldots,\hat{\Omega}_{\rho_n}^{(K)})$ of precision matrices $\Omega$ is thus defined as
\begin{equation*}
\hat{\Omega}_{\rho_n}^{(k)} = \{\hat{\Xi}^{(k)}\}^{-1} \hat{\Theta}_{\rho_n}^{(k)}\{\hat{\Xi}^{(k)}\}^{-1}, \quad k=1,\ldots, K,
\end{equation*}
where $\hat{\Xi}^{(k)}=\diag(1/ \{\hat{\sigma}_{11}^{(k)}\}^{1/2},\ldots,1/ \{\hat{\sigma}_{pp}^{(k)}\}^{1/2})$ with sample variance $\hat{\sigma}_{ii}^{(k)}$ for the $i$th element in the $k$th subpopulation.

A number of alternative strategies can be used instead of the graph Laplacian penalty in \eqref{eqn:originallasso}. 
First, similarity among coefficients of precision matrices can also be imposed using a ridge-type penalty, $\lVert \Theta_{ij}\rVert^2_L$. The main difference is that our  penalty $\lVert \Theta_{ij}\rVert_L$ discourages the inclusion of edges $\theta^{(1)}_{ij}, \ldots, \theta^{(K)}_{ij}$ if they are very different across the $K$ subpopulations. 
Another option is to use the graph trend filtering \citep{WangEtal2014},  which impose a fused lasso penalty over the subpopulation graph $G$. 
Finally, ignoring the weights $W_{kk'}$ in \eqref{eq:lappen}, the Laplacian shrinkage penalty resembles the Markov random field (MRF) prior used in Bayesian variable selection with structured covariates \cite{LiZhang2010}. Our penalized estimation framework can thus be seen as an alternative to using an MRF prior to estimate the precision matrices in a mixture of Gaussian distributions.

\subsection{Connections to Other Estimators}\label{sec:others}
To connect our proposed estimator to existing methods for joint estimation of multiple graphical models, we first give an alternative interpretation of the graph Laplacian penalty $\lVert \Theta_{ij}\rVert_L = \left(\Theta_{ij}^T L \Theta_{ij}\right)^{1/2}$ as a norm for a transformed version of $\theta^{(k)}_{ij}$s.
More specifically, consider the mapping $g_G: \mathbb{R}^K \rightarrow \mathbb{R}^K$ defined based on the Laplacian matrix for graph $G$ 
\begin{eqnarray*}
g_G(\Theta_{ij}) = \left\{\begin{array}{ll}
0, & k=k',\\
\sqrt{W_{kk'}}\,\left(\frac{\theta_{ij}^{(k)}}{\sqrt{2 d_k}}-\frac{\theta_{ij}^{(k')}}{\sqrt{2 d_{k'}}}\right), & k\neq k',\\
\end{array}\right.
\end{eqnarray*}
if $G$ has at least one edge. 
For a graph with no edges, define $g_G(\Theta_{ij}) = I_K \otimes \Theta_{ij} = \diag(\Theta_{ij})$, where $I_K$ is the $K$-identity matrix, and $\otimes$ denotes the Kronecker product. 
It can then be seen that the graph Laplacian penalty can be rewritten as
\begin{equation*}
\lVert \Theta_{ij}\rVert_L = \lVert g_G(\Theta_{ij})\rVert_F.
\end{equation*}
where $\lVert \cdot \rVert_F$ is the Frobenius norm.

Using the above interpretation, other methods for joint estimation of multiple graphical models can be seen as penalties on transformations $g_G(\Theta_{ij})$ corresponding to different graphs $G$. We illustrate this connection using the hypothetical subpopulation network shown in Figure~\ref{fig:compMethods}a. 
\begin{figure}[htbp]
\centering 
\includegraphics[width=.75\textwidth]{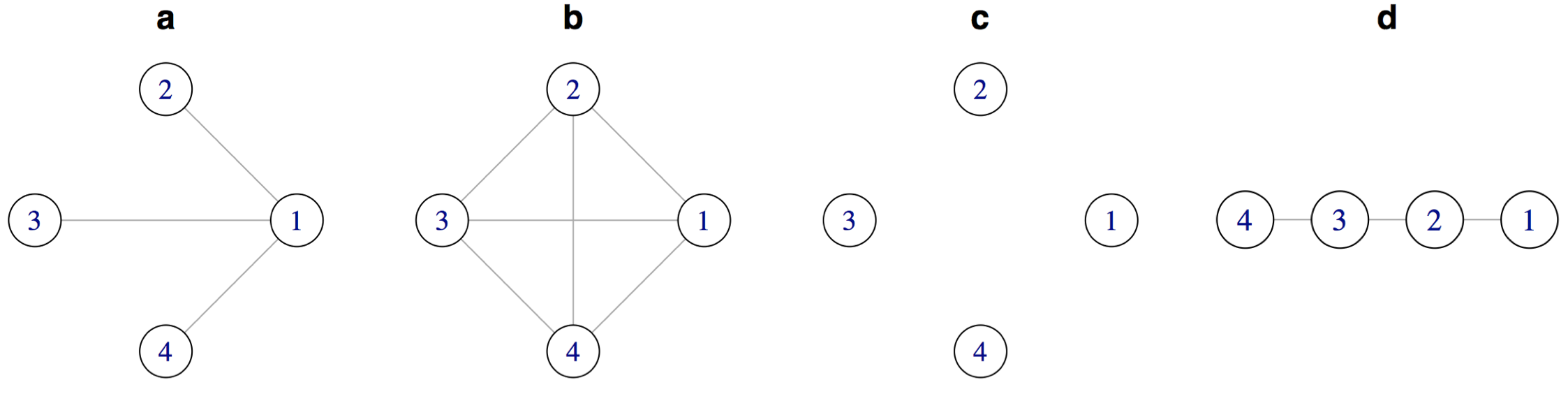} 
\caption{Comparison of subpopulation networks used in the penalty for different methods for joint estimation of multiple precision matrices: \textbf{a}) the true network, modeled by LASICH; \textbf{b}) FGL; \textbf{c}) GGL \& Guo et al; and \textbf{d}) estimation of time-varying networks (Kolar \& Xing, 2009); see Section~\ref{sec:others} for details.}
\label{fig:compMethods}
\end{figure}

Consider first the FGL penalty of \citet{Danaher}, applied to elements of the inverse correlation matrix $| \theta_{ij}^{(k)} - \theta_{ij}^{(k')} |$. Let $G_C$ be a complete unweighted graph ($W_{kk'}=1 \, \forall k \ne k'$), in which all ${K \choose 2}$ node-pairs are connected to each other (Figure~\ref{fig:compMethods}b). It is then easy to see that 
\begin{equation*}
\sum_{k\neq l} |\theta_{ij}^{(k)} - \theta_{ij}^{(l)}|
= \sqrt{2(K-1)}\lVert g_{G_C}(\Theta_{ij})\rVert_1,
\end{equation*}
where the factor of $\sqrt{2(K-1)}$ can be absorbed into the tuning parameter for the FGL penalty.
A similar argument can also be applied to the GGL penalty of \citet{Danaher}, $\rVert \Theta_{ij} \lVert$, by considering instead an empty graph $G_e$ with no edges between nodes (Figure~\ref{fig:compMethods}c). In this case, the mapping $g_G$ would give a diagonal matrix with elements $\theta^{(k)}_{ij}$, and hence $\rVert \Theta_{ij} \lVert = \lVert g_{G_e}(\Theta_{ij})\rVert_F$. 

Unlike proposals of \citet{Danaher}, the estimator of \citet{MR2804206} is based on a non-convex penalty, and does not naturally fit into the above framework. However, Lemma 2 in \citet{MR2804206} establishes a connection between the optimal solutions of the original optimization problem, with those obtained by considering a single penalty of the form $\left\{\sum_{k=1}^K|\theta_{ij}^{(k)}|\right\}^{1/2} \equiv \rVert \Theta_{ij} \lVert_{1,2}$. Similar to GGL, the connection with the method of \citet{MR2804206} can be build based on the above alternative formulation, by considering again the empty graph $G_e$ (Figure~\ref{fig:compMethods}c), but instead the $\rVert . \lVert_{1,2}$ penalty, which is a member of the CAP family of penalties~\cite{zhao2009composite}. More specifically, 
\begin{equation*}
\left\{\sum_{k=1}^K|\omega_{ij}^{(k)}|\right\}^{1/2} = \lVert g_{G_e}(\Theta_{ij}) \rVert_{1,2}.
\end{equation*}

Using the above framework, it is also easy to see the connection between our proposed estimator and the proposal of \citet{kolar2009}: the total variation penalty in \citet{kolar2009} is closely related to FGL, with summation over differences in consecutive time points. It is therefore clear that the penalty of \citet{kolar2009} (up to constant multipliers) can be obtained by applying the graph Laplacian penalty defined for a line graph connecting the time points (Figure~\ref{fig:compMethods}d).

The above discussion highlights the generality of the proposed estimator, and its connection to existing methods. In particular, while FGL and GGL/\citet{MR2804206} consider extreme cases with isolated, or fully connected nodes, one can obtain more flexibility in estimation of multiple precision matrices by defining the penalty based on the known subpopulation network, e.g. based on phylogenetic trees or spatio-temporal similarities between fMRI samples.
The clustering-based approach of Section~\ref{sec:HC} further extends the applicability of the proposed estimator to the settings where the  subpopulation network in not known \emph{a priori}. 
The simulation results in Section~\ref{sec:NumRes} show that the additional flexibility of the proposed estimator can result in significant improvements in estimation of multiple precision matrices, when $K > 2$.
The above discussion also suggests that other variants of the proposed estimator can be defined, by considering other norms. We leave such extensions to future work.

\section{Theoretical Properties}\label{sec:theory}
In this section, we establish norm and model selection consistency of the LASICH estimator.
We consider a high-dimensional setting $p\gg n_k,k=1,\ldots,K$, where both $n$ and $p$ go to infinity.
As mentioned in the Introduction, the normality assumption is not required for establishing these results.
We instead require conditions on tails of random vectors $X^{(k)}$ for each $k=1,\ldots,K$.
We consider two cases, exponential tails and polynomial tails, which both allow for distributions other than multivariate normal. 
\begin{cond}[Exponential Tails]
\label{cond:exptail}
There exists a constant $c_1\in (0,\infty)$  such that
\begin{equation*}
\mathbb{E}\left[ \exp \left\{t(X_j^{(k)}-\mu_j^{(k)})/(\sigma_{jj}^{(k)})^{1/2}\right\} \right] \leq e^{c_1^2t^2/2}, \, \forall t\in\mathbb{R}, k=1,\ldots,K,j=1,\ldots,p.
\end{equation*}
\end{cond}
\begin{cond}[Polynomial Tails]
\label{cond:polytail}
There exist constants $c_2,c_3> 0$ and $c_4$ such that
\begin{equation*}
\mathbb{E}\left[ \left\{X_j^{(k)}/(\sigma_{jj}^{(k)})^{1/2}\right\}^{4(c_2+c_3+1)} \right] \leq c_4, \quad k=1,\ldots,K, j=1,\ldots,p.
\end{equation*}
\end{cond}

Since we adopt the correlation-based Gaussian log-likelihood, we require the boundedness of the true variances to control the error between true and sample correlation matrices.
\begin{cond}[Bounded variance]
\label{cond:eigen}
There exist constants $c_5>0$ and $c_6<\infty$ such that $c_5\leq \min_{k,j}\sigma^{(k)}_{jj}$ and $\max_{k,j}\sigma^{(k)}_{jj}\leq c_6$.
\end{cond}

\begin{cond}[Sample size]
\label{cond:samplesize1}
Let $\lambda_{\Theta} \equiv \max_k\lVert \Theta_0^{(k)}\rVert_2$.
Let 
\begin{equation*}
C_1\equiv \left\{2c_5^{-2}+c_5 + c_6^{-3/2} +2c_5^{-5/2}c_6 + (c_5^{-4}+2c_5^{-5}c_6)^{1/2}\right\}^{-1}.
\end{equation*}
(i) (Exponential tails). 
It holds that 
\begin{equation*}
n\geq  \max\left\{\frac{12}{\min_k\pi_k},2^{18}3^{3}C_1^2(1+4c_1^2)^2c_6^2 \lambda_{\Theta}^4\left(1+\lVert L\rVert_2^{1/2}\right)^2 s \right\}\log p,
\end{equation*}
and $\log p/n\rightarrow 0.$

(ii) (Polynomial tails). 
Let $C_2 = \sup_n\{ \rho_n\sqrt{n/\log p}\}=O(1)$ where $\rho_n$ is given in Lemma~\ref{lemma:l1const} in the Appendix and $c_7>0$ be some constant.
It holds that 
\begin{equation*}
n\geq \max\left\{ \frac{p^{1/c_2}}{c_7^{1/c_2}},2^73^2C_1^2C_2^2K\min_k\pi_k \lambda_{\Theta}^4\left(1+\lVert L\rVert_2^{1/2}\right)^2 s
\log p\right\}.
\end{equation*}
\end{cond}
Condition~\ref{cond:samplesize1} determines the sufficient sample size $n = \sum_{k}$ for consistent estimation of precision matrices $\Theta^{(1)},\ldots,\Theta^{(K)}$ in relation to, among other quantities, the number of variables $p$, the sparsity pattern $s$ and the spectral norm of the Laplacian matrix $\|L\|_2$ of the subpopulation network $G$. While a general characterization of $\|L\|_2$ is difficult, investigating its value in special cases provides insight into the effect of the underlying population structure on the required sample size. Consider, for instance, two extreme cases: for a fully connected graph $G$ associated with $K$ subpopulations, $\|L\|_2 = 1/(K-1)$; for a minimally connected ``line'' graph, corresponding to e.g. multiple time points, $\|L\|_2 = 2$: with $K=5$, 30\% more samples are needed for the line graph, compared to a fully connected network. 
The above calculations match our intuition that fewer samples are needed to consistently estimate precision matrices of $K$ subpopulations that share greater similarities.  
This, of course, makes sense, as information can be better shared when estimating parameters of similar subpopulations. 
Note that, here $L$ represents the Laplacian matrix of the \emph{true} subpopulation network capturing the underlying population structure. The above conditions thus do not provide any insight into the effect of misspecifying the relationship between subpopulations, i.e., when an incorrect $L$ is used. 
This is indeed an important issue that garners additional investigation; see \citet{ZhaoShojaie2015} for some insight in the context of inference for high dimensional regression. In Section~\ref{sec:HC}, 
 we will discuss a data-driven choice of $L$ that results in consistent estimation of precision matrices.

Before presenting the asymptotic results, we introduce some additional notations. 
For a matrix $A = (a_{ij})_{i,j=1}^p\in\mathbb{R}^{p\times p}$, we denote the spectral norm  $\lVert A \rVert_2 = \max_{x\in\mathbb{R}^p,\lVert x\rVert =1}\lVert Ax\rVert$, and the element-wise $\ell_\infty$-norm $\lVert A\rVert_\infty=\max_{i,j}|a_{i,j}|$ where $\lVert x \rVert$ is the Euclidean norm for a vector $x$.
We also write the induced $\ell_\infty$-norm $\lVert A\rVert_{\infty/\infty} = \sup_{\lVert x\rVert_\infty=1}\lVert Ax\rVert_\infty$ where $\lVert x\rVert_\infty = \max_{i}|x_i|$ for $x=(x_1,\ldots,x_p)$.
For the ease of presentation, the results in this section are presented in asymptotic form; non-asymptotic results and proofs are deferred to the Appendix.

\subsection{Consistency in Spectral Norm}\label{sec:spectral}
Let $s\equiv \#\{(i,j):\omega_{0,ij}^{(k)}\neq 0,i,j=1,\ldots,p,i\neq j,k=1,\ldots,K \}$, and $d=\max_{k,i}\#\{(i,j):\omega_{0,ij}^{(k)}\neq 0,j=1,\ldots,p,i\neq j\}$.
The following theorem establishes the rate of convergence of the LASICH estimator, in spectral norm, under either exponential or polynomial tail conditions (Condition \ref{cond:exptail} or \ref{cond:polytail}).
Convergence rates for LASICH in $\ell_\infty$-and Frobenius norm are discussed in Section~\ref{sec:cors}.

\begin{thm}
\label{thm:rate}
Suppose Conditions \ref{cond:eigen} and \ref{cond:samplesize1} hold.
\noindent Under Condition \ref{cond:exptail} or \ref{cond:polytail},
\begin{eqnarray*}
&&\sum_{k=1}^K\lVert \hat{\Omega}_{\rho_n}^{(k)}-\Omega^{(k)}_0\rVert_2 = O_P\left(\sqrt{\frac{\lambda_{\Theta}^4(s+1)\log p}{n}}\right),
\end{eqnarray*}
as $n,p\rightarrow \infty$ where $\rho_n$ is given in Lemma~\ref{lemma:l1const} in the Appendix with $\gamma=\min_k\pi_k/2$.
\end{thm}
Theorem~\ref{thm:rate} is proved in the Appendix. The proof builds on tools from \citet{Negahban}. However, our estimation procedure does not match their general framework: First, we do not penalize the diagonal elements of the inverse correlation matrices; our penalty is thus not a norm. Second, the Laplacian matrix is nonpositive definite. Thus, the Laplacian shrinkage penalty is not strictly convex. 
The results from \citet{Negahban} are thus not directly applicable to our problem.
To establish the estimation consistency, we first show, in Lemma~\ref{lemma:const1}, that the function $r(\cdot) = \| \cdot \|_1 + \rho_2 \| \cdot \|_L$ is a seminorm, and is, moreover, convex and decomposable. We also characterize the subdifferential of this  seminorm in Lemma~\ref{lemma:sbdiff}, based on the spectral decomposition of the graph Laplacian $L$. 
The rest of the proof uses tools from \citet{Negahban}, \citet{MR2417391} and \citet{MR2836766}, as well as new inequalities and concentration  bounds. In particular, in Lemma~\ref{lemma:const4} we establish a new $\ell_\infty$ bound for the empirical covariance matrix for random variables with polynomial tails, which is used to established the consistency in the spectral norm  under Condition~\ref{cond:polytail}.

The convergence rate in Theorem~\ref{thm:rate} compares favorably to several other methods based on penalized likelihood.
Few results are currently available for estimation of multiple precision matrices. 
An exception is \citet{MR2804206}, who obtained a slower rate of convergence $O_{p}(\{(s+p)\log p/n\}^{1/2})$ under the normality assumption and based on a bound on the Frobenius norm. 
Our rates of convergence are comparable to the results of \citet{MR2417391} for spectral norm convergence of a single precision matrix, obtained under the normality assumption. 
\citet{MR2836766}, on the other hand, assumed the irrepresentability condition to obtain the rate $O_{p}(\{\min\{s+p,d^2\}\log p/n\}^{1/2})$ and $O_{p}(\{\min\{s+p,d^2\}p^{\tau/(c_2+c_3+1)}/n\}^{1/2})$, under exponential and polynomial tail conditions, respectively, where $\tau>2$ is some scalar.
The rate in Theorem~\ref{thm:rate} is obtained without assuming the irrepresentability condition. In fact, our rates of convergence are faster than those of \citet{MR2836766} given the irrepresentability condition~\ref{cond:irrep} (see Corollary \ref{col:rate}).
\citet{MR2847973} obtained improved rates of convergence under both tail conditions for an estimator that is not found by minimizing the  penalized likelihood objective function, and may be nonpositive definite. 
Finally, note that the results in \cite{MR2417391, MR2836766, MR2847973} are for separate estimation of precision matrices and hold for the minimum sample size across subpopulations, $\min_k n_k$, whereas our results hold for the total samples size $\sum_k n_k$. 

\subsection{Model Selection Consistency}\label{sec:other}
Let $S^{(k)} = \{(i,j): \omega_{0,ij}^{(k)} \neq 0,i,j=1,\ldots,p\}$ be the support of $\Omega_0^{(k)}$, and denote by $d$ the maximum number of nonzero elements in any rows of $\Omega^{(k)}_0,k=1,\ldots,K$.
Define the event
\begin{equation}
\mathcal{M}(\hat{\Omega}_{\rho_n},\Omega_0)
\equiv \left\{\mbox{sign}(\hat{\omega}_{\rho_n,ij}^{(k)}) = \mbox{sign}(\omega_{0,ij}^{(k)}),i,j=1,\ldots,p, k=1,\ldots,K\right\},
\end{equation}
where $\mbox{sign}(a)$ is $1$ if $a>0$, $0$ if $a=0$ and $-1$ if $a<0$.
We say that an estimator $\hat{\Omega}_{\rho_n}$ of $\Omega_0$ is
model-selection consistent if $P\{\mathcal{M}(\hat{\Omega}_{\rho_n},\Omega_0)\} \rightarrow 1$.

We begin by discussing an irrepresentability condition for estimation of multiple graphical models. This restrictive condition is commonly assumed to establish model selection consistency of lasso-type estimators, and is known to be almost necessary \citep{MR2278363,ZhaoYu2006}. 
For the graphical lasso, \citet{MR2836766} showed that the irrepresentability condition amounts to a constraint on the correlation between entries of the Hessian matrix $\Gamma = \Omega^{-1} \otimes \Omega^{-1}$ in the set $S$ corresponding to nonzero elements of $\Omega$, and those outside this set. 
Our irrepresentability condition is motivated by that in \citet{MR2836766}, however, we adjust the index set $S$ to also account for covariances of ``non-edge variables'' that are correlated with each other. 
More specifically, the description of irrepresentability condition in \citet{MR2836766} involves $\Gamma_{SS}$ consisting only of elements $\sigma_{ij}\sigma_{kl}$ with $(i,j)\in S$ and $(k,l)\in S$. 
However, $\sigma_{ij}\neq 0$ for $(i,j)\notin S$ is not taken into account by this definition.
We thus adjust the index set $S$ so that $\Gamma_{SS}$ also includes elements $\sigma_{ij}\sigma_{kl}$ if $(i,k)\in S$ and $(j,l)\in S$. 
This definition is based on the crucial observations that $\Gamma = \Sigma\otimes\Sigma$ involves the covariance matrix $\Sigma$ instead of the precision matrix $\Omega$, and that some variables are correlated (i.e., $\sigma_{ij}\neq 0$) even though they may be
conditionally independent (i.e., $\omega_{ij}=0$).
Defining $S^{(k)}$ for $k = 1, \ldots, K$ as above, we assume the following condition. 
\begin{cond}[Irrepresentability condition]
\label{cond:irrep}
The inverse $\Theta_0^{(k)}$ of the correlation matrix $\Psi^{(k)}_0$ satisfies the irrepresentability condition for $S^{(k)}$ with parameter $\alpha$:
(a) $(\Theta_0^{(k)}\otimes \Theta_0^{(k)})_{S^{(k)}S^{(k)}}$ and $(\Psi^{(k)}_0\otimes \Psi^{(k)}_0)_{S^{(k)}S^{(k)}}$ are invertible, 
and (b) there exists some $\alpha\in (0,1]$ such that
\begin{equation}
\label{eqn:irrep}
\max_{(i,j)\in (S^{(k)})^c} \lVert \Gamma^{(k)}_{\{(i,j)\}\times S^{(k)}} \{\Gamma^{(k)}_{S^{(k)}S^{(k)}}\}^{-1}\rVert_1 \leq 1-\alpha,
\end{equation}
for $k=1,\ldots,K$ where $\Gamma^{(k)} \equiv \Psi^{(k)}_0\otimes \Psi^{(k)}_0$.
\end{cond}

In addition to the irrepresentability condition, we require bounds on the magnitude of $\theta_{ij}^{(k)}\neq 0$ and their normalized difference.
\begin{cond}[Lower bounds for the inverse correlation matrices]
\label{cond:nullL}
There exists a constant $c_8\in\mathbb{R}$ such that
\begin{equation*}
\theta_{\min}\equiv \min_{k=1,\ldots,K,i\neq j}|\theta_{0,ij}^{(k)}| \geq c_8>0.
\end{equation*}
Moreover, for $\Omega_{0,ij}\neq 0$, $L\Omega_{0,ij}\neq 0$ and there exists a constant $c_9>0$ such that
\begin{equation*}
\min_{l_{kk'}\neq 0, \, \frac{\omega_{0,ij}^{(k)}}{\sqrt{d_k}}-\frac{\omega_{0,ij}^{(k')}}{\sqrt{d_{k'}}}\neq 0}\left|\frac{\theta_{0,ij}^{(k)}}{\sqrt{d_k}} -\frac{\theta_{0,ij}^{(k')}}{\sqrt{d_{k'}}}\right|\geq c_9.
\end{equation*}
\end{cond}
The first lower bound in Condition~\ref{cond:nullL} is the usual ``min-beta'' condition for model selection consistency of lasso-type estimators. The second lower bound, which is represented here for the normalized Laplacian penalty, is a mild condition which ensures estimates based on inverse correlation matrices can be mapped to precision matrices. 
For any pair of subpopulations $k$ and $k'$ connected in $G$ it requires that if the difference in (normalized) entries of the entires of the precision matrices are nonzero, the difference in (normalized) entries of inverse correlation matrices are bounded away from zero. In other words, the bound guarantees that $\Theta_{0,ij}$ is not in the null space of $L$, whenever $\Omega_{0,ij}$ is outside of the null space. 
This bound can be relaxed if we use a positive definite matrix $L_\epsilon =L+\epsilon I$ for $\epsilon >0$ small. 

Our last condition for establishing the model selection consistency  concerns the minimum sample size and the tuning parameter for the graph Laplacian penalty. 
This condition is necessary to control the $\ell_\infty$-bound of the error $\hat{\Theta}_{\rho_n} -\Theta_0$, as in \citet{MR2836766}. 
Our minimum sample size requirement is related to the irrepresentability condition. 
Let $\kappa_{\Gamma}$ be the maximum of the absolute column sums of the matrices $\{(\Gamma^{(k)})^{-1}\}_{S^{(k)}S^{(k)}},k=1,\ldots,K$, and  $\kappa_{\Psi}$ be the maximum of the absolute column sums of the matrices $\Psi_0^{(k)},k=1,\ldots,K$. 
The minimum sample size in \citet{MR2836766} is also a function of the irrepresentability constant, in particular, their $\kappa_{\Gamma}$ involves $\{(\Gamma^{(k)}_{S^{(k)}S^{(k)}})\}^{-1}$. 
There is, therefore, a subtle difference between our definition and theirs: in our definition, the matrix is first inverted and then partitioned, while in \citet{MR2836766}, the matrix is first partitioned and then  inverted.
Corollary~\ref{col:modelconst} establishes the model selection consistency under a weaker sample size requirement, by exploiting instead the control of the spectral norm in Theorem~\ref{thm:rate}.
\begin{cond}[Sample size and regularization parameters]
\label{cond:rhos}
Let 
\begin{equation*}
C_3 = \max\left\{\frac{2^63^4\kappa_{\Psi}^2\kappa^2_{\Gamma}}{\min_k\pi_k^2}\max\left\{1,\frac{2^67^2\kappa_{\Psi}^4\kappa_{\Gamma}^2}{\alpha^2\min_k\pi_k^2}\right\},\frac{36}{c_8^2},\frac{2^43^2}{c_9^2\min_kd_k } \right\}
\end{equation*}

(i) (Exponential tails).
It holds 
\begin{equation*}
n> \frac{12\log p}{\min_k\pi_k}\max\left\{1,2^63^2C_1^2(1+c_1^2)^2c_6^2C_3d^2\right\}
\end{equation*}

(ii) (Polynomial tails).
It holds 
$n> \max\{p^{1/c_2}c_7^{-1/c_2}, C_1^2C_2^2C_3d^2\log p\}$.

(iii) It holds that $\rho_2 \leq \alpha^2/\{4\lVert L\rVert_2^{1/2}(2-\alpha)\}.$
\end{cond}

With these condition, we obtain
\begin{thm}
\label{thm:modelconst}
Suppose that Conditions \ref{cond:eigen}, \ref{cond:irrep}, \ref{cond:nullL} and \ref{cond:rhos} hold.
Under Condition~\ref{cond:exptail} or \ref{cond:polytail}, $P(\mathcal{M}(\hat{\Omega}_{\rho_n},\Omega_0)) \rightarrow 1$ as $n,p\rightarrow \infty$ where $\rho_n$ is given in Lemma~\ref{lemma:l1const} in the Appendix with $\gamma=\min_k\pi_k/2$.
\end{thm}

\subsection{Additional Results}\label{sec:cors}
In this section, we establish norm and variable selection consistency of LASICH under alternative assumptions.
Our first result gives better rates of convergence for consistency in the $\ell_\infty$-, spectral and Frobenius norms, under the condition for model selection consistency. 
Our rates in Corollary~\ref{col:rate} improve the previous results by \citet{MR2836766}, and are comparable to that of \citet{MR2847973} in the $\ell_\infty$- and spectral norms under both tail conditions. 
\begin{col}
\label{col:rate}
Suppose the conditions in Theorem~\ref{thm:modelconst} hold.
Then, under Condition~\ref{cond:exptail} or \ref{cond:polytail},
\begin{eqnarray*}
&&\sum_{k=1}^K\lVert \hat{\Omega}_{\rho_n}^{(k)}-\Omega^{(k)}_0\rVert_F = O_P\left(\sqrt{\frac{\min\{\lambda_{\Theta}^4p(s+1),\kappa_{\Gamma}^2(s+p)\}\log p}{n}}\right),\\
&&\sum_{k=1}^K\lVert \hat{\Omega}_{\rho_n}^{(k)}-\Omega^{(k)}_0\rVert_2 = O_P\left(\sqrt{\frac{\min\{\lambda_{\Theta}^4(s+1),\kappa_{\Gamma}^2d^2\}\log p}{n}}\right),\\
&&\sum_{k=1}^K\lVert \hat{\Omega}_{\rho_n}^{(k)}-\Omega^{(k)}_0\rVert_\infty = O_P\left(\sqrt{\frac{\kappa_{\Gamma}^2\log p}{n}}\right).
\end{eqnarray*}
\end{col}

Our next result in Corollary~\ref{col:modelconst} establishes the model selection consistency under a weaker version of the irrepresentability condition (Condition~\ref{eqn:irrep}). 
Aside from the difference in the index sets $S^{(k)}$, the form of the Condition~\ref{eqn:irrep} and the assumption of invertibility of 
$(\Psi^{(k)}_0\otimes \Psi^{(k)}_0)_{S^{(k)}S^{(k)}}$ are similar to those in \citet{MR2836766}. 
On the other hand, \citet{MR2836766} do not require invertibility of $(\Theta_0^{(k)}\otimes \Theta_0^{(k)})_{S^{(k)}S^{(k)}}$. 
However, their proof is based on an application of Brouwer's fixed point theorem, which does not hold for the corresponding function (Eq. (70) in page 973) since it involves a matrix inverse, and is hence not continuous on its range. 
The additional inevitability assumption in Condition~\ref{eqn:irrep} is used to address this issue in Lemma~\ref{lemma:rvkmr6}.
The condition can be relaxed if we assume an alternative scaling of the sample size stated in Condition~\ref{cond:rhos2} below instead of  Condition~\ref{cond:rhos}.
\begin{cond}
\label{cond:rhos2}
Let $\lambda_{\Psi} = \max_k\lVert \Psi_0^{(k)}\rVert$. Suppose $\rho_2 \leq \alpha^2/\{4\lVert L\rVert_2^{1/2}(2-\alpha)\}$ and 

(i) (Exponential tails) 
\begin{equation*}
n> 2^{19}3^3\{\min_k\pi_k\}^{-3}C_1^2(1+4c_1^2)^2c_6^2\lambda_{\Theta}^4\left(1+\rho_2\lVert L\rVert_2^{1/2}\right)^2s\log p \max\{\lambda_{\Psi},4\lambda_{\Theta}^4\alpha^{-1}\}, 
\end{equation*}
or

(ii) (Polynomial tails) 
\begin{equation*}
n> 2^{12}3^3\{\min_k\pi_k\}^{-2}K^2C_1^2C_2^2\lambda_{\Theta}^4\left(1+\rho_2\lVert L\rVert_2^{1/2}\right)^2s\log p \max\{\lambda_{\Psi},4\lambda_{\Theta}^4\alpha^{-1}\}.
\end{equation*}
\end{cond}
\begin{col}
\label{col:modelconst}
Suppose that Conditions~\ref{cond:eigen}, 
\ref{cond:nullL} and \ref{cond:rhos2} hold.
Suppose also that Condition~\ref{cond:irrep} holds without requiring the invertibility of $(\Theta_0^{(k)}\otimes \Theta_0^{(k)})_{S^{(k)}S^{(k)}}$.
Then, under Condition~\ref{cond:exptail} or \ref{cond:polytail}, 
$P(\mathcal{M}(\hat{\Omega}_{\rho_n},\Omega_0)) \rightarrow 1$ as $n,p\rightarrow \infty$ where $\rho_n$ is given in Lemma \ref{lemma:l1const} in the Appendix with $\gamma=\min_k\pi_k/2$.
\end{col}

\section{Laplacian Shrinkage based on Hierarchical Clustering}\label{sec:HC}
Our proposed LASICH approach utilizes the information in the subpopulation network $G$.
In practice, however, similarity between subpopulations may be difficult to ascertain or quantify.
In this section, we present a modified LASICH framework, called HC-LASICH, which utilizes hierarchical clustering to learn the relationships among subpopulations. The information from hierarchical clustering is then used to define the weighted subpopulation network. 
Importantly, HC-LASICH can even be used in settings where the subpopulation membership is unavailable, for instance, to learn the genetic network of cancer patients, where cancer subtypes may be unknown.

We use hierarchical clustering with a complete, single or average linkage to estimate both the subpopulation memberships and the weighted subpopulation network $G$.
Specifically, the length of a path between two subpopulations in the dendrogram is used as a measure of dissimilarity between two subpopulations; the weights for the subpopulation networks are simply defined by taking the inverse of these lengths.
Throughout this section, we assume that the number of subpopulations $K$ is known. While a number of methods have been proposed for estimating the number of subpopulations in hierarchical clustering (see e.g. \citet{borysov2014} and the references therein), the problem is beyond the scope of this paper.

Let $\mathcal{I}=(\mathcal{I}^{(1)},\ldots,\mathcal{I}^{(K)})$ be the subpopulation membership indicator such that $\mathcal{I}$ follows the multinomial distribution $\textrm{Mult}_K(1,(\pi_1,\ldots,\pi_K))$ with parameter 1 and subpopulation membership probabilities $(\pi_1,\ldots,\pi_K)\in(0,1)^K$. Note that $\mathcal{I}$ is missing and is to be estimated.
Let $\mathcal{I}_i,i=1,\ldots,n$ be i.i.d. copies of $\mathcal{I}$ and $\hat{\mathcal{I}}_i = (\hat{\mathcal{I}}_i^1,\ldots,\hat{\mathcal{I}}_i^K)$ be an estimated subpopulation indicator for the $i$th observation via hierarchical clustering.
Based on the estimated subpopulation membership and subpopulation network $\hat{G}$, we apply our method to obtain the estimator, HC-LASICH, $\hat{\Omega}_{HC,\rho_n} =(\hat{\Omega}_{HC,\rho_n}^{(1)},\ldots,\hat{\Omega}_{HC,\rho_n}^{(K)})$.
Interestingly, HC-LASICH enjoys the same theoretical properties as LASICH, under the normality assumption.
To show this, we first establish the consistency of hierarchical clustering in high dimensions, which is of independent interest. Our result is motivated by the recent work of \citep{borysov2014}, who study the consistency of hierarchical clustering for independent normal variables $X^{(k)} \sim N(\mu^{(k)},\sigma^{(k)}I)$; we establish similar results for multivariate normal distributions with arbitrary covariance structures. 
We make the following assumption.
\begin{cond}
\label{cond:clst}
For $k,k'=1,\ldots,K$, let
\begin{eqnarray*}
&&\overline{\lambda}^{(k)}= p^{-1}\sum_{j=1}^p\lambda^{(k),j},\\
&&\mu^{(k,k')} = p^{-1}\left\lVert \Lambda_{k,k'}^{1/2}Q_{k,k'}^T\left[\Sigma^{(k)}+\Sigma^{(k')}\right]^{1/2}\left[\mu^{(k)}-\mu^{(k')}\right]\right\rVert^2,
\end{eqnarray*}
where $\lambda^{(k),j}$ is the eigenvalues of $\Sigma^{(k)}$ with $\lambda^{(k),1}\leq \lambda^{(k),2}\leq \ldots\leq \lambda^{(k),p}$, and
the spectral decomposition of $\Sigma^{(k)}+\Sigma^{(k')}$ is $\Sigma^{(k)}+\Sigma^{(k')}=Q_{k,k'}\Lambda_{k,k'}Q_{k,k'}^T$.
It holds that
\begin{eqnarray*}
&&\mu^{(k,k')} >2\min\left\{\overline{\lambda}^{(k)},\overline{\lambda}^{(k')}\right\} - \lambda^{(k),p}-\lambda^{(k'),p}, \quad k\neq k', \  k,k'=1,\ldots,K,\\
&& 0<c_{10}\leq \lambda^{(k),j} \leq c_{11}<\infty,\quad
\lVert\mu^{(k)}\rVert\leq  c_{11}, \quad \quad   k=1,\ldots,K, j=1,\ldots,p.
\end{eqnarray*}
for constants $m$ and $M$.
\end{cond}
Under the normality assumption, the following results shows that the probability of successful clustering converges to 1, as $p, n \rightarrow \infty$.
\begin{thm}
\label{thm:clst}
Suppose that that $X^{(k)},k=1,\ldots,K$, is normally distributed.
Under Condition \ref{cond:clst},
\begin{equation}\label{eq:clusterconst}
P(\hat{\mathcal{I}}_i = \mathcal{I}_i,i=1,\ldots,n) \rightarrow 1,
\end{equation}
as $n,p\rightarrow \infty$. 
\end{thm}
To proof of Theorem~\ref{thm:clst} generalizes recent results of \citet{borysov2014} to the case of arbitrary covariance structures. A key component of the proof is a new bound on the $\ell_2$ norm of a multivariate normal random variable with arbitrary mean and covariance matrix established in Lemma~\ref{lemma:3}. The proof of the lemma uses new concentration inequalities for high-dimensional problems in \cite{boucheron2013concentration},  and may be of independent interest. 
 
Note that the consistent estimation of subpopulation memberships \eqref{eq:clusterconst} implies that the estimated hierarchy among clusters also matches the true hierarchy.
Thus, with successful clustering established in Theorem \ref{thm:clst}, theoretical properties of $\hat{\Omega}_{HC,\rho_n}$ naturally follow. 
\begin{thm}
\label{thm:ratemodelclust}
Suppose that $X^{(k)},k=1,\ldots,K$, is normally distributed and that Condition \ref{cond:clst} holds.
\noindent (i) Under the conditions of Theorem \ref{thm:rate},
\begin{eqnarray*}
&&\sum_{k=1}^K\lVert \hat{\Omega}_{HC,\rho_n}^{(k)}-\Omega^{(k)}_0\rVert_2 = O_P\left(\sqrt{\frac{\lambda_{\Theta}^4(s+1)\log p}{n}}\right).
\end{eqnarray*}
Suppose, moreover, that the conditions of Theorem \ref{thm:modelconst} holds. Then
\begin{eqnarray*}
&&\sum_{k=1}^K\lVert \hat{\Omega}_{HC,\rho_n}^{(k)}-\Omega^{(k)}_0\rVert_F = O_P\left(\sqrt{\frac{\min\{\lambda_{\Theta}^4p(s+1),\kappa_{\Gamma}^2(s+p)\}\log p}{n}}\right),\\
&&\sum_{k=1}^K\lVert \hat{\Omega}_{HC,\rho_n}^{(k)}-\Omega^{(k)}_0\rVert_2 = O_P\left(\sqrt{\frac{\min\{\lambda_{\Theta}^4(s+1),\kappa_{\Gamma}^2d^2\}\log p}{n}}\right),\\
&&\sum_{k=1}^K\lVert \hat{\Omega}_{HC,\rho_n}^{(k)}-\Omega^{(k)}_0\rVert_\infty = O_P\left(\sqrt{\frac{\kappa_{\Gamma}^2\log p}{n}}\right).
\end{eqnarray*}

\noindent (ii) Under the conditions of Theorem \ref{thm:modelconst},
\begin{equation*}
P(\mathcal{M}(\hat{\Omega}_{HC,\rho_n},\Omega_0)) \rightarrow 1,\quad  \mbox{as }n,p\rightarrow \infty.
\end{equation*}
\end{thm}

\section{Algorithms}\label{sec:alg}
We develop an alternating directions method of multipliers (ADMM) to efficiently solve the convex optimization problem (\ref{eqn:originallasso}). 

Let $A^{(k)}=(a^{(k)}_{ij})_{i,j=1}^p\in\mathbb{R}^{p\times p}$, $B^{(k)}=(b^{(k)}_{ij})_{i,j=1}^p\in\mathbb{R}^{p\times p}$, $C^{(k)}=(c^{(k)}_{ij})_{i,j=1}^p\in\mathbb{R}^{p\times p}$, $D^{(k)}=(d^{(k)}_{ij})_{i,j=1}^p\in\mathbb{R}^{p\times p}$, $k=1,\ldots,K$.
Define $A = (A^{(1)},\ldots,A^{(K)})$, $B = (B^{(1)},\ldots,B^{(K)})$, $C = (C^{(1)},\ldots,C^{(K)})$, $D = (D^{(1)},\ldots,D^{(K)})$, and
$c_{ij} \equiv (c_{ij}^{(1)},\ldots,c_{ij}^{(K)})^T \in \mathbb{R}^K$,
$d_{ij} \equiv (d_{ij}^{(1)},\ldots,d_{ij}^{(K)})^T \in \mathbb{R}^K$,
$e_{C,ij} \equiv (e_{C,ij}^{(1)},\ldots,e_{C,ij}^{(K)})^T \in \mathbb{R}^K$
where $E_C^{(k)} = (e_{C,ij}^{(k)})_{i,j=1}^p$.

To facilitate the computation, we consider instead a perturbed graph Laplacian $L_\epsilon = L + \epsilon I$, where $I$ is the identity matrix and $\epsilon>0$ is a small perturbation.
The difference between solutions to the original and modified optimization problem is largely negligible for small $\epsilon$; however, the positive definiteness of $L_\epsilon$ results in more efficient computation. 
A similar idea was used in \citet{MR2804206} and \citet{MR2417391} to avoid dividision by zero.
The optimization problem (\ref{eqn:originallasso}) with $L$ replaced by $L_\epsilon$ can then be written as
\begin{eqnarray}
\mbox{minimize}&&
\sum_{k=1}^K  \frac{n_k}{n}\left(\mbox{tr}\left(\Psi_n^{(k)}A^{(k)}\right)-\log \mbox{det}(A^{(k)})\right)\nonumber
 + \rho_n\sum_{k=1}^K\lVert B^{(k)} \rVert_1
+  \rho_n\rho_2\sum_{i\neq j} (c_{ij}^TL_\epsilon c_{ij})^{1/2} \nonumber \\
\mbox{subject to } &&
A^{(k)} = D^{(k)},
B^{(k)} = D^{(k)},
L_\epsilon c_{ij} =L_\epsilon d_{ij}
\quad k=1,\ldots,K, i,j=1,\ldots,p.
\label{eqn:admmlasso}
\end{eqnarray}
Using Lagrange multipliers $E = (E_A,E_B,E_C)^T$, with $E_A = (E_A^{(1)},\ldots,E_A^{(K)})$ with $E_A^{(k)}\in\mathbb{R}^{p\times p},k=1,\ldots,K$, $E_B = (E_B^{(1)},\ldots,E_B^{(K)})$ with $E_B^{(k)}\in\mathbb{R}^{p\times p},k=1,\ldots,K$, and $E_C = (E_C^{(1)},\ldots,E_C^{(K)})$ with $E_C^{(k)}\in\mathbb{R}^{p\times p},k=1,\ldots,K$, the augmented Lagrangian in scaled form is given by
\begin{eqnarray*}
&&L_\varrho (A,B,C,D,E) \\
&&\equiv n^{-1}\sum_{k=1}^K  n_k\left(\mbox{tr}\left(\Psi_n^{(k)}A^{(k)}\right)-\log \mbox{det}(A^{(k)})\right)\nonumber
 + \rho_n\sum_{k=1}^K\lVert B^{(k)} \rVert_1
+  \rho_n\rho_2\sum_{i\neq j} (c_{ij}^TL_\epsilon c_{ij})^{1/2} \nonumber \\
&& \quad+ \frac{\varrho}{2}\sum_{k=1}^K\left\lVert A^{(k)}-D^{(k)} + E_A^{(k)}\right\rVert_{F}^2
 + \frac{\varrho}{2}\sum_{k=1}^K\left\lVert B^{(k)}-D^{(k)} + E_B^{(k)}\right\rVert_{F}^2\\
&& \quad+ \frac{\varrho}{2}\sum_{i,j}\left\lVert L_\epsilon^{1/2}c_{ij}-L_\epsilon^{1/2}d_{ij} + e_{C,ij}\right\rVert_{F}^2.
\end{eqnarray*}
Here $\varrho>0$ is a regularization parameter and $L_\epsilon^{1/2}$ is the square root of $L_\epsilon$ with $L_\epsilon = (L_\epsilon^{1/2} )^TL_\epsilon^{1/2}$.

The proposed ADMM algorithm is as follows.
\begin{itemize}
\item \emph{Step 0}. Initialize $A^{(k)} = A^{(k),0}$, $B^{(k)} = B^{(k),0}$, $C^{(k)} = C^{(k),0}$, $D^{(k)} = D^{(k),0}$, $E_A^{(k)} = E_A^{(k),0}$, $E_B^{(k)} = E_B^{(k),0}$, $E_C^{(k)} = E_C^{(k),0}$ and choose $\varrho>0$. Select a scalar $\varrho>0$.
\item \emph{Step $m$}. Given the $(m-1)$th estimates, 
\begin{itemize}
\item Update $A^{(k)}$) Find $A^{m}$ minimizing $-\ell_n(A) -(\varrho/2)\sum_{k=1}^K\lVert A^{(k)}-D^{(k),m-1} -E_A^{(k),m-1}\rVert$ (see pages 46-47 of \citet{BoydPCPE11} for details).
\item (Update $B^{(k)}$) Compute 
$B_{ij}^{(k),m} = S_{\rho_n/\varrho}(D^{(k),m-1}_{ij} -E_{B,ij}^{(k),m-1})$,
where
$S_y(x)$ is $x-y$ if $x>y$, is $0$ if $|x|\leq y$, and is $x+y$ if $x<-y$.
\item (Update $C^{(k)}$) For $(x)_+=\max\{x,0\}$, compute
\begin{equation*}
c_{ij}^{m} = \left(1-\frac{\rho_n\rho_2}{\varrho\lVert L_\epsilon^{1/2}d_{ij}^{m-1}-e_{C,ij}^{m-1}\rVert}\right)_+(d_{ij}^{m-1}-L_\epsilon^{-1/2}e_{C,ij}^{m-1}).
\end{equation*}
\item (Update $D^{(k)}$) Compute 
\begin{eqnarray*}
d_{ij}^{m} = (2I + L_\epsilon)^{-1}\{a_{ij}^m+e_{A,ij}^{m-1}+b_{ij}^m+e_{B,ij}^{m-1}
+L_\epsilon c_{ij}^m+(L_\epsilon^{1/2})^Te_{C,ij}^{m-1}\}.
 \end{eqnarray*}
\item (Update $E_A$) Compute $E_A^{(k),m} =  E_A^{(k)} +A^{(k),m} - D^{(k),m}$.
\item (Update $E_B$) Compute $E_B^{(k),m} =  E_B^{(k)} +B^{(k),m} - D^{(k),m}$,
\item (Update $E_C$) Compute $e_{C,ij}^{(k),m} =  e_{C,ij}^{(k)} +L^{1/2}(c^{(k),m}_{ij} - d^{(k),m}_{ij})$.
\end{itemize}
\item Repeat the iteration until the maximum of the errors $r_A^{(k)} = A^{(k)} -D^{(k)}$,
$r_B^{(k),m} = B^{(k),m} -D^{(k),m}$,
$r_C^{(k),m} = C^{(k),m} -D^{(k),m}$,
$s^{(k),m} = \varrho (D^{(k),m} - D^{(k),m-1})$ in the Frobenius norm is less than a specified tolerance level.
\end{itemize}

The proposed ADMM algorithm facilitates the estimation of  parameters of moderately large problems. However, parameter estimation in high dimensions can be computationally challenging. 
We next present a result that determines whether the solution to the optimization problem \eqref{eqn:originallasso}, for given values of tuning parameters $\rho_n, \rho_2$, is block diagonal. (Note that this result is an \emph{exact} statement about the \emph{solution} to \eqref{eqn:originallasso}, and does not assume block sparsity of the true precision matrices; see Theorems~1 and 2 of \citet{Danaher} for similar results.)
More specifically, the condition in Proposition~\ref{prob:block} provides a very fast check, based on the entries of the empirical correlation matrices $\Psi_{n}^{(k)},k=1,\ldots,K$, to identify the block sparsity pattern in $\hat{\Omega}_{\rho_n}^{(k)},k=1,\ldots,K$ after some permutation of the features. 

Let $U_L=[u_1 \ldots u_K]\in\mathbb{R}^{K\times K}$ where $u_1,\ldots,u_K$'s are eigenvectors of $L$ corresponding to $0,\lambda_{L,2},\ldots,\lambda_{L,K}$. 
Define $\Lambda^{-1/2}_L$ as the diagonal matrix with diagonal elements $0,\lambda_{L,2}^{-1/2},\ldots, \lambda_{L,K}^{-1/2}$.  
\begin{prop}\label{prob:block}
The solution $\hat{\Omega}_{\rho_n}^{(k)},k=1,\ldots,K$ to the optimization problem (\ref{eqn:originallasso}) consists of the block diagonal matrices with the same block structure $\diag(\Omega_1,\ldots,\Omega_B)$ among all groups if and only if for $\Psi_{n,ij}=(\psi_{n,ij}^{(1)},\ldots,\psi_{n,ij}^{(K)})^T$
\begin{equation}
\label{eqn:blck}
\min_{v \in [-1,1]^K} \left\lVert \Lambda^{-1/2}_LU_L\left(\frac{n_k}{n}\Psi_{n,ij}-\rho_n v \right)\right\rVert \leq \rho_n\rho_2,
\end{equation}
and for all $i,j$ such that the $(i,j)$ element is outside the blocks.
\end{prop}
The proof of the Proposition is similar to Theorems 1 of \citet{Danaher} and is hence omitted. 
Condition~\ref{eqn:blck} can be easily verified by applying quadratic programming to the left hand side of the inequality.
The solution to \eqref{eqn:originallasso} can then be equivalently found by solving the optimization problem separately for each of the blocks; this can result in significant computational advantages for moderate to large values of $\rho_n \rho_2$.

\section{Numerical Results}\label{sec:NumRes}
\subsection{Simulation Experiments}\label{sec:sims}
We compare our method with four existing methods, graphical lasso, the method of \citet{MR2804206}, FGL and GGL of \citet{Danaher}.
For graphical lasso, estimation was carried out separately for each group
with the same regularization parameter.

Our simulation setting is motivated by estimation of gene networks
for healthy subjects and patients with two similar diseases caused by
inactivation of certain biological pathways.
We consider $K=3$ groups with sample sizes $n=(50,100,50)$ 
and dimension $p=100$. 
Data are generated from multivariate normal distributions 
$N(\mu^{(k)},(\Omega^{(k)}_0)^{-1}),k=1,2,3$; all precision matrices $\Omega^{(k)}_0$ are block diagonal with 4 blocks of equal size.

To create the precision matrices, we first generated a graph with 4 components of equal size, each as either an Erd\H{o}s-R\'{e}nyi or scale free graphs with $95$ total edges.
We randomly assigned $\textrm{Unif}((-.7,-.5)\cup (.5,.7))$ values to nonzero entries of the corresponding adjacency matrix $A$ and obtained a matrix $\tilde{A}$.
We then added $0.1$ to the diagonal of $\tilde{A}$ to obtain a positive definite matrix $\Omega_0^{(1)}$.
For each of subpopulations 2 and 3, we removed one of the components of the graph by setting the off diagonal entries of $\tilde{A}$ to zero, and added a perturbation from $\textrm{Unif}(-.2,.2)$ to nonzero entries in $\tilde{A}$.
Positive definite matrices $\Omega_0^{(2)}$ and $\Omega_0^{(3)}$ were obtained by adding $0.1$ to the diagonal elements. 
All partial correlations ranges from .28 to .54 in the absolute values.
A similar setting was considered in in \citet{Danaher}, where the graph included more components, but no perturbation was added. 
We consider two simulation settings, with \emph{known} and \emph{unknown} subpopulation network $G$.

\subsubsection{Known subpopulation network $G$}
In this case, we set $\mu^{(k)} = 0,k=1,2,3$ and use the graph in Figure~\ref{fig:graphgroup} as the subpopulation network. 

Figures~\ref{fig:nocluster}a,c show the average number of true positive edges versus the average number of detected edges over 50 simulated data sets.
Results for multiple choices of the second tuning parameter are presented for FGL, GGL and LASICH. 
It can be seen that in both cases, LASICH outperforms other methods, when using relatively large values of $\rho_2$. Smaller values of $\rho_2$, on the other hand, give similar results as other methods of joint estimation of multiple graphical models. 
These results indicate that, when the available subpopulation network is informative, the Laplacian shrinkage constraint can result in significant improvement  in estimation of the underlying network. 

\begin{figure}[htbp]
\centering 
\includegraphics[width=.85\textwidth]{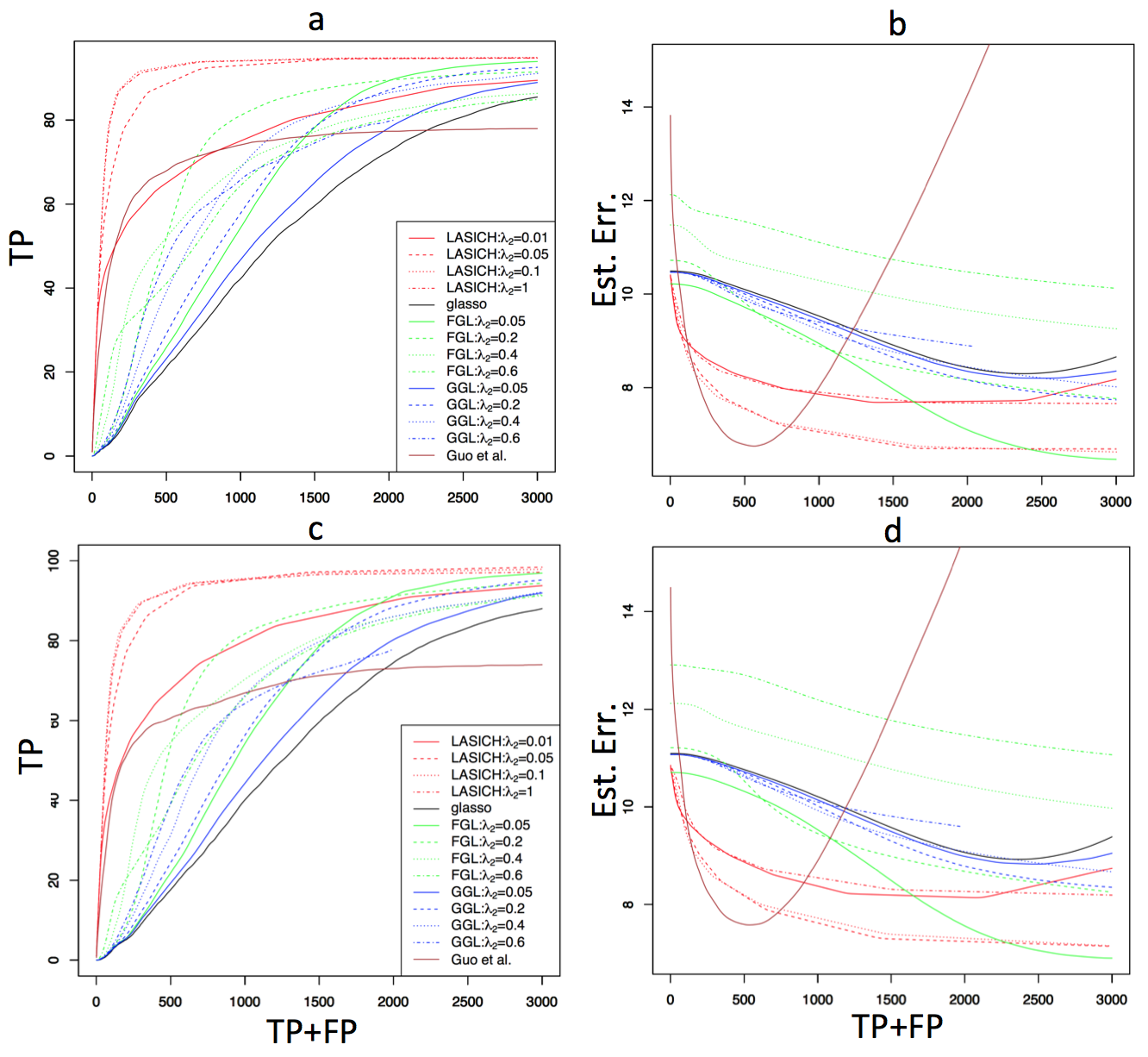}
\caption{Simulation results for joint estimation of multiple precision matrices with known subpopulation memberships. Results show the average number of true positive edges (\textbf{a} \& \textbf{c}) and estimation error, in Frobenius norm (\textbf{b} \& \textbf{d}) over 50 data sets with $n=200$ multivariate normal observations generated from a graphical model with $p=100$ features; results in top row (\textbf{a} \& \textbf{b}) are for an Erd\H{o}s-R\'{e}nyi graph and those in bottom row (\textbf{c} \& \textbf{d}) are for a scale free (power-law) graph.}
\label{fig:nocluster}
\end{figure}
Figures~\ref{fig:nocluster}b,d show the estimation error, in Frobenius norm, versus the number of detected edges. 
LASICH has larger errors when the estimated graphs have very few edges, but, its error decreases as the number of detected edges increase, eventually yielding smaller errors than other methods. 
The non-convex penalty of \citet{MR2804206} performs well in terms of estimation error, although determining the appropriate range of tuning parameter for this method may be difficult.

\subsubsection{Unknown subpopulation network $G$}
In this case, the subpopulation memberships and the subpopulation network $G$ are estimated based on hierarchical clustering. 
We randomly generated $\mu^{(1)}$ from a multivariate normal distribution with a covariance matrix $\sigma^2 I$.
For subpopulations 2 and 3, the elements of $\mu^{(1)}$ corresponding to the empty components of the graph were set to zero to obtain $\mu^{(2)}$ and $\mu^{(3)}$. 
Hierarchical clustering with complete linkage was applied to data to obtain the dendrogram; we took inverse of distances in the dendrogram to obtain similarity weights used in the graph Laplacian.

\begin{figure}[t]
  \begin{center}
    \includegraphics[width=.45\textwidth]{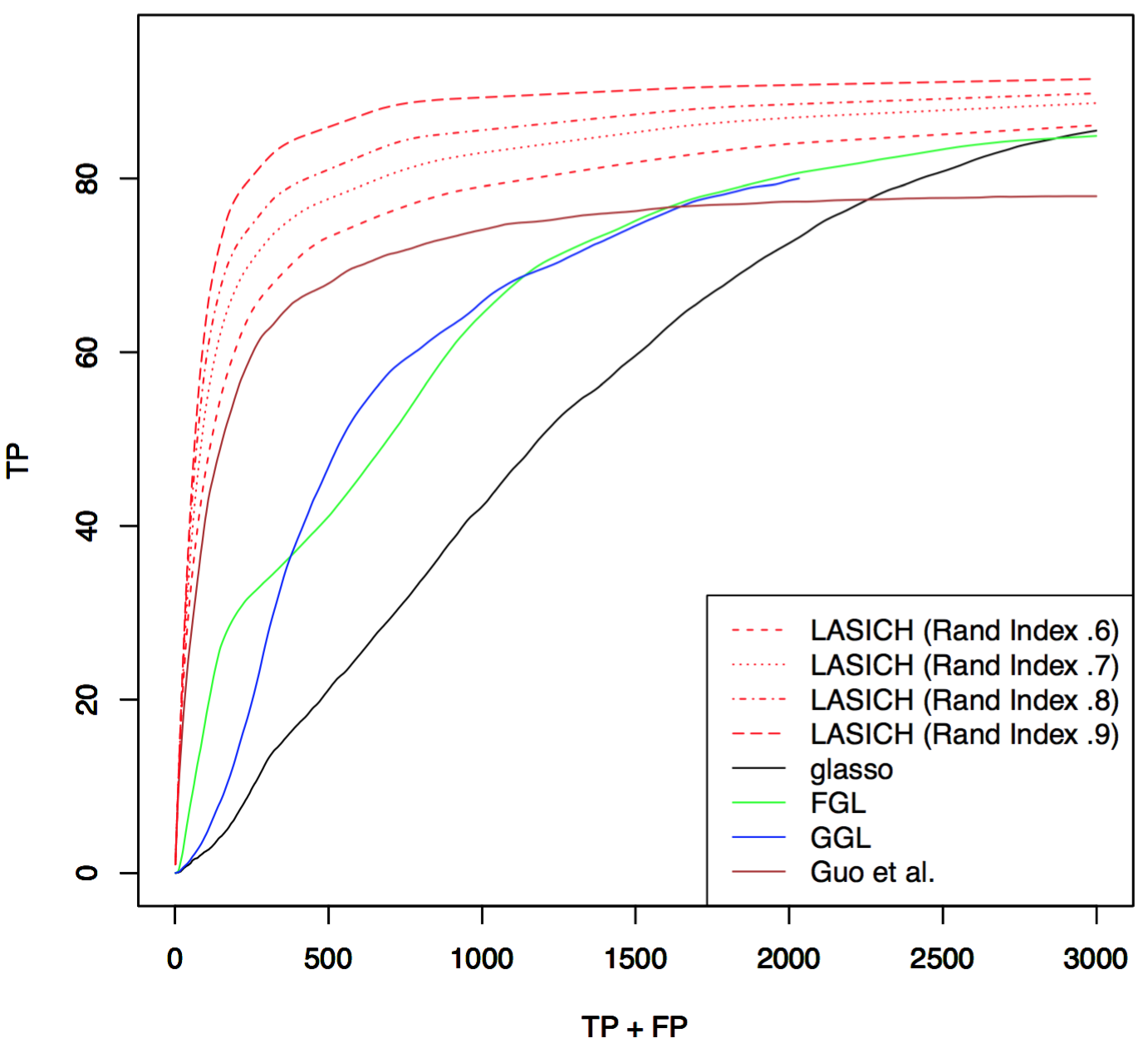}
  \end{center}
\caption{Simulation results for joint estimation of multiple precision matrices with unknown subpopulation memberships. Results show the average number of true positive edges over 50 data sets with $n=200$ multivariate normal observations generated from a graphical model with over an Erd\H{o}s-R\'{e}nyi graph with $p=100$ features. Results for HC-LASICH and FGL/GGL correspond to the best choice of the second tuning parameter among those in Figure~\ref{fig:nocluster}a. The Rand indices for HC-LASICH are averages over 50 generated data sets.}
\label{fig:cluster}
\end{figure}
Figures~\ref{fig:cluster} compares the performance of HC-LASICH, in terms of support recovery, to competing methods, in the setting where the subpopulation memberships and network are estimated from data (Section~\ref{sec:HC}). 
Here the differences in subpopulation means $\mu^{(k,k')}$ are set up to evaluate the effect of clustering accuracy. The four settings considered correspond to average Rand indices of .6 .7, .8 and .9 across 50 data sets, respectively.
Here the second tuning parameter for HC-LASICH, GGL and FGL is chosen according to the best performing model in Figure~\ref{fig:nocluster}. 
As expected, changing the mean structure, and correspondingly the Rand index, does not affect the performance of other methods. 
The results indicate that, as long as features can be clustered in a meaningful way, HC-LASICH can result in improved support recovery. 
Data-adaptive choices of the tuning parameter corresponding to the Laplacian shrinkage penalty may result in further 
improvements in the performance of the HC-LASICH. However, we do not pursue such choices here.

\subsection{Genetic Networks of Cancer Subtypes}
Breast cancer is heterogenous with multiple clinically verified subtypes \citep{perou2000}. 
\citet{pmid20576095} used copy number variation and gene expression measurements to identify new subtypes of breast cancer and showed that the identified subtypes have distinct clinical outcomes. 
The genetic networks of these different subtypes are expected to share similarities, but to also have unique features. Moreover, the similarities among the networks are expected to corroborate with the clustering of the subtypes based on their molecular profiles. 
We applied network estimation methods of Section~\ref{sec:sims} to a subset of the microarray gene expression data from \citet{pmid20576095}, containing data for 218 patients classified into three previously known subtypes of breast cancer: 46  Luminal-simple, 105 Luminal-complex and 67 Basal-complex samples. For ease of presentation, we focused on 50 genes with largest variances. The hierarchical clustering results of \citet{pmid20576095}, reproduced in Figure~\ref{fig:dendrogram} for the above three subtypes, were used to identify the subpopulation membership; reciprocals of distances in the dendrogram were used to define similarities among subtypes used in the graph Laplacian penalty. 

To facilitate the comparison, tuning parameters were selected such that the estimated networks of the three subtypes using each method contained a total of 150 edges. 
For methods with two tuning parameters, pairs of tuning parameters were determined using the Bayesian information criterion (BIC), as described in \citet{MR2804206}.
Estimated genetic networks of the three cancer subtypes are shown in Figure~\ref{fig:dendrogram}. For each method, edges common in all three subtypes, those common in Luminal subtypes and subtype specific edges are distinguished. 

In this example, results from separate graphical lasso estimation and FGL/GGL are two extremes.
Estimated network topologies from graphical lasso vary from subtype to subtype, and common structures are obscured; this variability may be because similarities among subtypes are not incorporated in the estimation.
In contrast, FGL and GGL give identical networks for all subtypes, perhaps because both methods encourage the estimated networks of all subtypes to be equally similar.
Intermediate results are obtained using LASICH and the method of \citet{MR2804206}.
The main difference between these two methods is that \citet{MR2804206} finds more edges common to all three subtypes, whereas LASICH finds more edges common to the Luminal subtypes. 
This difference is likely because LASICH prioritizes the similarity between the Luminal subtypes via graph Laplacian while the method of \citet{MR2804206} does not distinguish between the three subtypes. 
The above example highlights the potential advantages of LASICH in providing network estimates that better corroborate with the known hierarchy of subpopulations.
\begin{figure}
\centering
\includegraphics[width=0.65\textwidth]{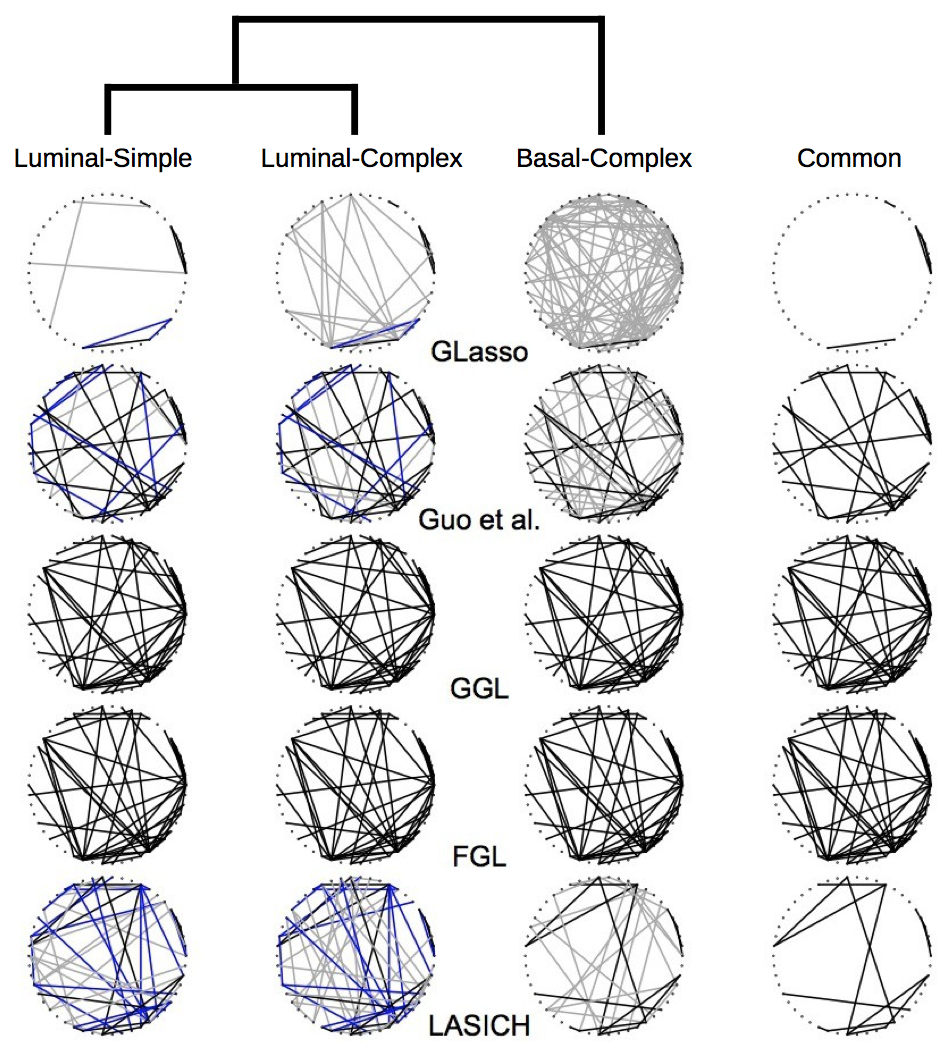}
\caption{Dendrogram of hierarchical clustering of three subtypes of breast cancer from J{\"o}nsson et al. (2010) along with estimated gene networks using graphical lasso (Glasso), method of Guo et al., FGL and GGL of Daneher et al. (2014) and LASICH. 
Blue edges are common to Luminal subtypes and black edges are shared by all three subtypes; condition specific edges are drawn in gray.} 
\label{fig:dendrogram}
\end{figure}

\section{Discussion}\label{sec:conc}
We introduced a flexible method for joint estimation of multiple precision matrices, called LASICH, which is particularly suited for settings where observations belong to three or more subpopulations. 
In the proposed method, the relationships among heterogenous subpopulations is captured by a weighted network, whose nodes correspond to subpopulations, and whose edges capture their similarities. 
As a result, LASICH can model complex relationships among subpopulations, defined, for example, based on hierarchical clustering of samples.

We established asymptotic properties of the proposed estimator in the setting where the relationship among subpopulations is externally defined. We also extended the method to the setting of unknown relationships among subpopulations, by showing that clusters estimated from the data can accurately capture the true relationships. 
The proposed method generalizes existing convex penalties for joint estimation of graphical models, and can be particularly advantageous in settings with multiple subpopulations. 

A particularly appealing feature of the proposed extension of LASICH is that it can also be applied in settings where the subpopulation memberships are unknown. The latter setting is closely related to estimation of precision matrices for mixture of Gaussian distributions. 
Both approaches have limitations and drawbacks: on the one hand, the extension of LASICH to unknown subpopulation memberships requires certain assumptions on differences of population means (Section~\ref{sec:HC}). On the other hand, estimation of precision matrices for mixture of Gaussians is computationally challenging, and known rates of convergence of parameter estimation in mixture distributions (e.g. in \citet{stadler2010}) are considerably slower.

Throughout this paper we assumed that the number of subpopulations is known. Extensions of this method to estimation of graphical models in populations with an unknown number of subpopulations would be particularly interesting for analysis of genetic networks associated with heterogeneity in cancer samples, and are left for future research.

\section{Appendix: Proofs and Technical Detials}\label{sec:appendix}
We denote true inverse correlation matrices as $\Theta_0 =(\Theta^{(1)}_0,\ldots,\Theta_0^{(K)})$ and true correlation matrices as $\Psi_0 =(\Psi^{(1)}_0,\ldots,\Psi_0^{(K)})$, where $\Theta_0^{(k)}\equiv (\Psi_0^{(k)})^{-1}\equiv (\theta^{(k)}_{0,ij})_{i,j=1}^p$, and $\Psi_0^{(k)}=(\psi_{0,ij}^{(k)})_{i,j=1}^p$.
The estimates of the population parameters are dented as $\hat{\Sigma}^{(k)}_n = (\hat{\sigma}_{ij})_{i,j=1}^p$, $\Psi^{(k)}_n = (\psi_{n,ij})_{i,j=1}^p$, and $\hat{\Theta}^{(k)}_{\rho_n} = (\hat{\theta}_{\rho_n,ij}^{(k)})_{i,j=1}^p$.
For a vector $x=(x_1,\ldots,x_p)^T$ and $J\subset \{1,\ldots,p\}$, 
we denote $x_J = (x_j, j \in J)^T$. 
For a matrix $A$, $\lambda_k(A)$ is the $k$th smallest eigenvalue and $\vec{A}$ is the vectorization of $A$.
For $J \subset \{(i,j):i,j=1,\ldots,p\}$ and $A\in\mathbb{R}^{p\times p}$,
$\vec{A}_J$ is a vector in $\mathbb{R}^{|J|}$ obtained by removing elements corresponding to $(i,j)\notin J$ from $\vec{A}$.
A zero-filled matrix $A_J \in\mathbb{R}^{p\times p}$ is obtained from $A$ by replacing $a_{ij}$ by 0 for $(i,j)\notin J$. 

\subsection{Consistency in Matrix Norms}
Theorem~\ref{thm:rate} is a direct consequence of the following result.
\begin{lemma}\label{lemma:l1const}
(i) Suppose that Condition \ref{cond:exptail} holds.
Let $\gamma\in(0,\min_k\pi_k)$ be arbitrary.
For
\begin{equation*}
n\geq\max\left\{6\gamma^{-1}\log p, \, 2^{15}3^3C_1^2 \gamma^{-3}(1+4c_1^2)^2 \max_{k,i}\{\sigma_{ii}^{(k)}\}^2
\lambda_{\Theta}^4 
\left(1+\rho_2\lVert L\rVert_2^{1/2}\right)^2s\log p\right\}
\end{equation*}
and $\rho_n = 2^3\sqrt{6}C_1(1+4c_1^2)\gamma^{-1/2}\max_{k,i}\sigma_{ii}^{(k)}\sqrt{\log p/n}$,
we have with probability $(1-2K/p)(1-2K\exp(-2n(\min_k\pi_k-\gamma)^2)) $ that
\begin{eqnarray*}
\sum_{k=1}^K\lVert \hat{\Theta}_{\rho_n}^{(k)}-\Theta_0^{(k)}\rVert_F
\leq 2^{15/2}3^{3/2}C_1\gamma^{-3/2}(1+4c_1^2)\max_{k,i}\sigma_{ii}^{(k)}
\lambda_{\Theta}^2 
\left(1+\rho_2\lVert L\rVert_2^{1/2}\right)\sqrt{\frac{ s\log p}{n}}.
\end{eqnarray*}

(ii) Suppose that Condition \ref{cond:polytail} holds with $p\leq c_7n^{c_2}$, $c_2,c_3,c_7>0$.
For $\rho_n = C_1K\delta_n$ satisfying
\begin{equation*}
2^43^2C_1\rho_n^2\gamma^{-2}s(1+\rho_2\lVert L\rVert_2^{1/2})^2
\lambda_{\Theta}^4 
\leq 1/4
\end{equation*}
and $\tau > (2^7+2^3\sqrt{1+2^{4}3^2c_4\max_{k,i}\{\sigma_{ii}^{(k)}\}^2})/(9c_4\max_{k,i}\{\sigma_{ii}^{(k)}\}^2)$
we have with probability $(1-2K\exp(-2n(\min_k\pi_k-\gamma)^2))\nu_n  $
that
\begin{eqnarray*}
\sum_{k=1}^K\lVert \hat{\Theta}_{\rho_n}^{(k)}-\Theta_0^{(k)}\rVert_F
\leq 2^43^{3/2}C_1 \gamma^{-2}K\left(1+\rho_2\lVert L\rVert_2^{1/2}\right)
\lambda_{\Theta}^2 
s^{1/2} \delta_n,
\end{eqnarray*}
where
\begin{eqnarray*}
\delta_n
&\equiv& \max_{k,i}\{\sigma_{ii}^{(k)}\}^2c_4(4+\tau)\gamma^{-1}\frac{\log p}{n}+(1+2\max_{k,i}|\mu^{(k),i}|)\sqrt{\max_{k,i}\{\sigma_{ii}^{(k)}\}^2c_4(4+\tau)\gamma^{-1}\frac{\log p}{n}}\\
&&+2\max_{k,i,j}\E|X^{(k),i}X^{(k),j}|I\left(|X^{(k),i}X^{(k),j}|\geq \sqrt{\frac{\gamma n}{\log p}}\right) + 4\left\{\max_{k,i} \E|X^{(k),i}|I\left(|X^{(k),i}|\geq \sqrt{\frac{\gamma n}{\log p}}\right)\right\}^2 \\
&&+ 2(1+2\max_{k,i}|\mu^{(k),i}|)\max_{k,i} \E|X^{(k),i}|I\left(|X^{(k),i}|\geq \sqrt{\frac{\gamma n}{\log p}}\right)\\
&& = O\left( \sqrt{\frac{\log p}{n}}\right),
\end{eqnarray*}
and
\begin{eqnarray*}
\nu_n
&\equiv& \frac{3c_7c_4\max_{k,i}\{\sigma_{ii}^{(k)}\}^2(\log p)^{c_2+c_3+1}}{\gamma^{c_3}n^{c_3}}
+  \frac{c_7c_4\max_{k,i}\sigma_{ii}^{(k)}(\log p)^{2(c_2+c_3+1)}}{n^{c_2+c_3+1}}\\
&&+ 8p^2 \exp\left(-\frac{\max_{k,i}\sigma_{ii}^{(k)}c_4(4+\tau)\log p}{2\max_{k,i}\sigma_{ii}^{(k)}c_4+ \sqrt{\max_{k,i}\{\sigma_{ii}^{(k)}\}^2c_4(64+16\tau)}/3} \right)\\
&& = o(1).
\end{eqnarray*}
\end{lemma}

Our proofs adopt several tools from \citet{Negahban}. Note however that our penalty does not penalize the diagonal elements, and is hence a seminorm; thus, their results do not apply to our case.
We first introduce several notations.
To treat multiple precision matrices in a unified way, our parameter space is defined to be the set $\tilde{\mathbb{R}}^{(pK)\times (pK)}$ of $(pK)\times (pK)$ symmetric block diagonal matrices, where the $k$th diagonal block is a $p\times p$ matrix corresponding to the precision matrix of subpopulation $k$.
We write $A \in \tilde{\mathbb{R}}^{(pK)\times (pK)}$ for a $K$-tuple $(A^{(k)})_{k=1}^K$ of diagonal blocks $A^{(k)}\in\mathbb{R}^{p\times p}$.
Note that for $A,B\in\tilde{\mathbb{R}}^{(pK)\times (pK)}$,  
$\langle A,B\rangle_{pK} =\sum_{k=1}^K \langle A^{(k)},B^{(k)}\rangle_p$ where $\langle \cdot,\cdot\rangle_p$ is the trace inner product on $\mathbb{R}^{p\times p}$.
In this parameter space, we evaluate the following map from $\tilde{\mathbb{R}}^{(pK)\times(pK)} $ to $\mathbb{R}$ given by
\begin{equation*}
f(\Delta) = -\tilde\ell_n(\Theta_0+\Delta) + \tilde\ell_n(\Theta_0) + \rho_n \{r(\Theta_0 + \Delta) -r(\Theta_0)\},
\end{equation*}
where $r:\tilde{\mathbb{R}}^{(pK)\times (pK)}\mapsto \mathbb{R}$ is given by $r(\Theta) = \lVert \Theta\rVert_1 +\rho_2\lVert \Theta\rVert_L$.
This map provides information on the behavior of our criterion function in the neighborhood of $\Theta_0$. A similar map with a different penalty was studied in \citet{MR2417391}.
A key observation is that $f(0) = 0$ and $f(\hat{\Delta}_n) \leq 0$ where $\hat{\Delta}_n = \hat{\Theta}_{\rho_n}-\Theta_0$.

The following lemma provides a non-asymptotic bound on the Frobenius norm of $\Delta$ (see Lemma 4 in \citet{NegahbanSupp} for a similar lemma in a different context).
Let $S= \cup_{k=1}^K S^{(k)}$ be the union of the supports of $\Omega_{0}^{(k)}$.
Define a model subspace $\mathcal{M} = \{\Omega\in\tilde{\mathbb{R}}^{(pK)\times (pK)}:\omega_{ij}^{(k)}=0,(i,j)\notin S,k=1,\ldots,K\}$ and its orthocomplement
$\mathcal{M}^\perp = \{\Omega\in\tilde{\mathbb{R}}^{(pK)\times (pK)}: \omega_{ij}^{(k)}=0,(i,j)\in S, k=1,\ldots K\}$ under the trace inner product in $\tilde{\mathbb{R}}^{(pK)\times (pK)}$.
For $A=(a_{ij})_{i,j=1}^{pK}\in\tilde{\mathbb{R}}^{(pK)\times (pK)}$, we write $A=A_{\mathcal{M}}+A_{\mathcal{M}^\perp}$ where $A_{\mathcal{M}}$ and $A_{\mathcal{M}^\perp}$ are the projection of $A$ into $\mathcal{M}$ and $\mathcal{M}^\perp$, in the Frobenius norm, respectively. In other words, the $(i,j)$-element of $A_{\mathcal{M}}$ is $a_{ij}$ if $(i,j)\in S$ and zero otherwise, and the $(i,j)$-element of $A_{\mathcal{M}^\perp}$ is $a_{ij}$ if $(i,j)\notin S$ and zero otherwise.
Note that $\Theta_0\in \mathcal{M}$.
Define the set $\mathcal{C} = \{\Delta\in\tilde{\mathbb{R}}^{(pK)\times (pK)}: r(\Delta_{\mathcal{M}^\perp})\leq 3 r(\Delta_{\mathcal{M}})\}$.

\begin{lemma}
\label{lemma:NegLemma4}
Let $\epsilon >0$ be arbitrary.
Suppose $\rho_n \geq 2\max_{1\leq k\leq K}\lVert \hat{\Psi}_n^{(k)} -\Psi_0^{(k)}\rVert_{\infty}$.
If $f(\Delta)>0$ for all elements $\Delta\in \mathcal{C}\cap \{\Delta\in\tilde{\mathbb{R}}^{(pK)\times (pK)}:\lVert \Delta \rVert_F = \epsilon\}$ then $\lVert \hat{\Delta}_n\rVert_F\leq \epsilon$.
\end{lemma}
\begin{proof}
We first show that $\hat{\Delta}_n \in \mathcal{C}$.
We have by the convexity of $-\tilde{\ell}_n(\Theta)$ that
\begin{equation*}
-\tilde{\ell}_n(\Theta_0+\hat{\Delta}_n) +\tilde{\ell}_n(\Theta_0)
\geq -|\langle -\nabla \tilde{\ell}_n(\Theta_0) ,\hat{\Delta}_n\rangle |.
\end{equation*}
It follows from Lemma \ref{lemma:const1}(iv) with our choice $\rho_n$ that the right hand side of the inequality is further bounded below by $-2^{-1}\rho_n\left(r(\hat{\Delta}_{n,\mathcal{M}})+r(\hat{\Delta}_{n,\mathcal{M}^\perp})\right)$.
Applying Lemma \ref{lemma:const1}(iii), we obtain
\begin{eqnarray*}
0&\geq& f(\hat{\Delta}_n)=-\tilde{\ell}_n(\Theta_0+\hat{\Delta}_n) +\tilde{\ell}_n(\Theta_0)  +r(\Theta_0+\hat{\Delta}_n) -r(\Theta_0)\\
&\geq& \frac{\rho_n}{2}r(\hat{\Delta}_{n,\mathcal{M}^\perp}) -\frac{3\rho_n}{2}r(\hat{\Delta}_{n,\mathcal{M}}),
\end{eqnarray*}
or $r(\hat{\Delta}_{n,\mathcal{M}^\perp})\leq 3r(\hat{\Delta}_{n,\mathcal{M}})$. This verifies $\hat{\Delta}_n \in \mathcal{C}$.
Note that $f$, as a function of $\Delta$ is sum of two convex functions $\ell_n$ and $r$, and is hence convex. 
Thus, the rest of the proof follows exactly as Lemma 4 in \citet{NegahbanSupp}.
\end{proof}

\begin{lemma}
\label{lemma:const1}
Let $\Delta \in \tilde{\mathbb{R}}^{(pK)\times (pK)}$.

\noindent (i) The gradient of $\tilde{\ell}_n(\Theta_0)$ is a block diagonal matrix given by
\begin{eqnarray}
\nabla \tilde{\ell}_n (\Theta_0)
= n^{-1}\diag\{n_1( \Psi_0^{(1)}-\hat{\Psi}_n^{(1)}),\ldots,n_K(\Psi_0^{(K)}-\hat{\Psi}_n^{(K)} ) \}.
\end{eqnarray}

\noindent (ii) Let $c>0$ be a constant. For $\lVert \Delta \rVert_F \leq c$  and $n_k/n\geq \gamma>0$ for all $k$ and $n$,
\begin{eqnarray}
\label{eqn:loglikderiv}
 -\tilde{\ell}_n (\Theta_0+\Delta) + \tilde{\ell}_n (\Theta_0) +\langle \nabla \tilde{\ell}_n(\Theta_0),\Delta\rangle
\geq \frac{\gamma}{2\left\{
\lambda_{\Theta} 
+  c\right\}^2}\lVert \Delta\rVert_F^2
\equiv \kappa_{\ell_n,c}\lVert \Delta\rVert_F^2.
\end{eqnarray}

\noindent (iii) The map $r$ is a seminorm, convex, and decomposable with respect to $(\mathcal{M},\mathcal{M}^\perp)$ in the sense that $r(\Theta_1+\Theta_2) = r(\Theta_1) +r(\Theta_2)$ for every $\Theta_1\in\mathcal{M}$ and $\Theta_2\in\mathcal{M}^\perp$.
Moreover,
\begin{equation*}
r(\Theta_0+\Delta) -r(\Theta_0) \geq r(\Delta_{\mathcal{M}^\perp}) -r(\Delta_{\mathcal{M}}).
\end{equation*}

\noindent (iv) For $\Delta\in\tilde{\mathbb{R}}^{(pK)\times (pK)}$,
\begin{equation}
\label{eqn:uplambda}
|\langle\nabla \tilde{\ell}_n(\Theta_0),\Delta\rangle|
\leq r(\Delta) \max_{1\leq k \leq K}\lVert \hat{\Psi}_n^{(k)}-\Psi_0^{(k)}\rVert_\infty.
\end{equation}

\noindent (v) For $\Theta\in\tilde{\mathbb{R}}^{(pK)\times (pK)}$,
\begin{equation*}
r(\Theta_{\mathcal{M}}) \leq (s+1)^{1/2}\left(1+\rho_2\lVert L\rVert_2^{1/2}\right)\lVert \Theta_{\mathcal{M}}\rVert_F.
\end{equation*}
\end{lemma}

\begin{proof}
(i) The result follows by taking derivatives blockwise.

(ii) \citet{MR2417391} (page 500-502) showed that
\begin{eqnarray*}
&& -\tilde{\ell}_n(\Theta_0+\Delta) +\tilde{\ell}_n(\Theta_0) - \langle -\nabla \tilde{\ell}_n(\Theta_0),\Delta\rangle\\
&&=\sum_{k=1}^K\frac{n_k}{n}\left(-\log \mbox{det}(\Theta_0^{(k)}+\Delta^{(k)}) +\log \mbox{det}(\Theta_0^{(k)}) + \langle \Psi_0^{(k)},\Delta^{(k)}\rangle\right)\\
&&\geq \sum_{k=1}^K\frac{n_k}{n}\frac{\lVert \Delta^{(k)}\rVert_F^2}{ 2\min_{0\leq v\leq 1} \left\{\left\lVert \Theta_0^{(k)}\right\rVert_2+v\left\lVert\Delta^{(k)} \right\rVert_2\right\}^{2}}.
\end{eqnarray*}
Since $\lVert A\rVert_2 \leq \lVert A\rVert_F$, $n_k/n\geq \gamma$ and $\lVert \Delta\rVert_F \leq c$, this is further bounded below by
\begin{eqnarray*}
\sum_{k=1}^K \frac{\gamma}{2}\frac{\lVert \Delta^{(k)}\rVert_F^2}{\left\{\lVert \Theta_0^{(k)}\rVert_2 +  \left\lVert\Delta^{(k)} \right\rVert_F\right\}^2}
\geq  \kappa_{\ell_n,c}\lVert \Delta\rVert_F^2 .
\end{eqnarray*}

(iii) Because the graph Laplacian $L$ is a positive semidefinite matrix, the triangle inequality $r(\Theta_1+\Theta_2)\leq r(\Theta_1) + r(\Theta_2)$ holds. To see this let $L=\tilde{L}\tilde{L}^T$ be any Cholesky decomposition of $L$. Then
\begin{equation*}
\{(x+y)^TL(x+y)\}^{1/2}
= \lVert \tilde{L}^T(x +y)\rVert
\leq \lVert \tilde{L}^Tx\rVert + \lVert\tilde{L}y\rVert  = \{x^TLx\}^{1/2} + \{y^TLy\}^{1/2}.
\end{equation*}
It is clear that $r(c\Theta)=cr(\Theta)$ for any constant $c$. Thus, given that $r$ does not penalize the diagonal elements, it is a seminorm. 
The decomposability follows from the definition of $r$.
The convexity follows from the same argument for the triangle inequality.
Since $\Theta_0 +\Delta = \Theta_0 + \Delta_{\mathcal{M}}+ \Delta_{\mathcal{M}^\perp}$, the triangle inequality and the decomposability of $r$ yield
\begin{eqnarray*}
r(\Theta_0 +\Delta) -r(\Theta_0)
\geq r(\Theta_0+\Delta_{\mathcal{M}^\perp}) -r(\Delta_{\mathcal{M}})-r(\Theta_0)
= r(\Delta_{\mathcal{M}^\perp}) -r(\Delta_{\mathcal{M}}).
\end{eqnarray*}

(iv) We show that, for $A,B\in\tilde{\mathbb{R}}^{(pK)\times (pK)}$ with $\diag(B)=0$, $\langle A,B\rangle \leq r(A)\lVert B\rVert_\infty$. 
If $A$ is a diagonal matrix (or if $A=0$), the inequality trivially holds since $\langle A,B\rangle=0$.
If not, $r(A)\neq 0$ so that 
\begin{eqnarray*}
\frac{\langle A,B\rangle}{r(A)}
 \leq \frac{\lVert A\rVert_1\lVert B\rVert_\infty}{\lVert A \rVert_1 }
 = \lVert B\rVert_\infty.
\end{eqnarray*}
Since the diagonal elements of $\nabla \tilde{\ell}_n(\Theta_0)$ are all zero, the result follows.

(v) For $s\neq 0$, we have
\begin{eqnarray*}
\frac{r(\Theta_{\mathcal{M}})}{\lVert \Theta_{\mathcal{M}}\rVert_F}
&\leq &\sup_{\Theta \in\mathcal{M}}\frac{\sum_{k=1}^K\lVert \Theta^{(k)} \rVert_1}{\lVert \Theta-\diag(\Theta)\rVert_F} + \sup_{\Theta \in\mathcal{M}}\frac{\rho_2\sum_{i\neq j} \sqrt{\theta_{ij}^TL\theta_{ij} }}{\lVert \Theta\rVert_F}\\
&\leq& s^{1/2} + \rho_2\sup_{\Theta \in\mathcal{M}}\frac{\sum_{i\neq j} \sqrt{\lVert L\rVert_2 \lVert \theta_{ij}\rVert_F^2 }}{\lVert \Theta\rVert_F}  \\
&\leq  &s^{1/2}\left(1+\rho_2\lVert L\rVert_2^{1/2}\right).
\end{eqnarray*}
In the last inequality we used the fact that $\sqrt{\sum_{j=1}^J\sum_{i=1}^{I}a_{ij}^2} \geq J^{-1/2}\sum_{j=1}^J\sqrt{\sum_{i=1}^{I}a_{ij}^2}$, which follows by the concavity of the square root function.
For $s=0$, we trivially have $0=r(\Theta_{\mathcal{M}})\leq s^{1/2}\{1+\rho_2\lVert L\rVert_2^{1/2}\}\lVert \Theta_{\mathcal{M}}\rVert_F$.
Combining these two cases yields the desired result.
\end{proof}

Next, we obtain an upper bound for $\max_{1\leq k\leq K}\lVert \hat{\Psi}_n^{(k)}-\Psi_0^{(k)}\rVert_\infty$, which holds with high-probability assuming the tail conditions of the random vectors.
\begin{lemma}
\label{lemma:const4}
Suppose that $n_k/n\geq \gamma>0$ for all $k$ and $n$.

\noindent (i) Suppose that Condition \ref{cond:exptail} holds.
Then for $n \geq 6 \gamma^{-1}\log p$ we have
\begin{eqnarray}
\label{eqn:lambdaprob}
P\left(  \lVert \hat{\Sigma}_n-\Sigma_0\rVert_\infty\geq 2^3\sqrt{6}(1+4c_1^2)^2\gamma^{-1/2}\max_{k,i}\sigma^{(k)}_{ii}\sqrt{\frac{ \log p}{\gamma n}}\right)\leq 2K/p.
\end{eqnarray}

\noindent (ii) Suppose that Condition \ref{cond:polytail} holds with $c_2,c_3>0$ and $p\leq c_7n^{c_2}$.
Then we have for $\tau >\max_k(2^7+2^3\sqrt{1+2^{4}3^2c_4\max_{k,i}\{\sigma_{ii}^{(k)}}\}^2)/(9c_4\max_{k,i}\{\sigma_{ii}^{(k)}\}^2)$
\begin{eqnarray}\label{eq:probii}
&&P\left(  \lVert \hat{\Sigma}_n-\Sigma_0\rVert_\infty\geq \sum_{k=1}^K\delta_n^{(k)}\right)\leq  K \nu_n
\end{eqnarray}
where
\begin{eqnarray*}
&&\delta_n^{(k)} \equiv (1+2\max_{i}|\mu^{(k),i}|)(2\delta_{n,1}^{(k)}+\delta_{n,2}^{(k)})+(\delta_{n,1}^{(k)})^2+(\delta_{n,2}^{(k)})^2+2\delta_{n,3}^{(k)},
\end{eqnarray*}
with
\begin{eqnarray*}
&&\delta_{n,1}^{(k)} \equiv \max_{i,j}\E|X^{(k),i}_lX^{(k),j}_l|I(|X^{(k),i}_lX^{(k),j}_l|\geq n_k^{1/2}(\log p)^{-1/2}),\\
&&\delta_{n,2}^{(k)} \equiv \{c_4\max_{k,i}\{\sigma_{ii}^{(k)}\}^2(4+\tau)\log p/n_k\}^{1/2},\\
&&\delta_{n,3}^{(k)}\equiv \max_i \E|X^{(k),i}_l|I(|X^{(k),i}_l|\geq n_k^{1/2}(\log p)^{-1/2}).
\end{eqnarray*}

\noindent (iii) Suppose that Condition \ref{cond:eigen} holds and that $P(\lVert \hat{\Sigma}_n -\Sigma_0\rVert_\infty \geq b_n) = o(1)$ and $b_n=o(1)$ as $n\rightarrow \infty$. Then $P(\lVert \hat{\Psi}_n -\Psi_0\rVert_\infty \geq C_1 b_n) = o(1)$.
\end{lemma}

\begin{proof}
(i) This was proved by \citet{MR2836766}.

(ii) Note that
\begin{eqnarray*}
\hat{\Sigma}^{(k)}_n - \Sigma^{(k)} &=& {n_k}^{-1}\sum_{l=1}^{n_k}(X^{(k)}_l)^{\otimes 2} -\E(X^{(k)})^{\otimes 2} -(\overline{X}^{(k)}-\mu^{(k)})^{\otimes 2} \\
&&- \mu^{(k)}(\overline{X}^{(k)}-\mu^{(k)})^T - (\overline{X}^{(k)}-\mu^{(k)})(\mu^{(k)})^T.
\end{eqnarray*}
We first evaluate the probability in \eqref{eq:probii} for $n_k^{-1}\sum_{l=1}^{n_k}(X^{(k)}_l)^{\otimes 2} -\E(X^{(k)})^{\otimes 2}$.
Let
\begin{eqnarray*}
&&Y^{(k),ij}_l \equiv X^{(k),i}_lX^{(k),j}_l - \E X^{(k),i}_lX^{(k),j}_l,\\
&&\bar{Y}^{(k),ij}_l\equiv X^{(k),i}_lX^{(k),j}_lI\left(|X^{(k),i}_lX^{(k),j}_l|\leq \sqrt{\frac{n_k}{\log p}}\right) - \E X^{(k),i}_lX^{(k),j}_lI\left(|X^{(k),i}_lX^{(k),j}_l|\leq \sqrt{\frac{n_k}{\log p}}\right),\\
&&\tilde{Y}^{(k),ij}_l \equiv Y^{(k),ij}_l - \bar{Y}^{(k),ij}_l.
\end{eqnarray*}
We have
\begin{eqnarray}
&&P\left(\max_{i,j}\left|\sum_{l=1}^{n_k}\tilde{Y}^{(k),ij}_l \right|\geq 2n_k\delta_{n,1}^{(k)}\right)\nonumber\\
&&\leq P\left(\max_{i,j}\left|\sum_{l=1}^{n_k}X^{(k),i}_lX^{(k),j}_lI\left(|X^{(k),i}_lX^{(k),j}_l|\geq \sqrt{\frac{n_k}{\log p}}\right) \right|\geq n_k\delta_{n,1}^{(k)}\right)\quad \mbox{(triangle inequality)}\nonumber\\
&&\leq P\left(\max_{l,i}(X^{(k),i}_l)^2\geq n_k^{1/2}(\log p)^{-1/2}\right)\quad (xy\leq \max\{x^2,y^2\})\nonumber\\
&&\leq  p n_k \frac{\E X_{0i}^{4(c_2+c_3+1)}(\log p)^{c_2+c_3+1}}{n_k^{c_2+c_3+1}} \quad \mbox{(Markov's inequality)}\nonumber\\
&&\leq   \frac{c_7c_4\max_{k,i}\{\sigma_{ii}^{(k)}\}^2(\log p)^{c_2+c_3+1}}{n_k^{c_3}} \quad (p\leq c_7n^{c_2})\nonumber\\
&&\leq   \frac{c_7c_4\max_{k,i}\{\sigma_{ii}^{(k)}\}^2(\log p)^{c_2+c_3+1}}{\gamma^{c_3}n^{c_3}}\equiv \nu_{n,1}.\label{eqn:polytail1}
\end{eqnarray}
Note that
\begin{eqnarray*}
\E(\bar{Y}^{(k),ij}_l)^2
&\leq &\E\left[X^{(k),i}_lX^{(k),j}_lI\left(|X^{(k),i}_lX^{(k),j}_l|\leq \sqrt{\frac{n_k}{\log p}}\right)\right]^2
\leq \E|X^{(k),i}_lX^{(k),j}_l|^2\\
&\leq&  2^{-1}(\E(X^{(k),i}_l)^4 + \E(X^{(k),j}_l)^4)
\leq c_4\max_{k,i}\{\sigma_{ii}^{(k)}\}^2.
\end{eqnarray*}
It follows from Bernstein's inequality that
\begin{align}
& P\left(\max_{i,j}\left|\sum_{l=1}^{n_k}\bar{Y}^{(k),ij}_l\right| \geq n_k\delta_{n,2}^{(k)} \right)\nonumber\\
& \leq 2p^2 \exp\left(-\frac{c_4\max_{k,i}\{\sigma_{ii}^{(k)}\}^2(4+\tau)\log p}{2c_4\max_{k,i}\{\sigma_{ii}^{(k)}\}^2+ 2\sqrt{c_4\max_{k,i}\{\sigma_{ii}^{(k)}\}^2(64+16\tau)}/3} \right)\equiv \nu_{n,2}\label{eqn:polytail2}. 
\end{align}
Note that $\nu_{n,2} \rightarrow 0$ as $p\rightarrow \infty$ for $\tau > (2^7+2^3\sqrt{1+2^{4}3^2c_4\max_{k,i}\{\sigma_{ii}^{(k)}}\}^2)/(9c_4\max_{k,i}\{\sigma_{ii}^{(k)}\}^2)$.
To see this note that we need to have
\begin{eqnarray*}
\frac{3c_4\max_{k,i}\{\sigma_{ii}^{(k)}\}^2(4+\tau)}{6c_4\max_{k,i}\{\sigma_{ii}^{(k)}\}^2 + 8\sqrt{c_4\max_{k,i}\{\sigma_{ii}^{(k)}\}^2(4+\tau)}}> 2.
\end{eqnarray*}
so that the power in the exponent is strictly negative.
This inequality reduces to
\begin{eqnarray*}
3c_4\max_{k,i}\{\sigma_{ii}^{(k)}\}^2\tau > 16 \sqrt{c_4\max_{k,i}\{\sigma_{ii}^{(k)}\}^2(4+\tau)}.
\end{eqnarray*}
We can solve this by changing a quadratic equation for $\tau$, since $\tau$ of our interest is positive.
Combining (\ref{eqn:polytail1}) and (\ref{eqn:polytail2}) yields
\begin{eqnarray}
P\left(\left\lVert\frac{1}{n_k}\sum_{i=1}^{n_k}(X^{(k)}_l)^{\otimes 2} -\E(X^{(k)})^{\otimes 2}\right\rVert_\infty \geq 2\delta_{n,1}^{(k)}+\delta_{n,2}^{(k)} \right)\leq \nu_{n,1}+\nu_{n,2}.\label{eqn:polytail3}
\end{eqnarray}
Let
\begin{align*}
&Z^{(k),i}_l \equiv X^{(k),i}_l - \E X_{l}^{(k),i},\\
&\bar{Z}^{(k),i}_l \equiv X^{(k),i}_lI(|X^{(k),i}_l|\leq n_k^{1/2}(\log p)^{-1/2}) - \E X^{(k),i}_lI(|X^{(k),i}_l|\leq n_k^{1/2}(\log p)^{-1/2}),\\
&\tilde{Z}^{(k),i}_l \equiv U^{(k),i}_l - \bar{Z}^{(k),i}_l.
\end{align*}
Proceeding as for $Y^{(k),ij}_l$'s, we have
\begin{eqnarray*}
P\left(\max_{i}\left|\sum_{l=1}^{n_k}\tilde{Z}^{(k),i}_l \right|\geq 2n_k\delta_{n,3}^{(k)}\right)
\leq   \frac{c_7c_4\max_{k,i}\{\sigma_{ii}^{(k)}\}^2(\log p)^{2(c_2+c_3+1)}}{\gamma^{c_2+c_3+1}n^{c_2+c_3+1}}\equiv \nu_{n,3},
\end{eqnarray*}
and
\begin{eqnarray*}
P\left(\max_{i}\left|\sum_{k=1}^n\bar{Z}^{(k),i}_l\right| \geq n_k\delta_{n,2}^{(k)} \right)
\leq  \nu_{n,2}.
\end{eqnarray*}
Thus, we have
\begin{eqnarray}
&&P(\lVert (\overline{X}^{(k)}-\mu^{(k)})^{\otimes 2}\rVert_\infty  \geq (\delta_{n,2}^{(k)})^2 + (2\delta_{n,3}^{(k)})^2)
\leq P\left(\max_{i}|\overline{X}^{(k),i}-\mu^{(k),i}| \geq \sqrt{(\delta_{n,1}^{(k)})^2 + (\delta_{n,2}^{(k)})^2}\right)\nonumber \\
&&\leq P\left(\max_{i}\left|\sum_{k=1}^n\bar{Z}^{(k),i}_l\right| \geq n_k\delta_{n,2}^{(k)} \right) + P\left(\max_{i}\left|\sum_{l=1}^{n_k}\tilde{Z}^{(k),i}_l \right|\geq 2n_k\delta_{n,3}^{(k)}\right)\nonumber \\
&&\leq \nu_{n,2}+\nu_{n,3},\label{eqn:polytail4}
\end{eqnarray}
and
\begin{eqnarray}
&&P\left(\lVert (\overline{X}^{(k)}-\mu^{(k)})(\mu^{(k)})^T \rVert_\infty \geq \max_{i}|\mu^{(k),i}|(2\delta_{n,1}^{(k)}+\delta_{n,2}^{(k)})\right)\nonumber \\
&&\leq P\left(\max_{i}|\overline{X}^{(k),i}-\mu^{(k),i}| \geq 2\delta_{n,1}^{(k)}+\delta_{n,2}^{(k)}\right)
\leq \nu_{n,1}+\nu_{n,2}. \label{eqn:polytail5}
\end{eqnarray}
Combining (\ref{eqn:polytail3})-(\ref{eqn:polytail5}) yields
\begin{eqnarray*}
&&P\left(\lVert \hat{\Sigma}^{(k)}_n-\Sigma^{(k)}\rVert_\infty \geq (1+2\max_{i}|\mu^{(k),i}|)(2\delta_{n,1}^{(k)}+\delta_{n,2}^{(k)})+(\delta_{n,2}^{(k)})^2+(2\delta_{n,3}^{(k)})^2 \right)\\
&&\leq 3\nu_{n,1}+ 4\nu_{n,2} + \nu_{n,3} = \nu_n.
\end{eqnarray*}
Note that $\delta_{n,1}^{(k)},\delta_{n,2}^{(k)},\delta_{n,3}^{(k)},\nu_{n,1},\nu_{n,2},\nu_{n,3}\rightarrow 0$ as $n,p\rightarrow \infty$ if $\log p/n\rightarrow 0$.
Note also that $\delta_{n,1}^{(k)}$, $\delta_{n,2}^{(k)}$ and $(\delta_{n,3}^{(k)})^2$ are $O\left(\sqrt{\log p/n}\right)$ on the set where $n_k/n\geq \gamma$.
For example, we have by Jensen's inequality that
\begin{eqnarray*}
\sqrt{\frac{n}{\log p}}(\delta_{n,3}^{(k)})^2 &=& \sqrt{\frac{n}{\log p}}\max_{i}  \{\E|X^{(k),i}|I\{|X^{(k),i}|\geq n_k^{1/2}(\log p)^{-1/2}\}\}^2\\
&\leq &\max_{i}  \E\frac{n}{n_k}\sqrt{\frac{n_k}{\log p}}|X^{(k),i}|^2I\{|X^{(k),i}|\geq n_k^{1/2}(\log p)^{-1/2}\} \\
&\leq & \gamma^{-1}\max_{i}  \E|X^{(k),i}|^3I\{|X^{(k),i}|\geq n_k^{1/2}(\log p)^{-1/2}\}\\
&\leq &c_4\gamma^{-1}\max_{i}\{\sigma^{(k)}_{ii}\}^2.
\end{eqnarray*}

(iii) Given that $|\sigma_{0,ij}^{(k)}|\leq \sqrt{\sigma_{0,ii}^{(k)}\sigma_{0,jj}^{(k)}}$,
\begin{eqnarray*}
&&|\psi_{n,ij}^{(k)} -\psi_{0,ij}^{(k)}|
= \left|\frac{\hat{\sigma}_{n,ij}^{(k)}}{\sqrt{\hat{\sigma}^{(k)}_{n,ii}\hat{\sigma}^{(k)}_{n,jj}}} - \frac{\sigma_{0,ij}^{(k)}}{\sqrt{\sigma^{(k)}_{0,ii}\sigma^{(k)}_{0,jj}}} \right| \\
&&=\frac{1}{\sqrt{\hat{\sigma}^{(k)}_{n,ii}\hat{\sigma}^{(k)}_{n,jj}\sigma^{(k)}_{0,ii}\sigma^{(k)}_{0,jj}}}
\left|\sqrt{\sigma^{(k)}_{0,ii}\sigma^{(k)}_{0,jj}}(\hat{\sigma}_{n,ij}^{(k)} - \sigma_{0,ij}^{(k)}) +\sigma_{0,ij}^{(k)}\left( \sqrt{\sigma^{(k)}_{0,ii}\sigma^{(k)}_{0,jj}} - \sqrt{\hat{\sigma}^{(k)}_{n,ii}\hat{\sigma}^{(k)}_{n,jj}}\right)\right|\\
&&\leq \frac{\sqrt{\sigma^{(k)}_{0,ii}\sigma^{(k)}_{0,jj}}}{\sqrt{\hat{\sigma}^{(k)}_{n,ii}\hat{\sigma}^{(k)}_{n,jj}\sigma^{(k)}_{0,ii}\sigma^{(k)}_{0,jj}}}
\left\{ \left|\hat{\sigma}_{n,ij}^{(k)} - \sigma_{0,ij}^{(k)}\right| + \left| \sqrt{\sigma^{(k)}_{0,ii}\sigma^{(k)}_{0,jj}} - \sqrt{\hat{\sigma}^{(k)}_{n,ii}\hat{\sigma}^{(k)}_{n,jj}}\right|\right\},
\end{eqnarray*}
wherein
\begin{eqnarray*}
&&\sqrt{\sigma^{(k)}_{0,ii}\sigma^{(k)}_{0,jj}} - \sqrt{\hat{\sigma}^{(k)}_{n,ii}\hat{\sigma}^{(k)}_{n,jj}}\\
&& = \frac{\sqrt{\sigma^{(k)}_{0,jj}}}{\sqrt{\sigma^{(k)}_{0,ii}} + \sqrt{\hat{\sigma}^{(k)}_{n,ii}}}(\sigma^{(k)}_{0,ii} - \hat{\sigma}^{(k)}_{n,ii}) 
+\frac{\sqrt{\hat{\sigma}^{(k)}_{n,ii}}}{\sqrt{\sigma^{(k)}_{0,jj}}+\sqrt{\hat{\sigma}^{(k)}_{n,jj}}}(\sigma^{(k)}_{0,jj} - \hat{\sigma}^{(k)}_{n,jj}).
\end{eqnarray*}
Since $b_n \rightarrow 0$, $b_n \leq c_5/2$ for $n$ sufficiently large by Condition \ref{cond:eigen}.
On the event $\lVert \hat{\Sigma}_n-\Sigma_0\rVert_\infty \leq b_n$ with $n$ large, $0< c_5/2\leq \sigma_{0,ii}^{(k)} -c_5/2 \leq \hat{\sigma}_{n,ii}^{(k)} \leq \sigma_{0,ii}^{(k)} +c_5/2\leq c_6+c_5/2$.
Thus, 
\begin{eqnarray*}
\frac{\sqrt{\sigma^{(k)}_{0,ii}\sigma^{(k)}_{0,jj}}}{\sqrt{\hat{\sigma}^{(k)}_{n,ii}\hat{\sigma}^{(k)}_{n,jj}\sigma^{(k)}_{0,ii}\sigma^{(k)}_{0,jj}}}
&\leq& \frac{2(c_5+2c_6)}{c_5^2} \\
\frac{\sqrt{\sigma^{(k)}_{0,jj}}}{\sqrt{\sigma^{(k)}_{0,ii}} + \sqrt{\hat{\sigma}^{(k)}_{n,ii}}}
&\leq& \frac{\sqrt{c_6}}{2\sqrt{c_5}} \\ 
\frac{\sqrt{\hat{\sigma}^{(k)}_{n,ii}}}{\sqrt{\sigma^{(k)}_{0,jj}}+\sqrt{\hat{\sigma}^{(k)}_{n,jj}}}
&\leq&  \frac{\sqrt{c_5+2c_6}}{2\sqrt{c_5}}.
\end{eqnarray*}
It follows that
\begin{eqnarray*}
|\psi_{n,ij}^{(k)} -\psi_{0,ij}^{(k)}|
\leq \left\{2c_5^{-2}+c_5 + c_6^{-3/2} +2c_5^{-5/2}c_6 + (c_5^{-4}+2c_5^{-5}c_6)^{1/2}\right\} \max_{k,i,j}|\hat{\sigma}_{n,ij}^{(k)}-\sigma_{0,ij}^{(k)}|.
\end{eqnarray*}
Thus we have
\begin{eqnarray*}
&&P\left(\lVert \hat{\Psi}_n - \Psi_0\rVert_\infty \geq C_1 b_n \right)\\
&&\leq  P\left(\lVert \hat{\Psi}_n - \Psi_0\rVert_\infty \geq C_1 b_n , \lVert \hat{\Sigma}_n - \Sigma_0\rVert_\infty < b_n\right)
+ P\left(\lVert \hat{\Sigma}_n - \Sigma_0\rVert_\infty \geq b_n \right)\\
&&\leq 2P\left(\lVert \hat{\Sigma}_n - \Sigma_0\rVert_\infty \geq b_n \right)\rightarrow 0.
\end{eqnarray*}
\end{proof}

So far we have assumed $n_k/n\geq \gamma$ in lemmas.
We evaluate the probability of this event noting that $n_k \sim \mbox{Binom}(n,\pi_k)$.
\begin{lemma}
\label{lemma:const6}
Let $\epsilon>0$ such that $\gamma \equiv \min_k\pi_k-\epsilon>0$.
Then
\begin{eqnarray}
P\left(\min_kn_k/n\leq  \min_k\pi_k-\epsilon\right)
\leq  2K \exp(-2n\epsilon^2).
\end{eqnarray}
\end{lemma}

\begin{proof}
We have by Hoeffding's inequality that
\begin{eqnarray*}
&&P\left(\min_k n_k/n\leq  \min_k\pi_k-\epsilon\right)
\leq P\left(\exists k, n_k/n  \leq \min_k\pi_k-\epsilon\right)\\
&&\leq P\left( \exists k, n_k/n  \leq \pi_k-\epsilon\right)
\leq P\left(\exists k,\left|n_k/n - \pi_k\right|\geq \epsilon\right)\\
&&\leq \sum_{k=1}^K P\left(\left|n_k/n - \pi_k\right|\geq \epsilon\right)
\leq  2K \exp(-2n\epsilon^2).
\end{eqnarray*}
\end{proof}

\begin{proof}[Proof of Lemma~\ref{lemma:l1const}]
We apply Lemma \ref{lemma:NegLemma4} to obtain the non-asymptotic error bounds.

We first compute a lower bound for $f(\Delta)$.
Suppose $\epsilon \leq c$.
For $\Delta\in \mathcal{C}\cap \{\Delta\in\tilde{\mathbb{R}}^{(pK)\times (pK)}:\lVert \Delta \rVert_F = \epsilon\}$, we have by Lemma~\ref{lemma:const1}(ii) and (iii) that
\begin{eqnarray*}
f(\Delta)
&\geq & -\langle \tilde{\ell}_n(\Theta_0),\Delta\rangle + \kappa_{\ell_n,c}\lVert \Delta\rVert_F^2 + \rho_n \{r(\Delta_{\mathcal{M}^\perp}) -r(\Delta_{\mathcal{M}})\}.
\end{eqnarray*}
The assumption on $\rho_n$ and Lemma~\ref{lemma:const1}(iii) and (iv) then yield
\begin{equation*}
|\langle \tilde{\ell}_n(\Theta_0),\Delta\rangle| \leq \frac{\rho_n}{2}\{r(\Delta_{\mathcal{M}}) +r(\Delta_{\mathcal{M}^\perp}) \}.
\end{equation*}
From this inequality and Lemma~\ref{lemma:const1}(v) we have
\begin{eqnarray*}
f(\Delta)
\geq \kappa_{\ell_n,c}\lVert \Delta\rVert_F^2 -\frac{3\rho_n}{2} r(\Delta_{\mathcal{M}})
\geq \kappa_{\ell_n,c}\lVert \Delta\rVert_F^2 -\frac{3\rho_n}{2} (s+1)^{1/2}\left(1+\rho_2\lVert L\rVert_2^{1/2}\right)\lVert \Delta\rVert_F.
\end{eqnarray*}

Viewing the right hand side of the above inequality as a quadratic equation in $\lVert \Delta\rVert_F$, we have $f(\Delta)>0$ if
\begin{equation*}
\lVert \Delta\rVert_F \geq \frac{3\rho_n}{\kappa_{\ell_n,c}}(s+1)^{1/2}\left(1+\rho_2\lVert L\rVert_2^{1/2}\right)\equiv \epsilon_c >0.
\end{equation*}
Thus, if we show that there exists a $c_0>0$ such that $\epsilon_{c_0}\leq c_0$, Lemma \ref{lemma:NegLemma4} yields that
$\lVert \hat{\Theta}_{\rho_n}-\Theta_0\rVert_F \leq \epsilon_{c_0}.$

Consider the inequality $(x+y)^2z^{1/2} \leq y$ where $x,y,z\geq 0$.
This inequality holds for $(x,y,z)$ such that $x=y$ and $xz^{1/2}=1/4$.
We apply the inequality above with $x=\lambda_{\Theta},y=c,z=2^43^2\rho_n^2\gamma^{-2}s(1+\rho_2\lVert L\rVert_2^{1/2})^2$ and solve
$xz\leq 1/4$ for $n$.
(i) For $\rho_n = 2^3\sqrt{6}C_1(1+4c_1^2)^2\gamma^{-1/2}\max_{k,i}\sigma^{(k)}_{ii}\sqrt{\log p/n}$, $xz\leq 1/4$ yields
\begin{eqnarray*}
n\geq\max\left\{6\gamma^{-1}\log p,2^{15}3^3C_1^2 \gamma^{-3}(1+4c_1^2)^2 \max_{k,i}\{\sigma_{ii}^{(k)}\}^2
\lambda_{\Theta}^4
\left(1+\rho_2\lVert L\rVert_2^{1/2}\right)^2s\log p\right\},
\end{eqnarray*}
and $(x+y)^4z$ becomes
\begin{eqnarray*}
\epsilon^2_{\max_k \{\lVert \Theta_0^{(k)}\rVert_2\}}
\leq 2^{15}3^3 (1+c_1^2)^2\max_{k,i}(\sigma^{(k)}_{ii})^2\left(1+\rho_2\lVert L\rVert_2^{1/2}\right)^2\gamma^{-3}
\lambda_{\Theta}^4 
\frac{s\log p}{n}.
\end{eqnarray*}
(ii) For $\rho_n = C_1K\delta_n$, there is no closed form solution for $n$. Note that $\delta_n\rightarrow 0$ if $\log p/n\rightarrow 0$ so that $xz\leq 1/4$ holds for $n$ sufficiently large, given that $\sum_{k=1}^K\delta_{n}^{(k)} \leq K\delta_n$.

Computing appropriate probabilities using Lemmas \ref{lemma:const4} and  \ref{lemma:const6} completes the proof.
\end{proof}

\begin{proof}[Proof of Theorem~\ref{thm:rate}]
The estimation error $\lVert \hat{\Omega}^{(k)}_{\rho_n}-\Omega_0\rVert^{(2)}_2$ in the spectral norm can be bounded and evaluated in the same way as in the proof of Theorem 2 of \citet{MR2417391} together with Lemma~\ref{lemma:l1const}. 
\end{proof}

\subsection{Model Selection Consistency}
Our proof is based on the primal-dual witness approach of \citet{MR2836766}, with some modifications to overcome a difficulty in their proof when applying the fixed point theorem to a discontinuous function.
First, we define the oracle estimator $\check{\Theta}_{\rho_n}= (\check{\Theta}_{\rho_n}^{(1)},\ldots,\check{\Theta}_{\rho_n}^{(K)})$ by
\begin{eqnarray}
\check{\Theta}_{\rho_n}
&&=\argmin_{\Theta^{(k)}>0,\Theta^{(k)}=(\Theta^{(k)})^T,\Theta^{(k)}_{\left(S^{(k)}\right)^c} =0} n^{-1}\sum_{k=1}^K  n_k
\left( \mbox{tr}\left(\Psi^{(k)}_n\Theta^{(k)}\right) -\log \mbox{det}(\Theta^{(k)})\right)\nonumber \\
&&\qquad  + \rho_{n}\sum_{k=1}^K\lVert \Theta^{(k)} \rVert_1
+  \rho_{n}\rho_2\sum_{i,j} \sqrt{\Theta_{ij}^TL\Theta_{ij} }, 
\label{eqn:oracle}
\end{eqnarray}
where $\Theta_{\left(S^{(k)}\right)^c}^{(k)} = 0$ indicates that 
$\Theta_{(i,j)}^{(k)} = 0$ for $(i,j)\notin S^{(k)}$.

\begin{lemma}
\label{lemma:sbdiff}

(i) Let $A\in\mathbb{R}^{p\times p}$ be a positive semidefinite matrix with eigenvalues $0\leq \lambda_1\leq \lambda_2\leq \cdots\leq \lambda_p$ and corresponding eigenvectors $u_i$ satisfying $u_i\perp u_j,i\neq j$ and $\lVert u_i\rVert =1$. The subdifferential $\partial \sqrt{x^TAx}$ of $f(x) = \sqrt{x^TAx}$  is
\begin{eqnarray*}
\partial \sqrt{x^TAx} =\left\{
\begin{array}{ll}
Ax/\sqrt{x^TAx}, & Ax\neq 0,\\
\{ U\Lambda^{1/2} y: \lVert y\rVert \leq 1\}, & Ax=0.\\
\end{array}\right.
\end{eqnarray*}
where $U\in\mathbb{R}^{p\times  p}$ has $u_i$ as the $i$th columns and $\Lambda^{1/2}$ is the diagonal matrix with $\lambda_i^{1/2},i=1,\ldots,p,$ as diagonal elements. 
Furthermore, the subgradients are bounded above, i.e.
\begin{equation*}
\lVert \nabla f(x)\rVert_\infty \leq \lVert A \rVert_2^{1/2}, \quad \mbox{ for all } \nabla f(x) \in \partial \sqrt{x^TAx}.
\end{equation*}

(ii) Let $A\in\mathbb{R}^{p\times p}$ be a positive semidefinite matrix and $S=\{S_i\}\subset \{1,\ldots,p\}$.
Suppose $A_{SS}$ has eigenvalues $0\leq \lambda_{1,S}\leq \lambda_{2,S}\leq \cdots\leq \lambda_{|S|,S}$ and corresponding eigenvectors $u_{i,S}$ satisfying $u_{i,S}\perp u_{j,S},i\neq j$ and $\lVert u_{i,S}\rVert =1$. 
Let $g_S:\mathbb{R}^{|S|} \rightarrow \mathbb{R}^p$ be a map defined by $g_S(x) = y$ where $y_i =x_{S_j}$ for $i=S_j$ for  and $y_i  =0$ for $i\notin S$.
The subdifferential $h_{A,S}(x) = \sqrt{g_S(x)^TAg_S(x)}$ equals to the subdifferential of $\sqrt{x^TA_{SS}x}$ given by
\begin{eqnarray*}
\partial \sqrt{x^TA_{SS}x} =\left\{
\begin{array}{ll}
A_{SS}x/\sqrt{x^TA_{SS}x}, & A_{SS}x\neq 0,\\
U_S\Lambda_S^{1/2}\{y:\lVert y\rVert \leq 1\}, & A_{SS}x=0.\\
\end{array}\right.
\end{eqnarray*}
where $U_{S}\in\mathbb{R}^{|S|\times  |S|}$ has $u_{i,S}$ as the $i$th columns and $\Lambda_{S}^{1/2}$ is the diagonal matrix with $\lambda_{i,S}^{1/2},i=1,\ldots,|S|,$ as diagonal elements. 
For $x$ with $A_{SS}x\neq 0$, there is a relationship between $\partial \sqrt{x^TA_{SS}x}$ and $\partial \sqrt{y^TAy}$ at $y=g_S(x)$ given by 
\begin{eqnarray*}
&&\left\{\frac{Ay}{\sqrt{y^TAy}}\right\}_{S} = \frac{A_{SS}x}{\sqrt{x^TA_{SS}x}},\\
&&\left\{\frac{Ay}{\sqrt{y^TAy}}\right\}_{S^c} = \frac{A_{S^cS}x}{\sqrt{x^TA_{SS}x}}.
\end{eqnarray*}
Subgradients are bounded above:
\begin{equation*}
\lVert \nabla h_{A,S}(x)\rVert_\infty \leq \lVert A_{SS} \rVert_2^{1/2} \leq \lVert A \rVert_2^{1/2}, \quad \forall \nabla f_{A,S}(x) \in \partial \sqrt{x^TA_{SS}x}.
\end{equation*}

\end{lemma}
\begin{proof}

(i) For $x$ with $Ax\neq 0$, $f(x)$ is differentiable and the subgradient of $f$ at $x$ is simply the matrix derivative.
By definition, for $x$ with $Ax=0$, the subgradient $v$ of $f$ at $x$ satisfies the following inequality
\begin{equation}
\label{eqn:defsubgrd}
 \sqrt{y^TAy} \geq \langle y-x,v\rangle, 
\end{equation}
for all $y$.
Choosing $y=2x$ and $y=0$ yield $0\geq \langle x,v\rangle$ and $0\geq -\langle x,v\rangle$, implying $\langle x,v\rangle = 0$.
The inequality (\ref{eqn:defsubgrd}) reduces to $\sqrt{y^TAy} \geq \langle y,v\rangle$, for any $y$.
If $Ay=0$, a similar argument implies that $\langle y,v\rangle =0$. Hence $v\perp y$ for every $y$ with $Ay=0$.

Let $j_0$ be the smallest index such that $\lambda_{j_0}>0$.
Because $u_j$'s form an orthonormal basis, any arbitrary vector $y$ can be written as $y = \sum_{j=1}^p \beta_j u_j$.
Moreover, the null space of $A$ is the span of $u_1,\ldots,u_{j_0-1}$. Thus, the subgradient $v$ can be written as $v=\sum_{j=j_0}^p\alpha_j u_j$.
Thus, using the spectral decomposition of $A$ as $A = \sum_{j=j_0}^p\lambda_j u_j u_j^T$, we can write $f(y) = \{\sum_{j=j_0}^p \lambda_j \beta_j^2\}^{1/2}$.
On the other hand, $\langle y, v\rangle = \sum_{j=j_0}^p\alpha_j\beta_j$.
Thus, the inequality \eqref{eqn:defsubgrd} further reduces to
\begin{equation*}
\left\{\sum_{j=j_0}^p\lambda_j\beta_j^2\right\}^{1/2} \geq \sum_{j=j_0}^p\alpha_j\beta_j, \quad \forall \beta_j\in\mathbb{R}. 
\end{equation*}
It follows from the Cauchy-Schwartz inequality that the left hand side of the inequality is bounded from above;
\begin{eqnarray*}
\sum_{j=j_0}^p \alpha_j\beta_j
=\sum_{j=j_0}^p \frac{\alpha_j}{\lambda_j^{1/2}}\lambda_j^{1/2}\beta_j
\leq \left\{ \sum_{j=j_0}^p \frac{\alpha_j^2}{\lambda_j}\right\}^{1/2}\left\{\sum_{j=j_0}^p\lambda_j\beta^2_j \right\}^{1/2}.
\end{eqnarray*}
Thus, 
\begin{eqnarray*}
\partial f(x) = \left\{v: v=\sum_{j=j_0}^p\alpha_j v_j, \sum_{j=j_0}^p \frac{\alpha_j^2}{\lambda_j}\leq 1,\alpha_j\in\mathbb{R}\right\}.
\end{eqnarray*}
It is easy to see that this set is the image of the map $U\Lambda^{1/2}$ on the closed ball of radius 1.

Given that $\lVert x\rVert_\infty\leq \lVert x\rVert$, to establish the bound in the $\ell_\infty$-norm, we compute the bound in the Euclidean norm.
We use the same notation as in (i).
For $x$ with $Ax\neq 0$, 
\begin{equation*}
\left\lVert \frac{Ax}{\sqrt{x^TAx}}\right\rVert = \frac{\lVert U \Lambda^{1/2}\Lambda^{1/2} U^T x \rVert}{\lVert \Lambda^{1/2}U^Tx\rVert}
\leq \lVert U\Lambda^{1/2}\rVert_2.
\end{equation*}
But $\lVert U\Lambda^{1/2}\rVert_2 = \sup_{\lVert x\rVert =1}\lVert U\Lambda^{1/2}x \rVert= \sup_{x\in\mathbb{R}^K}\lVert U\Lambda^{1/2}(U^Tx)\rVert/\lVert U^Tx\rVert  = \lVert A\rVert_2^{1/2}$, because $\lVert U^Tx\rVert =\lVert x\rVert$.
For $x$ with $Ax=0$, $\lVert \Lambda^{1/2} y/\lVert y\rVert\rVert \leq \lVert A\rVert_2^{1/2}$ for every $y$. Because of the form of the subdifferential and the fact that $\lVert Ux\rVert = \lVert x\rVert$, the result follows.

(ii) Let $B_S$ be a product of elementary matrices for row and column exchange such that $B_Sg_S(x) = (x,0)$. 
Notice that $B_S=B_S^{-1}$ and that $B_S=B_S^{T}$ since $B_S$ only rearranges elements of vectors and exchanges rows by multiplication from the left.
Note also that 
$\lVert B_S\rVert_2 \leq \lVert B_S\rVert_{\infty/\infty} = 1$,
since $\lVert C\rVert_2 \leq \lVert C\rVert_{\infty/\infty}$ for $C=C^T$ and each row of $B_S$ has only one element with value $1$.
Because
\begin{equation*}
\{h_{A,S}(x)\}^2 = g_S(x)^TAg_S(x) = (B_Sg_S(x))^T(B_S AB_S)(B_Sg_S(x))
= x^TA_{SS}x,
\end{equation*}
the subdifferential of $h_{A,S}(x)$ follows from (ii).
For $x$ with $A_{SS}x\neq x$ and $y=g_S(x)$, $Ay =B_SAB_S(x,0)^T = B_S(A_{SS}x,A^T_{S^cS}x)^T\neq 0$ because of invertibility of $B_S$.
The relationship holds since 
\begin{eqnarray*}
\left[\begin{array}{c}
(Ay/\sqrt{y^TAy})_S\\
(Ay/\sqrt{y^TAy})_{S^c}\\
\end{array}
\right]
=B_S \frac{Ay}{\sqrt{y^TAy}}  
=\frac{1}{\sqrt{x^TA_{SS}x}}\left[
\begin{array}{c}
A_{SS}x\\
A_{S^cS}x\\
\end{array}
\right]
\end{eqnarray*}
An $\ell_\infty$-bound follows from (i) and the fact that 
$\lVert A_{SS} \rVert_2 \leq \lVert B_S\rVert^2_2\lVert A\rVert_2 =\lVert A\rVert_2$.  
\end{proof}

\begin{lemma}
\label{lemma:RWRY3}
For any $\rho_n>0$ and sample correlation matrices $\hat{\Psi}_n = (\hat{\Psi}_n^{(1)},\ldots, \hat{\Psi}_n^{(K)})$, the convex problem (\ref{eqn:originallasso}) has a unique solution $\hat{\Theta}_{\rho_n} = (\hat{\Theta}_{\rho_n}^{(1)},\ldots,\hat{\Theta}_{\rho_n}^{(K)})$ with $\hat{\Theta}_{\rho_n}^{(k)}>0,k=1,\ldots,K,$ characterized by
\begin{eqnarray}
\label{eqn:kkt}
n^{-1}n_k(\psi^{(k)}_{n,ij} -[\{\hat{\Theta}_{\rho_n}^{(k)}\}^{-1}]_{ij}) + \rho_n \hat{U}_{1,ij}^{(k)} + \rho_n\rho_2 \hat{U}_{2,ij}^{(k)}=0,
\end{eqnarray}
with $\hat{U}_{1,ij}^{(k)}\in\partial |\hat{\theta}_{\rho_n,ij}^{(k)}|$ and $(\hat{U}_{2,ij}^{(1)},\ldots,\hat{U}_{2,ij}^{(K)})^T\in \partial \sqrt{\hat{\Theta}_{\rho_n,ij}^TL\hat{\Theta}_{\rho_n,ij}}$
for every $i\neq j$ and $k=1,\ldots,K$. 
Moreover,
\begin{eqnarray}
\label{eqn:kkt2}
n^{-1}n_k(\psi^{(k)}_{n,ii} -[\{\hat{\Theta}_{\rho_n}^{(k)}\}^{-1}]_{ii})+ \rho_n \hat{U}_{1,ij}^{(k)} + \rho_n\rho_2 \hat{U}_{2,ij}^{(k)} =0,
\end{eqnarray}
with $\hat{U}_{1,ij}^{(k)} = \hat{U}_{2,ij}^{(k)}=0$ for every $i=1,\ldots,p,$ and $k=1,\ldots,K$.

For each $(i,j) \in S$, let $S_{ij}=\{k: \Theta_{0,ij}^{(k)} \neq 0\}$.
The convex problem (\ref{eqn:oracle}) has a unique solution $\check{\Theta}_{\rho_n} = (\check{\Theta}_{\rho_n}^{(1)},\ldots,\check{\Theta}_{\rho_n}^{(K)})$ with $\check{\Theta}_{\rho_n}^{(k)}>0,k=1,\ldots,K,$ characterized by
\begin{eqnarray}
\label{eqn:kktOracle}
n^{-1}n_k(\psi^{(k)}_{n,ij} -[\{\check{\Theta}_{\rho_n}^{(k)}\}^{-1}]_{ij}) + \rho_n \check{U}_{1,ij}^{(k)} + \rho_n\rho_2 \check{U}_{2,ij}^{(k)}=0,
\end{eqnarray}
with $\check{U}_{1,ij}^{(k)}\in\partial |\check{\theta}_{\rho_n,ij}^{(k)}|$ and $\check{U}_{2,ij}^{(k)}\in \partial \sqrt{\{\check{\Theta}_{\rho_n,ij}\}_{S_{ij}}^TL_{S_{ij}S_{ij}}\{\check{\Theta}_{\rho_n,ij}\}_{_{S_{ij}}}}$
for every $i\neq j$ and $k=1,\ldots,K$. Moreover,
\begin{eqnarray}
\label{eqn:kktOracle2}
n^{-1}n_k(\psi^{(k)}_{n,ii} -[\{\check{\Theta}_{\rho_n}^{(k)}\}^{-1}]_{ii}) + \rho_n \check{U}_{1,ij}^{(k)} + \rho_n\rho_2 \check{U}_{2,ij}^{(k)}=0,
\end{eqnarray}
with $\check{U}_{1,ij}^{(k)}=\check{U}_{2,ij}^{(k)}=0$ for every $i=1,\ldots,p,$ and $k=1,\ldots,K$.
\end{lemma}

\begin{proof}
A proof for the uniqueness of the solution is similar to the proof of Lemma 3 of \citet{MR2836766}.
The rest is the KKT condition using Lemma \ref{lemma:sbdiff}.
\end{proof}

We choose a pair $\tilde{U}=(\tilde{U}_1,\tilde{U}_2)$ of the subgradients of the first and second regularization terms evaluated at $\check{\Theta}_{\rho_n}$.
For each $(i,j)$ with $\Omega_{0,ij} =0$ or with $L\check{\Theta}_{\rho_n,ij} =0$, set 
\begin{equation*}
\tilde{U}_{1,ij}^{(k)}  = \rho_n^{-1}n^{-1}n_k(-\psi^{(k)}_{n,ij} + [\{\check{\Theta}_{\rho_n}^{(k)}\}^{-1}]_{ij}),\quad \tilde{U}^{(k)}_{2,ij} = 0, \quad k=1,\ldots,K.
\end{equation*}
For $(i,j)$ with $\omega_{0,ij}^{(k)}\neq 0$, for all $k=1,\ldots,K$, set
\begin{equation*}
\tilde{U}_{1,ij}^{(k)}= \check{U}_{1,ij}^{(k)}, \quad
\tilde{U}_{2,ij}^{(k)}= \check{U}_{2,ij}^{(k)}, \quad k=1,\ldots,K.
\end{equation*}
For $(i,j)$ with $L\check{\Theta}_{\rho_n,ij}\neq 0$, $\Omega_{0,ij}\neq 0$ but $\omega_{0,ij}^{(k')}=0$ for some $k'$, set
\begin{equation*}
\tilde{U}_{1,ij}^{(k)}= \rho_n^{-1}n^{-1}n_k(-\psi^{(k)}_{n,ij} + [\{\check{\Theta}_{\rho_n}^{(k)}\}^{-1}]_{ij})-\rho_2\frac{l_k\check{\Theta}_{\rho_n,ij}}{\sqrt{\check{\Theta}_{\rho_n,ij}^TL\check{\Theta}_{\rho_n,ij}}}, \quad
\tilde{U}_{2,ij}^{(k)}=  \frac{l_k^T\check{\Theta}_{\rho_n,ij}}{\sqrt{\check{\Theta}_{\rho_n,ij}^TL\check{\Theta}_{\rho_n,ij}}},
\end{equation*}
if $\omega_{0,ij}^{(k)} =0$, and 
\begin{equation*}
\tilde{U}_{1,ij}^{(k)}= \check{U}_{1,ij}^{(k)}  , \quad
\tilde{U}_{2,ij}^{(k)}=  \frac{l_k^T\check{\Theta}_{\rho_n,ij}}{\sqrt{\check{\Theta}_{\rho_n,ij}^TL\check{\Theta}_{\rho_n,ij}}},
\end{equation*}
otherwise. Here, $l_k$ is the $k$th row of $L$.

The main idea of the proof is to show that $(\check{\Theta}_{\rho_n},\tilde{U})$ satisfies the optimality conditions of the original problem with probability tending to 1. 
In particular, we show the following equation, which holds by construction of $\tilde{U}_1$ and $\tilde{U}_2$, is in fact the KKT condition of the original problem  (\ref{eqn:originallasso}):
\begin{eqnarray}
\label{eqn:kktmat}
n^{-1}n_k(\hat{\Psi}_n^{(k)} -\{\check{\Theta}_{\rho_n}^{(k)}\}^{-1})+ \rho_n \tilde{U}_1^{(k)} + \rho_n\rho_2 \tilde{U}_2^{(k)}=0.
\end{eqnarray}
To this end, we show that $\tilde{U}_1$ and $\tilde{U}_2$ are both subgradients of the original problem.
We can then conclude that the oracle estimator in the restricted problem (\ref{eqn:oracle}) is the solution to the original problem (\ref{eqn:originallasso}).
Then it follows from the uniqueness of the solution that $\check{\Theta}_{\rho_n} = \hat{\Theta}_{\rho_n}$.

Let $\Xi^{(k)} = \hat{\Psi}_n^{(k)} - \Psi_0^{(k)}$, $R^{(k)}(\Delta^{(k)}) = \{\check{\Theta}_{\rho_n}^{(k)}\}^{-1}-\Psi_0^{(k)} + \Psi_0^{(k)}\Delta^{(k)} \Psi_0^{(k)}$, and $\check{\Delta}^{(k)} = \check{\Theta}_{\rho_n}^{(k)} - \Theta_0^{(k)}$.
\begin{lemma}
\label{lemma:strictDual}
Suppose that
$\max \{\lVert \Xi^{(k)}\rVert_\infty,\lVert R^{(k)}(\check{\Delta}^{(k)})\rVert_\infty\}
\leq \alpha\rho_n /8$,
and $\rho_2 \leq \alpha^2/\{4\lVert L\rVert_2^{1/2}(2-\alpha)\}.$
Suppose moreover that $L\check{\Theta}_{\rho_n,ij}\neq 0$ for $(i,j)\in S$. 
Then $|\tilde{U}_{1,ij}^{(k)}|<1$ for $(i,j)\in (S^{(k)})^c$.
\end{lemma}
\begin{proof}
We rewrite (\ref{eqn:kktmat}) to obtain
\begin{equation*}
\frac{n_k}{n}\Psi_0^{(k)}\check{\Delta}^{(k)}\Psi_0^{(k)} + \frac{n_k}{n}\Xi^{(k)} -\frac{n_k}{n}R^{(k)}(\check{\Delta}^{(k)}) +
\rho_n \tilde{U}_1^{(k)} + \rho_n\rho_2\tilde{U}_2^{(k)} = 0.
\end{equation*}
We further rewrite the above equation via vectorization;
\begin{equation*}
\frac{n_k}{n}(\Psi_0^{(k)}\otimes \Psi_0^{(k)})\vec{\check{\Delta}}^{(k)} + \frac{n_k}{n}\vec{\Xi}^{(k)} -\frac{n_k}{n}\vec{R}^{(k)}(\check{\Delta}^{(k)}) +
\rho_n \vec{\tilde{U}}_1^{(k)} + \rho_n\rho_2\vec{\tilde{U}}_2^{(k)} = 0.
\end{equation*}
We separate this equation into two equations depending on $S^{(k)}$;
\begin{eqnarray} \label{eqn:vecteqn2}
&&\frac{n_k}{n}\Gamma^{(k)}_{S^{(k)}S^{(k)}}\vec{\check{\Delta}}^{(k)}_{S^{(k)}} + \frac{n_k}{n}\vec{\Xi}^{(k)}_{S^{(k)}} -\frac{n_k}{n}\vec{R}^{(k)}_{S^{(k)}}(\check{\Delta}^{(k)}) +
\rho_n \vec{\tilde{U}}_{1,S^{(k)}}^{(k)} + \rho_n\rho_2\vec{\tilde{U}}_{2,S^{(k)}}^{(k)} = 0, \\
&&\frac{n_k}{n}\Gamma^{(k)}_{(S^{(k)})^c S^{(k)}}\vec{\check{\Delta}}^{(k)}_{S^{(k)}} + \frac{n_k}{n}\vec{\Xi}^{(k)}_{(S^{(k)})^c} -\frac{n_k}{n}\vec{R}^{(k)}_{(S^{(k)})^c}(\check{\Delta}^{(k)}) +
\rho_n \vec{\tilde{U}}_{1,(S^{(k)})^c}^{(k)} + \rho_n \rho_2 \vec{\tilde{U}}_{2,(S^{(k)})^c}^{(k)}= 0. \nonumber 
\end{eqnarray}
where $\left(\vec{\tilde{U}}_l\right)_J \equiv \vec{\tilde{U}}_{k,J},l=1,2$.
Here we used $\check{\Delta}^{(k)}_{(S^{(k)})^c} = 0$.
Since $\Gamma^{(k)}_{S^{(k)}S^{(k)}}$ is invertible, we solve the first equation to obtain
\begin{equation*}
\frac{n_k}{n}\vec{\check{\Delta}}^{(k)}_{S^{(k)}} = (\Gamma^{(k)}_{S^{(k)}S^{(k)}})^{-1} \left\{-\frac{n_k}{n}\vec{\Xi}^{(k)}_{S^{(k)}} +\frac{n_k}{n}\vec{R}^{(k)}_{S^{(k)}}(\check{\Delta}^{(k)}) -\rho_n \vec{\tilde{U}}_{1,S^{(k)}}^{(k)} - \rho_n\rho_2\vec{\tilde{U}}_{2,S^{(k)}}^{(k)}\right\}.
\end{equation*}
Substituting this expression into (\ref{eqn:vecteqn2}) yields
\begin{eqnarray*}
\vec{\tilde{U}}^{(k)}_{1,(S^{(k)})^c}
&=& \rho_n^{-1}\Gamma^{(k)}_{(S^{(k)})^c S^{(k)}}(\Gamma^{(k)}_{S^{(k)}S^{(k)}})^{-1}\left(\frac{n_k}{n}\vec{\Xi}_{S^{(k)}}^{(k)} - \frac{n_k}{n}\vec{R}^{(k)}_{S^{(k)}}(\check{\Delta}^{(k)})\right) \\
&&+ \Gamma^{(k)}_{(S^{(k)})^c S^{(k)}}(\Gamma^{(k)}_{S^{(k)}S^{(k)}})^{-1}\vec{\tilde{U}}_{1,S^{(k)}}^{(k)}
 + \rho_2\Gamma^{(k)}_{(S^{(k)})^c S^{(k)}}(\Gamma^{(k)}_{S^{(k)}S^{(k)}})^{-1}\vec{\tilde{U}}_{2,S^{(k)}}^{(k)} \\
&&- \rho_n^{-1}\left(\frac{n_k}{n}\vec{\Xi}^{(k)}_{(S^{(k)})^c} -\frac{n_k}{n}\vec{R}^{(k)}_{(S^{(k)})^c}(\check{\Delta}^{(k)})\right) -\rho_2\vec{\tilde{U}}_{2,(S^{(k)})^c}^{(k)}.
\end{eqnarray*}
Taking the $\ell_\infty$-norm yields
\begin{eqnarray*}
\left\lVert \vec{\tilde{U}}_{1,(S^{(k)})^c}^{(k)}\right\rVert_\infty
&\leq& \rho_n^{-1}\lVert\Gamma^{(k)}_{(S^{(k)})^c S^{(k)}}(\Gamma^{(k)}_{S^{(k)}S^{(k)}})^{-1}\rVert_{\infty/\infty}(\lVert\vec{\Xi}_{S^{(k)}}^{(k)}\rVert_\infty + \lVert\vec{R}^{(k)}_{S^{(k)}}(\check{\Delta}^{(k)})\rVert_\infty) \\
&& + \lVert\Gamma^{(k)}_{(S^{(k)})^c S^{(k)}}(\Gamma^{(k)}_{S^{(k)}S^{(k)}})^{-1}\rVert_{\infty/\infty}(\lVert \vec{\tilde{U}}_{1,S^{(k)}}^{(k)}\rVert_\infty + \rho_2\lVert \vec{\tilde{U}}_{2,S^{(k)}}^{(k)}\rVert_\infty)\\
&&  + \rho_n^{-1}(\lVert\vec{\Xi}^{(k)}_{(S^{(k)})^c}\rVert_\infty +\lVert\vec{R}^{(k)}_{(S^{(k)})^c}(\check{\Delta}^{(k)})\rVert_\infty) + \rho_2 \lVert\vec{\tilde{U}}_{2,(S^{(k)})^c}^{(k)}\rVert_\infty\\
&\leq&  \frac{2-\alpha}{\rho_n}(\lVert\vec{\Xi}^{(k)}_{(S^{(k)})^c}\rVert_\infty +\lVert\vec{R}^{(k)}_{(S^{(k)})^c}(\check{\Delta}^{(k)})\rVert_\infty) + 1-\alpha\\
&&+ (2-\alpha)\rho_2 \lVert L \rVert_2^{1/2}.
\end{eqnarray*}
Here we used the property that  $\lVert Ax\rVert_\infty \leq \lVert A\rVert_{\infty/\infty}\lVert x\rVert_\infty$, $\lVert \Gamma^{(k)}_{(S^{(k)})^c S^{(k)}}(\Gamma^{(k)}_{S^{(k)}S^{(k)}})^{-1}\rVert_{\infty/\infty}\leq 1-\alpha$, and applied Lemma \ref{lemma:sbdiff} to bound $\lVert \vec{\tilde{U}}_{2,(S^{(k)})^c}\rVert_\infty$ and $\lVert \vec{\tilde{U}}_{2,S^{(k)}}\rVert_\infty$ by $\lVert L\rVert_2^{1/2}$.
We also used $\lVert \vec{\tilde{U}}_{1,S^{(k)}}^{(k)}\rVert_\infty =\lVert \vec{\check{U}}_{1,S^{(k)}}^{(k)}\rVert_\infty\leq 1$ by construction of $\tilde{U}_1$ and the assumption that $\check{\Theta}_{\rho_n}^{(k)}\neq 0$ for $(i,j)\in S^{(k)}$.
It follows by the assumption of the lemma that
\begin{eqnarray*}
\lVert \tilde{U}_{(S^{(k)})^c}^{(k)}\rVert_\infty
&\leq& \frac{2-\alpha}{\rho_n} \frac{\alpha\rho_n}{4} + (1-\alpha)  + (2-\alpha)\rho_2 \lVert L\rVert_2^{1/2}\\
&\leq  &1-\frac{\alpha}{2}- \frac{\alpha^2}{4} + \frac{\alpha^2}{4}
< 1.
\end{eqnarray*}
\end{proof}
\begin{lemma}[Lemma~5 of \citet{MR2836766}]
\label{lemma:rvkmr5}
Suppose that $\lVert \Delta\rVert_\infty \leq 1/(3\kappa_{\Psi}d)$ with $$(\Delta^{(k)})_{(S^{(k)}\cup \{(i,i):i=1,\ldots,p]\})^c} = 0.$$
Then $\lVert H^{(k)}\rVert_{\infty/\infty} \leq 3/2$ where $H^{(k)}\equiv \sum_{j=1}^\infty (-1)^j (\Psi_0^{(k)}\Delta^{(k)})^j, k=1,\ldots,K,$ and $R^{(k)}(\Delta^{(k)})$ has representation
$R^{(k)}(\Delta^{(k)}) = \Psi_0^{(k)} \Delta^{(k)}\Psi_0\Delta H^{(k)} \Psi_0^{(k)}$
with
$\lVert R^{(k)}(\Delta^{(k)})\rVert_\infty \leq (3/2) d \lVert \Delta^{(k)}\rVert_\infty^2 (\kappa_{\Psi})^3$.
\end{lemma}

\begin{lemma}
\label{lemma:rvkmr5-2}
Suppose $\lVert \Delta\rVert_2 \leq 1/(2\max_k\lVert \Psi_0^{(k)}\rVert_2)$ with $(\Delta^{(k)})_{(S^{(k)}\cup \{(i,i):i=1,\ldots,p]\})^c} = 0$.
Then $\lVert H^{(k)}\rVert_{\infty/\infty} \leq 2$ where $H^{(k)}\equiv \sum_{t=1}^\infty (-1)^t (\Psi_0^{(k)}\Delta^{(k)})^t, k=1,\ldots,K,$ and $R^{(k)}(\Delta^{(k)})$ has representation
$R^{(k)}(\Delta^{(k)}) = \Psi_0^{(k)} \Delta^{(k)}\Psi_0\Delta H^{(k)} \Psi_0^{(k)}$
with
$\lVert R^{(k)}(\Delta^{(k)})\rVert_\infty \leq 2 \lambda_{\Theta}^3\lVert \Delta^{(k)}\rVert_2^2 $.
\end{lemma}
\begin{proof}
Note that the Neumann series for a matrix $(I-A)^{-1}$ converges if the operator norm of $A$ is strictly less than 1, and that the $\ell_\infty$-norm is bounded by the operator norm. A proof is similar to that of Lemma 5 of \citet{MR2836766} with the induced infinity norm $\lVert \cdot\rVert_{\infty/\infty}$ replaced by the operator norm in appropriate inequalities. 
\end{proof}

The following lemma is similar to the statement of Lemma 6 of \citet{MR2836766}. 
\begin{lemma} 
\label{lemma:rvkmr6}
Suppose that 
\begin{equation*}
r\equiv \frac{4}{\min_k \pi_k}\kappa_{\Gamma}(\max_{k}\lVert \Xi^{(k)}\rVert_\infty + \rho_n +\rho_n\rho_2\lVert L\rVert^{1/2}_2) 
< 
\frac{1}{6d\max\left\{\kappa_{\Psi},\kappa_{\Psi}^3\kappa_{\Gamma}\right\}}, \quad k=1,\ldots, K.
\end{equation*}
Suppose moreover that $(\Theta_0^{(k)}\otimes \Theta_0^{(k)})_{S^{(k)}S^{(k)}}$ are invertible for $k=1,\ldots,K$.
Then with probability $1-2K\exp(-n \min_k\pi_k^2/2)$,
\begin{equation*}
\max_k\lVert \tilde{\Theta}_{\rho_n}^{(k)}-\Theta_0^{(k)}\rVert_\infty  \leq  (3/2)r.
\end{equation*}
\end{lemma}
\begin{proof}
We apply Shauder's fixed point theorem on the event $\min_k\pi_k/2 \leq n_k/n$, which holds with probability $1-2K\exp(-n \min_k\pi_k^2/2)$ by Lemma \ref{lemma:const6} with $\epsilon = \min_k\pi_k/2$.  
We first define the function $f_k$ and its domain $\mathcal{D}_k $ to which the fixed point theorem applies.
Let $\overline{S}^{(k)} = S^{(k)}\cup \{(i,i):1\leq i\leq p\}$, and define
\begin{equation*}
\mathcal{D}_k = \{A \in \mathbb{R}^{p\times p}: A = A^T, \,\,  
x^T(A+\Theta_0^{(k)})x \geq 0, \ \forall x\in\mathbb{R}^p,\,\,
\lVert A_{\overline{S}^{(k)}}\rVert_\infty \leq r,\,\, 
A_{(\overline{S}^{(k)})^{c}}=0 \}.
\end{equation*}
This set is a convex, compact subset of the set of all symmetric matrices.

Let $\check{U}_l^{(k)} \in \mathbb{R}^{p\times p},l=1,2,$ be zero-filled matrices whose $(i,j)$-element is $\check{U}_{l,ij}^{(k)}$ in Lemma \ref{lemma:RWRY3} if $(i,j)\in S^{(k)}$ and zero otherwise. 
Define the map  $g_k$ on the set of invertible matrices in $\mathbb{R}^{p\times p} $ by $g_k(B) = (n_k/n)(B^{-1}-\hat{\Psi}_n^{(k)})-\rho_{n}\check{U}_{1}^{(k)}- \rho_n\rho_2\check{U}_2^{(k)}$. 
Note that $\{g_k(\check{\Theta}_{\rho_n}^{(k)})\}_{S^{(k)}} = 0$ is the KKT condition for the restricted problem (\ref{eqn:oracle}).
Let $\delta>0$ be a constant such that $\delta<\min\{1/2,1/\{10(4dr +1)\}\}r$ and $\delta+r\leq 1/\{6d\max\{\kappa_{\Psi},\kappa_{\Psi}^3\kappa_{\Gamma}\}$.
Define a continuous function $f_k:\mathcal{D}_k\mapsto \mathcal{D}_k$ as 
\begin{eqnarray*}
(f_k(A))_{ij}= 
\left\{\begin{array}{ll}
\left\{h_k(A)\Theta^{(k)}_0g_k(A+\Theta^{(k)}_0+\delta I)\Theta_0^{(k)}+ A\right\}_{ij}, & (i,j)\in S^{(k)},i\neq j\\
0, & (i,j) \in (S^{(k)})^c,i\neq j,\\
A_{ij} & i=j,
\end{array}
\right.
\end{eqnarray*}
where 
\begin{equation*}
h_k(A) \equiv \frac{2^{-1}\min\{\lambda_1(A+\Theta_0^{(k)}),2^{-1}\} + 2^{-1}}{\max\{|\lambda_1(\{\Theta_0^{(k)}g_k(A+\Theta^{(k)}_0+\delta I)\Theta_0^{(k)}\}_{S^{(k)}}- I)|,1\}}.
\end{equation*}
Let $\tilde{f}_k(A) = h_k(A)\Theta^{(k)}_0g_k(A+\Theta^{(k)}_0+\delta I)\Theta_0^{(k)}$. 
Then $f_k(A) = (\tilde{f}_k(A))_{S^{(k)}} + A$ for $A\in\mathcal{D}_k$.

We now verify the conditions of Shauder's fixed point theorem below. Once conditions are established, the theorem yields that $f_k(A) = A$.
Since $(f_k(A))_{(\overline{S}^{(k)})^c}=A$ for any $A\in\mathcal{D}_k$,  and $h_k(A)>0$, the solution $A$ to $f_k(A) = A$ is determined by $(\Theta_0^{(k)}g_k(A+\Theta_0^{(k)}+\delta I)\Theta_0^{(k)})_{S^{(k)}}=0$. 
Vectorizing this equation to obtain $(\Theta_0^{(k)}\otimes \Theta_0^{(k)})_{S^{(k)}S^{(k)}} \{g_k(A+\Theta_0^{(k)}+\delta I)\}_{S^{(k)}} =0$, it follows from the invertibility of $(\Theta_0^{(k)}\otimes \Theta_0^{(k)})_{S^{(k)}S^{(k)}}$ that $\{g_k(A+\Theta_0^{(k)}+\delta I)\}_{S^{(k)}} = 0$. 
By the uniqueness of the KKT condition, the solution is $A = \check{\Theta}_{\rho_n}^{(k)} - \Theta_0^{(k)} - \delta I$. 
Since $A\in \mathcal{D}_k$, and $\delta<r/2$, we conclude $\lVert \check{\Theta}_{\rho_n}^{(k)} - \Theta_0^{(k)}\rVert_\infty \leq (3/2)r$. 

In the following, we write $\vec{A} = \mbox{vec}(A)$ for a matrix $A$ for notational convenience. 
For $J\subset \{(i,j):i,j=1,\ldots,p\}$, $\mbox{vec}(A)_J$ should be understood as $\vec{A}_J$.

The function $f_k$ is continuous on $\mathcal{D}_k$. 
To see this, note first that $A+\Theta^{(k)}_0+\delta I$ is positive definite for every $A\in \mathcal{D}_k$ so that the inversion is continuous.
Note also that all elements in the matrices involved with eigenvalues in $h_k(A)$ are uniformly bounded in $\mathcal{D}_k$, and hence the eigenvalues are also uniformly bounded.

To show that $f_k(A)\in\mathcal{D}_k$, first we show that $f_k(A)+\Theta_0^{(k)}$ is positive semidefinite.
This follows  because for any $x\in \mathbb{R}^p$ 
\begin{eqnarray*}
&&x^T(f_k(A)+\Theta_0^{(k)})x \\
&&=x^T\{(\tilde{f}_{k}(A))_{S^{(k)}}-I\}x +x^T(A+\Theta_0^{(k)})x+x^Tx  \\
&& \geq h_k(A)\lambda_1(\{\Theta_0^{(k)}g_k(A+\Theta^{(k)}_0+\delta I)\Theta_0^{(k)}\}_{S^{(k)}}-I)\lVert x\rVert^2 +\lambda_1(A+\Theta_0^{(k)})\lVert x\rVert^2 +  \lVert x\rVert^2\geq 0.
\end{eqnarray*}
To see this, note that if $\lambda_A\equiv \lambda_1(\{\Theta_0^{(k)}g_k(A+\Theta^{(k)}_0+\delta I)\Theta_0^{(k)}\}_{S^{(k)}}-I)$ is positive, then the inequality easily follows. 
On the other hand, if $\lambda_A< -1$, we have
\begin{eqnarray*}
h_k(A)\lambda_A \lVert x\rVert^2
&\geq& -2^{-1}\min\{\lambda_1(A+\Theta_0^{(k)}),2^{-1}\}\lVert x\rVert^2 - 2^{-1}\lVert x\rVert^2\\
&\geq &-(\lambda_1(A+\Theta_0^{(k)})/2 +1/2)\lVert x\rVert^2.
\end{eqnarray*}
Lastly, if $-1\leq\lambda_A <0$, we have
\begin{eqnarray*}
h_k(A)\lambda_A \lVert x\rVert^2
&\geq &- |\lambda_A| [2^{-1}\min\{\lambda_1(A+\Theta_0^{(k)}),2^{-1}\}+1/2]\lVert x\rVert^2\\
&\geq &- |\lambda_A|(\lambda_1(A+\Theta_0^{(k)})/2 +1/2)\lVert x\rVert^2.
\end{eqnarray*}

Next, we show that  $\lVert f_k(A)_{\overline{S}^{(k)}}\rVert_\infty\leq r$.
Because $\diag(f_k(A)) = \diag(A)$, we have only to show $\lVert f_k(A)_{S^{(k)}}\rVert_\infty\leq r$.
Since $\delta+r\leq 1/\{6d\max\{\kappa_{\Psi},\kappa_{\Psi}^3\kappa_{\Gamma}\}$, we have
\begin{equation*}
\lVert \Psi_0^{(k)}(A+\delta I)\rVert_{\infty/\infty} 
\leq \kappa_{\Psi}d \lVert A+\delta I\rVert_\infty 
\leq \kappa_{\Psi}d (r+\delta) \leq 1/3.
\end{equation*}
It follows from Lemma \ref{lemma:rvkmr5} that
\begin{eqnarray*}
R(A+\delta I) = (A+\delta I + \Theta_0^{(k)})^{-1} -\Psi_0^{(k)} + \Psi_0^{(k)} (A+\delta I) \Psi_0^{(k)} = \{\Psi_0^{(k)}(A+\delta I)\}^2 H^{(k)} \Psi_0^{(k)}.
\end{eqnarray*}
Thus, adding and subtracting $\Psi_0^{(k)}$ yields
\begin{eqnarray*}
\tilde{f}_k(A) +A
&=& h_k(A)
\Theta_0^{(k)}((n_k/n)\{\Psi_0^{(k)}(A+\delta I)\}^2 H^{(k)} \Psi_0^{(k)}- (n_k/n)\Xi^{(k)} -\rho_n\check{U}_1^{(k)}- \rho_n\rho_2 \check{U}_2^{(k)})\Theta_0^{(k)}\\
 &&+ (1-(n_k/n)h_k(A))A - (n_k/n)\delta h_k(A)I.
\end{eqnarray*}
Vectorization and restriction on $S^{(k)}$ gives
\begin{eqnarray}
&&\lVert\mbox{vec}(f_k(A))_{S^{(k)}}\rVert_\infty
=\lVert\mbox{vec}(\tilde{f}_k(A)+A)_{S^{(k)}}\rVert_\infty
\nonumber\\
&&\leq (n_k/n)h_k(A)\lVert\{(\Gamma^{(k)})^{-1}\}_{S^{(k)}S^{(k)}}\mbox{vec}(\{\Psi_0^{(k)}(A+\delta I)\}^2 H^{(k)} \Psi_0^{(k)})_{S^{(k)}}\rVert_\infty\nonumber\\
&&\quad +h_k(A)\lVert\{(\Gamma^{(k)})^{-1}\}_{S^{(k)}S^{(k)}}\{\mbox{vec}((n_k/n)\Xi^{(k)})_{S^{(k)}}+\rho_n\mbox{vec}(\check{U}_1^{(k)})_{S^{(k)}}+\rho_n\rho_2\mbox{vec}(\check{U}_2^{(k)})_{S^{(k)}}\}\rVert_\infty \nonumber\\
&&\quad + (1-(n_k/n)h_k(A))\lVert \mbox{vec}(A)_{S^{(k)}}\rVert_\infty + (n_k/n)\delta,
\label{eqn:fxdthmr1}
\end{eqnarray}
where $\{(\Gamma^{(k)})^{-1}\}_{S^{(k)}S^{(k)}} = (\Theta^{(k)}\otimes \Theta_0^{(k)})_{S^{(k)}S^{(k)}}$.
Here we used $h_k(A) \leq (1/4+1/2)/1 =3/4$.
For the first term of the upper bound in (\ref{eqn:fxdthmr1}), it follows by the inequality $\lVert Ax\rVert_\infty \leq \lVert A\rVert_{\infty/\infty}\lVert x\rVert_\infty$ for $A\in\mathbb{R}^{p\times p}$ and $x\in\mathbb{R}^p$, Lemma \ref{lemma:rvkmr5} and the choice of $\delta$ satisfying $\delta+r\leq 1/\{6d\max\{\kappa_{\Psi},\kappa_{\Psi}^3\kappa_{\Gamma}\}$ that
\begin{eqnarray*}
&&\lVert\{(\Gamma^{(k)})^{-1}\}_{S^{(k)}S^{(k)}}\mbox{vec}(\{\Psi_0^{(k)}(A+\delta I)\}^2 H^{(k)} \Psi_0^{(k)})_{S^{(k)}}\rVert_\infty\\
&&\leq \kappa_{\Gamma}\lVert R^{(k)}(A+\delta I)\rVert_\infty 
\leq \kappa_{\Gamma}\frac{3}{2}d\lVert A + \delta I\rVert_\infty^2\kappa_{\Psi}^3
\leq \kappa_{\Gamma}\frac{3}{2}d\lVert A + \delta I\rVert_\infty(r+\delta)\kappa_{\Psi}^3\\
&&\leq (r+\delta)/4.
\end{eqnarray*}
For the second term, it follows by the assumption, the inequality that $\lVert Ax\rVert_\infty \leq \lVert A\rVert_{\infty/\infty}\lVert x\rVert_\infty$ for $A\in\mathbb{R}^{p\times p}$ and $x\in\mathbb{R}^p$, and Lemma \ref{lemma:sbdiff} that 
\begin{eqnarray*}
&&\lVert\{(\Gamma^{(k)})^{-1}\}_{S^{(k)}S^{(k)}}\{(n_k/n)\mbox{vec}(\Xi^{(k)})_{S^{(k)}}+\rho_n\mbox{vec}(\check{U}_1^{(k)})_{S^{(k)}}+\rho_n\rho_2\mbox{vec}(\check{U}_2^{(k)})_{S^{(k)}}\}\rVert_\infty\\
&&\leq \kappa_{\Gamma}(\lVert \Xi^{(k)} + \rho_n + \rho_n\rho_2 \lVert L\rVert_2^{1/2})
= (\min_k \pi_k)r/4\leq (n_k/n)r/2.
\end{eqnarray*}
Thus, we can further bound $\lVert\mbox{vec}((\tilde{f}_k(A)+A)_{S^{(k)}})\rVert_\infty$ by
\begin{equation}
\frac{n_k}{n}h_k(A) \frac{r+\delta}{4} + \frac{n_k}{n}h_k(A)\frac{r}{2} + \left(1-\frac{n_k}{n}h_k(A)\right) r + \frac{n_k}{n}\delta 
= r\left\{1- \frac{n_k}{n}\frac{h_k(A)}{4}\right\} +  \frac{n_k}{n}\left\{1+\frac{h_k(A)}{4}\right\}\delta
\label{eqn:fxdthmr2}
\end{equation}

Noting that $\delta \leq r/2$, a similar reasoning  shows that 
\begin{eqnarray*}
&&\lVert (\Theta_0^{(k)} g_k(A+\Theta_0^{(k)} + \delta I)\Theta_0^{(k)})_{S^{(k)}}\rVert_\infty \\
&&\leq \lVert A_{S^{(k)}}\rVert_\infty + \lVert \Theta_0^{(k)} g_k(A+\Theta_0^{(k)} + \delta I)\Theta_0^{(k)} + A)_{S^{(k)}}\rVert_\infty\\
&&\leq (n_k/n)\lVert\{(\Gamma^{(k)})^{-1}\}_{S^{(k)}S^{(k)}}\mbox{vec}(\{\Psi_0^{(k)}(A+\delta I)\}^2 H^{(k)} \Psi_0^{(k)})_{S^{(k)}}\rVert_\infty\nonumber\\
&&\quad +\lVert\{(\Gamma^{(k)})^{-1}\}_{S^{(k)}S^{(k)}}\{(n_k/n)\mbox{vec}(\Xi^{(k)})_{S^{(k)}}+\rho_n\mbox{vec}(\check{U}_1^{(k)})_{S^{(k)}}+\rho_n\rho_2\mbox{vec}((\check{U}_2^{(k)})_{S^{(k)}})\}\rVert_\infty \nonumber\\
&&\quad + (2-(n_k/n))\lVert \mbox{vec}(A)_{S^{(k)}}\rVert_\infty + (n_k/n)\delta\\
&& \leq \frac{r+\delta}{4} + \frac{r}{2} +2r +\delta 
\leq 4r.
\end{eqnarray*}
Thus, the inequality $\lVert B\rVert_2\leq \lVert B\rVert_{\infty/\infty}$ for $B=B^T$ implies that
\begin{eqnarray*}
&&|\lambda_1(\{\Theta_0^{(k)}g_k(A+\Theta^{(k)}_0+\delta I)\Theta_0^{(k)}\}_{S^{(k)}}- I)|\\
&&\leq \lVert \lambda_1(\{\Theta_0^{(k)}g_k(A+\Theta^{(k)}_0+\delta I)\Theta_0^{(k)}\}_{S^{(k)}})\rVert_2 +  1\\
&&\leq \lVert \lambda_1(\{\Theta_0^{(k)}g_k(A+\Theta^{(k)}_0+\delta I)\Theta_0^{(k)}\}_{S^{(k)}})\rVert_{\infty/\infty} +  1\\
&& \leq 4dr+1.
\end{eqnarray*}
Hence $h_k(A) \geq 1/(8dr+2)$ for every $A\in \mathcal{D}_k$. 

Now (\ref{eqn:fxdthmr2}) is further bounded by $r$:
\begin{eqnarray*}
&&r\left\{1- \frac{n_k}{n}\frac{h_k(A)}{4}\right\} +  \frac{n_k}{n}\left\{1+\frac{h_k(A)}{4}\right\}\delta\\
&&\leq r\left\{1- \frac{n_k}{n}\frac{h_k(A)}{4}\right\} + \frac{n_k}{n}\left\{1+\frac{h_k(A)}{4}\right\}
\frac{r}{10(4dr +1)}\\
&&\leq r\left\{1- \frac{n_k}{n}\frac{h_k(A)}{4}\right\} + \frac{n_k}{n}\left\{1+\frac{h_k(A)}{4}\right\}
\frac{h_k(A)r}{5}\\
&&\leq r -
\frac{n_k}{n}\frac{h_k(A)-h_k^2(A)}{20}r \leq r.
\end{eqnarray*}
Here we used the fact that $\delta\leq r/\{10(4dr +1)\}$ and $1/(8dr +2)\leq h_k(A)<1$.
Thus, $\lVert (f_k(A))_{S^{(k)}}\rVert_\infty \leq r$. 

Since $(f_k(A))_{(S^{(k)})^c} = 0$ by definition, all the conditions for the fixed point theorem are established. 
This completes the proof.
\end{proof}

We now give a proof of Theorem \ref{thm:modelconst}.
Note that Condition \ref{cond:rhos} implies that
\begin{equation*}
\rho_n
< \min\left\{\frac{\min_k\pi_k}{72d\kappa_\Gamma}\min\left\{\frac{1}{\kappa_\Psi},\frac{1}{\kappa_\Psi^3\kappa_\Gamma},\frac{\min_k \pi_k }{56\kappa_\Psi^3\kappa_{\Gamma}}\alpha\right\},\frac{c_8}{6},\frac{c_9\min_k\sqrt{d_k}}{12}\right\}.
\end{equation*}
\begin{proof}[Proof of Theorem \ref{thm:modelconst}]
We prove that the oracle estimator $\check{\Theta}_{\rho_n}$ satisfies (I) the model selection consistency and (II) the KKT conditions of the original problem (\ref{eqn:originallasso}) with $(\check{\Theta}_{\rho_n}, \tilde{U}_1,\tilde{U_2})$.
The model selection consistency of $\hat{\Theta}_{\rho_n} =\check{\Theta}_{\rho_n}$ then follows by the uniqueness of the solution to the original problem. 
The following discussion is on the event that $\min_k\pi_k/2\leq n_k/n,k=1,\ldots,K,$ and $\max_k\lVert \Xi^{(k)} \rVert_\infty\leq \alpha/8$. 
Note that this event has probability approaching 1 by Lemmas \ref{lemma:const4} and  \ref{lemma:const6}.

First we obtain an $\ell_\infty$-bound of the error of the oracle estimator.
Note that by Condition \ref{cond:rhos} and the fact that $\alpha\in [0,1)$
\begin{eqnarray*}
\frac{\alpha}{8} + 1 + \rho_2\lVert L\rVert_2^{1/2}
\leq \frac{\alpha}{8} + 1 + \frac{\alpha^2}{4(2-\alpha)} \leq 3. 
\end{eqnarray*}
Thus, it follows from Condition \ref{cond:rhos} that 
\begin{eqnarray*}
(4/\min_k\pi_k)\kappa_{\Gamma}(\lVert \Xi^{(k)}\rVert_\infty + \rho_n + \rho_n\rho_2\lVert L\rVert_2^{1/2}) 
&<& \frac{12 \kappa_{\Gamma}}{\min_k\pi_k}\frac{\min_k\pi_k}{72d\kappa_{\Gamma}}\min\left\{\frac{1}{\kappa_{\Psi}},\frac{1}{\kappa_{\Psi}^3\kappa_{\Gamma}}\right\}\\
&=&\frac{1}{6d\max\{\kappa_\Psi,\kappa_\Psi^3\kappa_{\Gamma}\}}.
\end{eqnarray*}
Because $(\Theta_0^{(k)}\otimes \Theta_0^{(k)})_{S^{(k)}S^{(k)}}$ is invertible by Condition \ref{cond:irrep}, we can apply Lemma \ref{lemma:rvkmr6} to obtain $\lVert \check{\Theta}_{\rho_n}^{(k)}-\Theta_0^{(k)}\rVert_\infty \leq (6/\min_k\pi_k)\kappa_{\Gamma}(\lVert \Xi^{(k)}\rVert_\infty + \rho_n + \rho_n\rho_2\lVert L\rVert_2^{1/2})$ with probability approaching 1.

As a consequence of the $\ell_\infty$-bound, $\check{\Theta}_{\rho_n,ij}\neq 0$ for $(i,j)\in S$, because $\lVert \check{\Theta}_{\rho_n}^{(k)}-\Theta_0^{(k)}\rVert_\infty\leq 3 \rho_n \leq c_8/2 <\min_{k=1,\ldots,K,i\neq j}|\theta_{0,ij}^{(k)}|$ by Conditions \ref{cond:nullL} and \ref{cond:rhos}.
This establishes the model selection consistency of the oracle estimator.

Next, we show that the Oracle estimator satisfies the KKT condition of the original problem (\ref{eqn:originallasso}).
As the first step, we prove $\tilde{U}_{1,ij}^{(k)}\in \partial \check{\Theta}_{\rho_n}^{(k)}$ for every $i,j,k$ with probability approaching 1.
Since $\check{\Theta}_{\rho_n,ij}\neq 0$ for $(i,j)\in S$ with probability approaching 1, $\tilde{U}_{1,ij}^{(k)} = \check{U}_{1,ij}^{(k)}$ for $(i,j)\in S^{(k)}$ by construction.
For $(i,j)\in (S^{(k)})^c$, we need to prove $|\tilde{U}_{1,ij}^{(k)}|< 1$ for every $i,j,k$.
To this end, it suffices to verify that $\lVert R^{(k)}(\check{\Theta}_{\rho_n}^{(k)}-\Theta_0^{(k)})\rVert_\infty\leq \alpha/8$ and apply Lemma \ref{lemma:strictDual}.
Applying Lemma~\ref{lemma:rvkmr5} with $\lVert \check{\Theta}_{\rho_n}^{(k)}-\Theta_0^{(k)}\rVert_\infty \leq (6/\min_k\pi_k)\kappa_{\Gamma}(\lVert \Xi^{(k)}\rVert_\infty + \rho_n + \rho_n\rho_2\lVert L\rVert_2^{1/2})$ and Condition \ref{cond:rhos} gives
\begin{eqnarray*}
&&\lVert R^{(k)}(\check{\Theta}_{\rho_n}^{(k)}-\Theta_0^{(k)})\rVert_\infty
\leq \frac{3}{2}d\kappa_{\Psi}^3 \lVert \check{\Theta}_{\rho_n}^{(k)}-\Theta_0^{(k)}\rVert_\infty^2 
\leq  \frac{3}{2}d\kappa_{\Psi}^3 \frac{324\kappa_{\Gamma}^2}{\min_k\pi_k^2}\rho_n^2\\
&&\leq \frac{486d\kappa_{\Psi}^3\kappa_{\Gamma}^2}{\min_k\pi_k^2} \left\{\frac{\min_k\pi_k}{72d\kappa_{\Gamma}} \frac{\min_k\pi_k}{56\kappa_\Psi^3\kappa_{\Gamma}}\alpha\right\} \rho_n \leq \frac{\alpha}{8}\alpha.
\end{eqnarray*}

Next, we prove that $\tilde{U}_{2,ij}\in \partial \sqrt{\check{\Theta}_{\rho_n,ij}L\check{\Theta}_{\rho_n,ij}}$ for every $(i,j)$.
For $(i,j)$ with $\omega_{0,ij}^{(k)}\neq 0$ for all $k=1,\ldots,K$, $\tilde{U}_{2,ij} = \check{U}_{\rho_n}\in \partial \sqrt{\check{\Theta}_{\rho_n,ij}L\check{\Theta}_{\rho_n,ij}}$. 
For $(i,j)$ with $\Omega_{0,ij} =0$, $\tilde{U}_{2,ij} = 0\in \partial \sqrt{\check{\Theta}_{\rho_n,ij}L\check{\Theta}_{\rho_n,ij}}$ by Lemma \ref{lemma:sbdiff}.
For $(i,j)$ with $\Omega_{0,ij}\neq 0$ and $\omega_{0,ij}^{(k')}=0$ for some $k'$, 
\begin{equation*}
\tilde{U}_{2,ij} = L\check{\Theta}_{\rho_n,ij}/\sqrt{\check{\Theta}_{\rho_n,ij}L\check{\Theta}_{\rho_n,ij}}\in \partial \sqrt{\check{\Theta}_{\rho_n,ij}L\check{\Theta}_{\rho_n,ij}}
\end{equation*}
if $L\check{\Theta}_{\rho_n,ij}\neq 0$.
To see $L\check{\Theta}_{\rho_n,ij}\neq 0$ holds with probability approaching 1,
let $(k,k')\in S$ with $k\neq k'$ such that $\Theta_{0,ij}^{(k)}/\sqrt{d_k}-\Theta_{0,ij}^{(k')}/\sqrt{d_{k'}}\neq 0$.
This pair $(k,k')$ exists by Condition \ref{cond:nullL} and the assumption $L\Theta_{0,ij}\neq 0$. 
We assume without loss of generality $\theta_{0,ij}^{(k)}/\sqrt{d_k}-\theta_{0,ij}^{(k')}/\sqrt{d_{k'}} > 0$.
Since $\lVert \check{\Theta}_{\rho_n}^{(k)} -\Theta_{0}^{(k)}\rVert_\infty \leq 3\rho_n\leq c_9\min_k\sqrt{d_k}/12$, it follows from Condition \ref{cond:rhos} that 
\begin{eqnarray*}
\frac{\check{\theta}_{\rho_n,ij}^{(k)}}{\sqrt{d_k}} - \frac{\check{\theta}_{\rho_n,ij}^{(k')}}{\sqrt{d_{k'}}}
&\geq& \frac{\theta_{0,ij}^{(k)}}{\sqrt{d_k}} - \frac{\theta_{0,ij}^{(k')}}{\sqrt{d_{k'}}} - 3\rho_n \left(\frac{1}{\sqrt{d_k}} + \frac{1}{\sqrt{d_{k'}}}\right)\\
&\geq& c_9 - 3\rho_n \left(\max_{W_{k,k'}\neq 0}\frac{1}{\sqrt{d_k}}+\frac{1}{\sqrt{d_{k'}}}\right)
\geq  \frac{1}{2}c_9.
\end{eqnarray*}
Hence, $\check{\Theta}_{\rho_n,ij}^TL\check{\Theta}_{\rho_n,ij} \geq W_{kk'}c_9^2/4>0$ or $L\check{\Theta}_{\rho_n,ij}\neq 0$.

Finally, we show that Equation~\eqref{eqn:kktmat} for the KKT condition holds. 
For the $(i,j)$-element of the equation with $\Omega_{0,ij}=0$, this equation hold by construction for every $k=1,\ldots,K$.
For the $(i,j)$-element with $\omega_{0,ij}^{(k)}\neq 0$ for every $k=1,\ldots,K$, the equation holds for every $k=1,\ldots,K$, because it is the equation for the KKT condition of the corresponding element in a restricted problem (\ref{eqn:oracle}).
For $(i,j)$-element with $\Omega_{0,ij}\neq 0$ and $\omega_{0,ij}^{(k')}=0$ for some $k'$, note that $\check{\Theta}_{\rho_n,ij}\neq 0$ with probability approaching 1 and that the rearrangement in $\Theta_{ij}$ and corresponding exchange of rows and columns of $L$ for each $i,j$ does not change the original and restricted optimization problems (\ref{eqn:originallasso}) and (\ref{eqn:oracle}).
Thus, with the appropriate rearrangement of elements and exchange of rows and columns, $\tilde{U}_{2,ij}^{(k)}$ with $\omega_{0,ij}^{(k)}\neq 0$ is in fact $\check{U}_{2,ij}^{(k)}$. Thus for such $k$ the equation holds because of the corresponding KKT condition in the restricted problem (\ref{eqn:oracle}). For other $k$, the equation holds by construction.
We thus conclude the oracle estimator satisfies the KKT condition of the original problem (\ref{eqn:originallasso}).
This completes the proof.
\end{proof}

\begin{proof}[Proof of Corollary \ref{col:rate}]
In the proof of Theorem \ref{thm:modelconst}, the $\ell_\infty$-bound of the error yields
\begin{equation*}
\lVert \hat{\Theta}_{\rho_n}^{(k)}-\Theta_0^{(k)} \rVert_\infty = O_P\left(\kappa_{\Gamma}\rho_n\right).
\end{equation*}
Note that if one of two matrices $A$ and $B$ is diagonal, $\lVert AB\rVert_\infty\leq \lVert A\rVert_\infty \lVert B\rVert_\infty$. 
Thus, we can proceed in the same way as in the proof of Theorem 2 of \citet{MR2417391} to conclude that
\begin{equation*}
\lVert \hat{\Omega}_n^{(k)}-\Omega_0^{(k)} \rVert_\infty = O_P\left(\kappa_{\Gamma}\rho_n\right).
\end{equation*}
The result follows from a similar argument to the proof of Corollary 3 in \citet{MR2836766}.
\end{proof}

\begin{proof}[Proof of Corollary \ref{col:modelconst}]
It follows from Condition~\ref{cond:rhos2} and Lemma~\ref{lemma:l1const} applied to $\check{\Theta}_{\rho_n}$ that $\lVert \check{\Theta}_{\rho_n}^{(k)}-\Theta_0^{(k)}\rVert_2 \leq 1/(2\lambda_\Theta)$. Then we can apply Lemma \ref{lemma:rvkmr5-2} instead of Lemma \ref{lemma:rvkmr5}. The rest is similar to the proof of Theorem \ref{thm:modelconst}.
\end{proof}

\subsection*{Hierarchical Clustering}
For simplicity, we prove Theorem \ref{thm:clst} for the case of $K=2$; the proof can be easily generalized to $K>2$.
Let $X$ and $Y$ be the random variable from the first and  subpopulation, respectively. 
Suppose that $X=(X_1,\ldots,X_p)^T\sim N(\mu_X,\Sigma_X)$ with $\mu_X = (\mu_{1,X},\ldots, \mu_{p,X})$ and the spectral decomposition $\Sigma_X=Q_X\Lambda_X Q^T_X$ of $\Sigma_X$ where $\lambda_{1,X},\ldots,\lambda_{p,X}$ are the eigenvalues of $\Sigma_X$ and that $Y\sim N(\mu_Y,\Sigma_Y)$ with $\mu_Y = (\mu_{1,Y},\ldots, \mu_{p,Y})$ and the spectral decomposition $\Sigma_Y=Q_Y\Lambda_Y Q^T_Y$ of $\Sigma_Y$ where $\lambda_{1,Y},\ldots,\lambda_{p,Y}$ are the eigenvalues of $\Sigma_Y$.
Define $Z=(X-Y)=(Z_1,\ldots,Z_p)^T\sim N(\mu_Z,\Sigma_Z)$ with $\mu_Z = (\mu_{1,Z},\ldots, \mu_{p,Z})$ and the spectral decomposition $\Sigma_Z=Q_Z\Lambda_Z Q^T_Z$ of $\Sigma_Z$ where $\lambda_{1,Z},\ldots,\lambda_{p,Z}$ are the eigenvalues of $\Sigma_Z$.
Let $\tilde{X} =(\tilde{X}_{1},\ldots,\tilde{X})^T= \Lambda_X^{1/2}Q_X^T\Sigma^{-1/2}_XX$, $\tilde{Y} =(\tilde{Y}_{1},\ldots,\tilde{Y})^T= \Lambda_Y^{1/2}Q_Y^T\Sigma^{-1/2}_YY$ and $\tilde{Z} =(\tilde{Z}_{1},\ldots,\tilde{Z})^T= \Lambda_Z^{1/2}Q_Z^T\Sigma^{-1/2}_ZZ$.
Then $\tilde{X}\sim N(\tilde{\mu}_X,\Lambda_X)$, $\tilde{Y}\sim N(\tilde{\mu}_Y,\Lambda_Y)$ and $\tilde{Z}\sim N(\tilde{\mu}_Z,\Lambda_Z)$
where $\tilde{\mu}_X = (\tilde{\mu}_{1,X},\ldots,\tilde{\mu}_{p,X})^T\equiv \Lambda_X^{1/2}Q_X^T\Sigma_X^{-1/2}\mu_X$, $\tilde{\mu}_Y = (\tilde{\mu}_{1,Y},\ldots,\tilde{\mu}_{p,Y})^T\equiv \Lambda_Y^{1/2}Q_Y^T\Sigma_Y^{-1/2}\mu_Y$ and $\tilde{\mu}_Z = (\tilde{\mu}_{1,Z},\ldots,\tilde{\mu}_{p,Z})^T\equiv \Lambda_Z^{1/2}Q_Z^T\Sigma_Z^{-1/2}\mu_Z$.
Let also
\begin{eqnarray*}
&&\mu^2_{\tilde{X}} = \lVert \tilde{\mu}_X^2\rVert/p, \quad
\mu^2_{\tilde{Y}} = \lVert \tilde{\mu}_Y^2\rVert/p, \quad
\mu^2_{\tilde{Z}} = \lVert \tilde{\mu}_Z^2\rVert/p, \\
&&\overline{\lambda}_{X} = \sum_{k=1}^p\lambda_{k,X}/p,\quad
\overline{\lambda}_{Y} = \sum_{k=1}^p\lambda_{k,Y}/p,\quad
\overline{\lambda}_{Z} = \sum_{k=1}^p\lambda_{k,Z}/p.
\end{eqnarray*}

\begin{lemma}[Lemma 1 of \citet{borysov2014}]
\label{lemma:BHM1}
Let $W_1,\ldots,W_p$ be independent non-negative random variables with finite second moments.
Let $S=\sum_{j=1}^p(W_j-\E W_j)$ and $v=\sum_{j=1}^p\E W_j^2$. Then for any $t>0$
  $P(S\leq -t)\leq \exp(-t^2/(2v))$.
\end{lemma}

The following lemma is an extension of Lemma 2 in \citet{borysov2014}.
\begin{lemma}
\label{lemma:2}
Let $0<a<\mu^2_{\tilde{X}}+\overline{\lambda}_X$.
Then
\begin{equation*}
P(\lVert X\rVert^2 <ap ) \leq \exp\left(-\frac{p^2(\mu^2_{\tilde{X}}+\overline{\lambda}_X -a)^2}{2\sum_{j=1}^p( \tilde{\mu}_{j,X}^4 +6\tilde{\mu}_{k,X}^2\lambda_{j,X}+3\lambda_{j,X}^2)}\right).
\end{equation*}
\end{lemma}
\begin{proof}
Note that elements of $\tilde{X}$ are independent and that $\tilde{X}_{j}\sim N(\tilde{\mu}_{j,X},\lambda_{j,X})$.
Thus, we have
\begin{eqnarray*}
&&\E\tilde{X}_{j}^2= \tilde{\mu}_{j,X}^2+\lambda_{j,X},\quad \mbox{Var}(\tilde{X}_j^2) = 2(\lambda_{j,X}^2+ 2\tilde{\mu}_{j,X}^2\lambda_{j,X}),\\
&&\E\tilde{X}_{j}^4 = \tilde{\mu}_{j,X}^4 +6\tilde{\mu}_{j,X}^2\lambda_{j,X}+3\lambda_{j,X}^2.
\end{eqnarray*}
Applying Lemma~\ref{lemma:BHM1} with $W_i = \tilde{X}_i^2$ gives
\begin{eqnarray*}
&&P(\lVert X\rVert^2 <ap)
=P(\lVert \tilde{X}\rVert^2 <ap)
=P\left[\sum_{j=1}^p (\tilde{X}_j^2 - \tilde{\mu}_{j,X}^2 -\lambda_{j,X})<-p(\mu^2_{\tilde{X}}+\overline{\lambda}_X -a)\right]\\
&&\leq \exp\left(-\frac{p^2(\mu^2_{\tilde{X}}+\overline{\lambda}_X -a)^2}{2\sum_{j=1}^p( \tilde{\mu}_{j,X}^4 +6\tilde{\mu}_{j,X}^2\lambda_{j,X}+3\lambda_{j,X}^2)}\right).
\end{eqnarray*}
\end{proof}

The following is an extension of Lemma 3 in \citet{borysov2014}.
\begin{lemma}
\label{lemma:3}
Let $a>\overline{\lambda}_X + \mu^2_{\tilde{X}}$.
Then
\begin{equation*}
P(\lVert X\rVert^2 >ap ) \leq \exp\left(-\frac{1}{2}\left(p+\sum_{j=1}^p\frac{a}{\lambda_{j,X}} -\sum_{j=1}^p\sqrt{1+2\frac{a}{\lambda_{j,X}}}\right)\right).
\end{equation*}
\end{lemma}
\begin{proof}
By Markov's inequality, for $t >\sum_{j=1}^p\lambda_{j,X}+\tilde{\mu}^2_{j,X}$, we get
\begin{eqnarray*}
&&P\left(\sum_{j=1}^pX_j^2\geq t\right)=P\left(\sum_{j=1}^p\tilde{X}_j^2\geq t\right)\\
&&= P\left[\exp\left(\sum_{j=1}^p\gamma\tilde{X}_j^2-\gamma\lambda_{j,X}-\gamma\tilde{\mu}_{j,X}^2\right)\geq \exp\left(\gamma t -
\gamma\sum_{j=1}^p(\lambda_{j,x}+\tilde{\mu}_{j,X}^2)\right)\right] \\
&&\leq \exp\left(-\gamma \left(t-\sum_{j=1}^p(\tilde{\mu}_{j,X}^2+\lambda_{j,X})\right)\right) \prod_{j=1}^p\E \exp((\gamma\lambda_{j,X})\tilde{X}_j^2/\lambda_{j,x})\\
&&= \exp\left(-\gamma \left(t-\sum_{j=1}^p\tilde{\mu}_{j,X}^2\right)\right)
\prod_{j=1}^p\exp\left(-\gamma\lambda_{j,X}-\frac{1}{2}\log (1-2\gamma \lambda_{j,X})\right)\exp\left(\frac{\gamma\tilde{\mu}_{j,X}^2}{1-2\gamma \lambda_{j,X}}\right).
\end{eqnarray*}
Since for all $u\in (0,1)$,
$-\log (1-u) - u \leq u^2/\{2(1-u)\}$
(see page 28 of \citet{boucheron2013concentration}),
the above display is bounded above by
\begin{equation*}
\exp\left(-\gamma \left(t-\sum_{i=1}^p\tilde{\mu}_{i,X}^2\right)\right)
\prod_{i=1}^p\exp\left(\frac{\gamma^2\lambda_{i,X}^2}{1-2\gamma\lambda_{i,X}}\right)\exp\left(\frac{\gamma\tilde{\mu}_{i,X}^2}{1-2\gamma \lambda_{i,X}}\right).
\end{equation*}
Using the following result from \citet{boucheron2013concentration} 
\begin{equation*}
\label{eqn:ineq2}
\inf_{\gamma\in(0,1/c)}\frac{v\gamma^2}{2(1-c\gamma)} -t\gamma
=-\frac{v}{c^2}h\left(\frac{ct}{v}\right).
\end{equation*}
wherein
$h(u) = 1+u -\sqrt{1+2u}, u>0$,
we further obtain the upper bound
\begin{equation*}
\exp\left(\gamma \sum_{i=1}^p\tilde{\mu}_{i,X}^2\right)
\prod_{i=1}^p\exp\left(-\frac{1}{2}\left(1+\frac{t}{\lambda_{i,X}p} -\sqrt{1+2\frac{t}{\lambda_{i,X}p}}\right)\right)\exp\left(\frac{\gamma\tilde{\mu}_{i,X}^2}{1-2\gamma \lambda_{i,X}}\right).
\end{equation*}
Taking $\gamma\downarrow 0$, the upper bound becomes
\begin{equation*}
\exp\left(-\frac{1}{2}\left(p+\sum_{i=1}^p\frac{t}{\lambda_{i,X}p} -\sum_{i=1}^p\sqrt{1+2\frac{t}{\lambda_{i,X}p}}\right)\right).
\end{equation*}
Choosing $t=ap$, we have
\begin{equation*}
P\left(\sum_{i=1}^p\tilde{X}_i^2\geq ap\right) \leq
\exp\left(-\frac{1}{2}\left(p+\sum_{i=1}^p\frac{a}{\lambda_{i,X}} -\sum_{i=1}^p\sqrt{1+2\frac{a}{\lambda_{i,X}}}\right)\right).
\end{equation*}
Note that $f(u)=(1+2u)^{1/2} \leq u$ for $u\geq 0$ because $f'(0) = 1$ and $f'$ is decreasing for $u>0$.
Thus, $P\left(\sum_{i=1}^p\tilde{X}_i^2\geq ap\right)\rightarrow 0$ as $p\rightarrow \infty$.
\end{proof}

\begin{proof}[Proof of Theorem \ref{thm:clst}]
For simplicity, we present the proof for the case of $K=2$; the proof can be easily generalized to $K>2$.
Let $n_1$ and $n_2$ be the sample sizes for the first and second subpopulations, respectively.
Define
\begin{eqnarray*}
&&E_1 = \left\{\max_{i,j}\lVert X_i-X_j\rVert <\min_{k,l}\lVert X_k-Y_l\rVert\right\},\quad
E_2 = \left\{\max_{i,j}\lVert Y_i-Y_j\rVert <\min_{k,l}\lVert X_k-Y_l\rVert\right\},\\
&&E_3 = \left\{\max_{i,j}\lVert X_i-X_j\rVert^2 <ap\right\},\quad
E_4 = \left\{\max_{i,j}\lVert Y_i-Y_j\rVert^2 <ap\right\},\\
&&E_5 = \left\{\max_{k,l}\lVert X_k-Y_l\rVert^2 >ap\right\}.
\end{eqnarray*}
for a fixed $a >0$ satisfying the assumption.
The intersection $E_1\cap E_2$ is contained in the event that the clustering performs in the way that two subpopulations are joined in the last step.
The intersection  $E_3\cap E_4\cap E_5$ is also contained in $E_1\cap E_2$, or in other words, $P((E_1\cap E_2)^c)\leq P(E_3^c)+P(E_4^c)+P(E_5^c)$. 
Thus, it suffices to show that $P(E_3^c)+P(E_4^c)+P(E_5^c) \rightarrow 0$ as $n,p\rightarrow \infty$.

For $E_3^c$ and $E_4^c$ we have by Lemma \ref{lemma:3} that
\begin{eqnarray*}
P(E_3^c) &\leq&  \sum_{i,j}^nP(\lVert X_i-X_j\rVert^2 >ap) = \frac{n_1(n_1-1)}{2}P(\lVert X_1-X_2\rVert^2>ap)\\
&\leq &\frac{n_1(n_1-1)}{2} \exp\left(-\frac{1}{2}\left(p+\sum_{l=1}^p\frac{a}{2\lambda_{l,X}} -\sum_{l=1}^p\sqrt{1+\frac{a}{\lambda_{l,X}}}\right)\right)\\
&\leq & \exp\left(-\frac{1}{2}\left(p+\sum_{l=1}^p\frac{a}{2\lambda_{l,X}} -\sum_{l=1}^p\sqrt{1+\frac{a}{\lambda_{l,X}}}\right) + 2\log n_1\right)\\
&= & \exp\left(-\frac{p}{2}\left(1+\frac{1}{p}\sum_{l=1}^p\frac{a}{2\lambda_{l,X}} -\frac{1}{p}\sum_{l=1}^p\sqrt{1+\frac{a}{\lambda_{l,X}}}+4\frac{\log n_1}{p}\right) \right)
\end{eqnarray*}
and that
\begin{eqnarray*}
P(E_4^c)
&\leq &
\exp\left(-\frac{p}{2}\left(1+\frac{1}{p}\sum_{l=1}^p\frac{a}{2\lambda_{l,Y}} -\frac{1}{p}\sum_{l=1}^p\sqrt{1+\frac{a}{\lambda_{l,Y}}}+4\frac{\log n_2}{p}\right) \right).
\end{eqnarray*}
for $a$ satisfying
$
a>2\max\{\overline{\lambda}_{X},\overline{\lambda}_{Y}\}$.

Note that $\log n_k/p\rightarrow 0,k=1,2$ as $n_1,n_2,p\rightarrow \infty$.
Moreover $x-\sqrt{1+2x}\geq 0$ for $x>0$.
Thus, $P(E_3^c)\rightarrow 0 $ and $P(E_4^c)\rightarrow 0$ as $n_1,n_2,p\rightarrow \infty$.
For $E_5^c$, we have by Lemma \ref{lemma:2} that
\begin{eqnarray*}
P(E_5^c) &\leq&  \sum_{i,j}P(\lVert X_i-Y_j\rVert^2 <ap) \leq  n_1n_2P(\lVert X_1-Y_1\rVert^2<ap)\\
&& \leq   \exp\left(-\frac{p^2(\mu^2_{\tilde{Z}}+\overline{\lambda}_Z -a)^2}{2\sum_{l=1}^p( \tilde{\mu}_{i,Z}^4 +6\tilde{\mu}_{l,Z}^2\lambda_{l,Z}+3\lambda_{l,Z}^2)}+\log n_1n_2\right)
\end{eqnarray*}
for $a<\mu^2_{\tilde{Z}}+\overline{\lambda}_Z$.
Given the assumption 
$c_{10}\leq \lambda_{j,X}\leq c_{11}$, $c_{10}\leq \lambda_{j,Y}\leq c_{11}$, $\max\{|\mu_{j,X}|,|\mu_{j,Y}|\}\leq c_{11}, j=1,2,\ldots$. Thus, we get $P(E_5^c) \rightarrow 0$ as $n_1,n_2,p\rightarrow \infty$.

Since
$2\overline{\lambda}_X-\lambda_{p,X}-\lambda_{p,Y}
\geq 2\overline{\lambda}_X-\overline{\lambda}_Z$, and
$2\overline{\lambda}_Y-\lambda_{p,X}-\lambda_{p,Y}
\geq 2\overline{\lambda}_Y-\overline{\lambda}_Z$,
the assumption that $\mu^2_{\tilde{Z}}>2\min\{\overline{\lambda}_X,\overline{\lambda}_Y\}-\lambda_{p,X}-\lambda_{p,Y}$ implies that there exists $a$ such that $a<\overline{\mu}_{\tilde{Z}}+\overline{\lambda}_Z$ and $a>2\max\{\overline{\lambda}_X,\overline{\lambda}_Y\}$. This completes the proof.
\end{proof}

\bibliographystyle{plainnat}
\bibliography{grpLap}
\end{document}